\documentclass[11pt]{article}

\usepackage{fullpage}

\usepackage{wrapfig}

\usepackage[utf8]{inputenc} 
\usepackage[T1]{fontenc}    
\usepackage{hyperref}       
\usepackage{url}            

\usepackage{amsfonts}       
\usepackage{nicefrac}       
\usepackage{microtype}      
\usepackage{graphicx}
\usepackage{amsthm}
\usepackage{amsmath}
\usepackage{amssymb}
\usepackage{amsfonts}       
\usepackage{siunitx}
\usepackage{booktabs} 
\usepackage{etoolbox}
\usepackage{float}
\usepackage{caption}
\usepackage{subcaption}
\usepackage{algorithm}
\usepackage{algorithmic}
\usepackage{mathtools}

\usepackage{makecell}
\usepackage{multirow}
\usepackage{booktabs}       

\usepackage{xspace}
\newcommand{\eqdef}{\coloneqq}

\DeclareMathOperator{\prox}{prox}

\newcommand{\Exp}[1]{\mathbb{E}\left[#1\right]}

\def\<#1,#2>{\langle #1,#2\rangle}

\newcommand{\cO}{\mathcal O}

\newcommand{\cD}{\mathcal D}
\newcommand{\cQ}{\mathcal Q}
\newcommand{\cS}{\mathcal S}

\usepackage{bm}

\newcommand{\R}{\mathbb R}

 \newcommand{\squeeze}{} 
\usepackage[flushleft]{threeparttable} 

\usepackage{color}

\usepackage{colortbl}
\definecolor{bgcolor}{rgb}{0.93,0.99,1}
\definecolor{bgcolor2}{rgb}{0.8,1,0.8}
\definecolor{bgcolor3}{rgb}{0.50,0.90,0.50}

\usepackage{tcolorbox}
\usepackage{pifont}
\definecolor{mydarkgreen}{RGB}{39,130,67}
\definecolor{mydarkred}{RGB}{192,25,25}
\newcommand{\green}{\color{mydarkgreen}}
\newcommand{\red}{\color{mydarkred}}
\newcommand{\cmark}{\green\ding{51}}%
\newcommand{\xmark}{\red\ding{55}}%

\usepackage{relsize} 
\newcommand{\algname}[1]{{\sf\red\relscale{0.90}#1}\xspace}

\newcommand{\dataset}[1]{{\tt\color{blue}#1}\xspace}

\usepackage{thmtools} 
\usepackage{thm-restate}
\newtheorem{theorem}{Theorem}
 
\newtheorem{corollary}{Corollary}
\newtheorem{lemma}{Lemma}
\newtheorem{assumption}[theorem]{Assumption}

\hypersetup{
    colorlinks=true,
    linkcolor=blue,
    citecolor=blue,
    filecolor=magenta,      
    urlcolor=cyan,
    pdftitle={PFL},
    pdfpagemode=FullScreen,
}

\newcommand{\MM}{W}

\usepackage{textcomp} 
 
\usepackage{natbib}
\usepackage{sidecap}

\usepackage[colorinlistoftodos,bordercolor=orange,backgroundcolor=orange!20,linecolor=orange,textsize=scriptsize]{todonotes}

 \title{\bf Variance Reduced ProxSkip: Algorithm, Theory \\ and Application to Federated Learning}

\author{
\bf Grigory Malinovsky \\ KAUST \\  \texttt{grigory.malinovsky@kaust.edu.sa}  \and
\bf Kai Yi \\ KAUST \\  \texttt{kai.yi@kaust.edu.sa} \and
\bf Peter Richt\'{a}rik \\ KAUST \\  \texttt{peter.richtarik@kaust.edu.sa}   
} 

\date{May 26, 2022 \\ (revised on July 9, 2022)}

\begin{document}

\maketitle

\begin{abstract}

We study distributed optimization methods based on the {\em local training (LT)} paradigm: achieving communication efficiency by performing richer local gradient-based training on the clients before  parameter averaging. 
Looking back at the progress of the field, we {\em identify 5 generations of LT methods}: 1) heuristic, 2) homogeneous, 3) sublinear, 4) linear, and 5) accelerated. The 5${}^{\rm th}$ generation, initiated by the \algname{ProxSkip}  method of \citet{ProxSkip} and its analysis, is characterized by the first theoretical confirmation that  LT is a communication  acceleration mechanism. Inspired by this recent progress, we contribute to the 5${}^{\rm th}$ generation of LT methods by showing that it is possible to enhance them further using {\em variance reduction}. While all previous theoretical results for LT methods ignore the cost of local work altogether, and are framed purely in terms of the number of communication rounds, we show that our methods can be substantially faster in terms of the {\em total training cost} than the state-of-the-art method \algname{ProxSkip} in theory and practice in the regime when local computation is sufficiently expensive. We characterize this threshold theoretically, and confirm our theoretical predictions with empirical results.
\end{abstract}

{\tableofcontents}

\section{Introduction}

Announced in April 2017 in a Google AI blog \citep{FLblog2017}, and citing four foundational papers \citep{FEDLEARN, FEDOPT, FL2017-AISTATS, FL-secure_aggreg} of what was to become a new and rapidly growing interdisciplinary field, {\em federated learning} (FL) constitutes a novel paradigm for training supervised machine learning models. The key idea  is the acknowledgement that increasing amounts of data are being captured and stored on edge devices, such as mobile phones, sensors and hospital workstations, and that moving the data to a  datacenter for centralized processing may be infeasible or undesirable for various reasons, including high energy costs and  data privacy concerns~\citep{FL-big,FL_survey_2020}. FL faces a multitude of challenges which are being actively addressed by the research community.

\subsection{Formalism}
We study the standard optimization formulation of federated learning \citep{FEDOPT,FL2017-AISTATS,FL-big,FieldGuide2021} given by
\begin{equation}\label{eq:P}
\squeeze \min\limits_{x\in \R^{d'}} \phi(x), \qquad  \phi(x)\eqdef  \sum\limits_{i=1}^M \frac{n_i}{n} \phi_i(x),  \qquad \phi_i(x) \eqdef 
\frac{1}{n_i} \sum\limits_{j=1}^{n_i} \phi_{ij}(x),
\end{equation}
where $M$ is the number of clients (devices, machines, workers), $n_i$ is the number of training data points on client $i\in \{1,2,\dots, M\}$, and  $n\eqdef \sum_{i=1}^M n_i$ is the total number of training data points collectively owned by this federation of $M$ clients. Note that $\phi$ is the empirical risk over the federated dataset. 
Perhaps conceptually the simplest method for solving (\ref{eq:P}) is {\em gradient descent (\algname{GD})},
\begin{equation}\squeeze \label{eq:GD}x_{t+1} = x_t - \gamma \nabla \phi(x_t) = x_t - \gamma \sum\limits_{i=1}^M \frac{ n_i}{n} \nabla \phi_i(x_t)  = \sum\limits_{i=1}^M \frac{ n_i}{n}  \left(x_t - \gamma  \nabla \phi_{i}(x_t) \right), \end{equation}
where $\gamma>0$ is the stepsize. It will be useful to describe how \algname{GD}  would be implemented in a federated environment. First, all clients $i\in \{1,\dots, M\}$ in parallel perform a single local gradient step starting from the current global model $x_t$, arriving at the local models $x_{it} \eqdef x_t - \gamma  \nabla \phi_{i}(x_t)$, $i\in \{1,\dots,M\}$. These local models are then communicated to the {\em orchestrating server}, which aggregates them via weighted averaging, arriving at the new global model $x_{t+1} =\sum_{i=1}^M \frac{n_i}{n} x_{it}$. This new model is then broadcast back to all clients, and the process is repeated until a model of sufficient quality is found.

\subsection{Federated averaging}\label{sec:FedAvg}

Proposed by \citet{Povey2015,SparkNet2016,FL2017-AISTATS}, federated averaging (\algname{FedAvg}) is arguably the most  popular method for solving the standard FL formulation (\ref{eq:P}). Motivated by the specific constraints of federated environments, \algname{FedAvg} can be seen as a practical enhancement of \algname{GD} via the simultaneous application of three techniques: a) data sampling (DS), b) client sampling (CS), and c) local training (LT). That is, $$\text{\algname{FedAvg} = \algname{GD}  + (DS + CS + LT)}.$$ We will now briefly describe each of these three \algname{GD}-enhancing techniques separately. 


\begin{itemize} 
\item [(a)] {\bf \algname{GD}  + Data Sampling.} In situations when the local datasets are so large that the computation of the exact local gradients  becomes a bottleneck, it makes sense to approximate them via data sampling. That is, instead of passing through all local data  to compute the local gradient $\nabla \phi_i(x_t)$, each client $i$ computes the gradients $\nabla \phi_{ij}(x_t)$ for $j\in \cD_{it}$ only, where $\cD_{it}$ is a suitably chosen  small-enough subset of the local dataset $\{1,\dots, n_i\}$. These gradients are then used to form gradient estimators $g_{i}(x_t)\approx \nabla \phi_i(x_t)$ which are used to perform  a local \algname{SGD} step on all clients. The rest of the procedure is the same as in the case of \algname{GD}. That is, the local models obtained in this way are sent to the orchestrating server, the server aggregates them via weighted averaging and broadcasts the resulting model back to all clients. Combination of \algname{GD} and DS can be seen as a particular version of \algname{SGD}, where the stochastic gradient estimator is formed from the gradients $\nabla \phi_{ij}(x_t)$ associated with the datapoints $(i,j)$ where $j\in \cup_{i=1}^M \cD_{it}$.
While DS is still an active area of research, it has been studied for a long time, and is in general well  understood \citep{pegasos2,Li2014,csiba2016importance,SGD-AS,nonconvex_arbitrary,ES-SGD-nonconvex}.

\item [(b)] {\bf \algname{GD} + Client Sampling.} 
 In practical federated environments, and especially in cross-device FL \citep{FL-big}, the number of clients  is enormous, they are not all available at all times, and the orchestrating server has limited compute and memory capacity. For these and other reasons, practical FL methods need to work in an environment in which a small subset $\cS_t\subseteq \{1,\dots,M\}$ of the clients is sampled (``participates'') in each communication/aggregation/training round only. Since only the participating clients $i\in \cS_t$ perform a local \algname{GD} step and  communicate the resulting local model to the orchestrating server for aggregation, this induces an error compared to \algname{GD}, which has an adverse effect on the convergence rate. Combination of \algname{GD} and CS can be seen as a particular version of \algname{SGD}, where the stochastic gradient estimator is formed from the gradients $\nabla \phi_{ij}(x_t)$ associated with the datapoints $(i,j)$ where $j\in \cup_{i=1}^{\cS_t} \{1,\dots,n_i\}$. While  CS is still an active area of research, since CS is a special type of DS, much was known about CS long before \algname{FedAvg} was proposed~\citep{SGD-AS,nonconvex_arbitrary}.  Still, CS poses new challenges tackled by the community \cite{Eichner2019semi-cyclicSGD_for_FL,OptClientSampling2020,SGD-AS,ClientSelection-Gauri,Cohort2021}.

\item [(c)] {\bf \algname{GD} + Local Training.} 
In federated learning, the cost of communication between the clients and the orchestrating server forms the key bottleneck. Indeed, in their \algname{FedAvg} paper,  which introduced LT to the world of federated learning, \citet{FL2017-AISTATS} wrote: \begin{quote}{\em \footnotesize ``In contrast\footnote{to datacenter optimization}, in federated optimization communication costs dominate''.}\end{quote} LT is a conceptually simple and surprisingly powerful communication-acceleration technique. The basic idea behind LT is for the clients to perform {\em multiple} local \algname{GD} steps instead of a single step (which is how \algname{GD} operates) before communication and aggregation takes place. The intuitive reasoning used in virtually all papers on this topic is: performing  multiple local \algname{GD} steps results in ``richer'' and ultimately more useful local training in  the sense that fewer communication rounds will {\em hopefully} suffice to finish the training. \citet{FL2017-AISTATS} supported this intuition with ample empirical evidence, and credited LT as the critical component behind the success of \algname{FedAvg}: \begin{quote}{\em\footnotesize ``Thus, our goal is to use additional computation in order to decrease the number of rounds of communication needed to train a model\dots'' ``Communication costs are the principal constraint, and we show a reduction in required communication rounds by 10--100$\times$ as compared to synchronized stochastic gradient descent.'' ``\dots the speedups we achieve are due primarily to adding more computation on each client''. }\end{quote}


 \end{itemize}

\section{Five Generations of Local Training Methods}

\begin{table*}[t]
    \centering
    \scriptsize
    \caption{\footnotesize Five generations of local training (LT) methods summarizing the progress made by the ML/FL community over the span of 7+ years in the understanding of the {\em communication acceleration properties of LT}. }
    \label{tab:comparison}
    \begin{threeparttable}
\begin{tabular}{lc  lll}
{\bf Generation}\tnote{\color{blue}(a)}  & \bf Theory & \bf Assumptions & {\bf Comm.\ Complexity}\tnote{\color{blue}(b)} & \bf Selected Key References \\
\hline
\multirow{3}{*}{1. Heuristic} 
& \xmark & --- & empirical results only & \algname{LocalSGD} \citep{Povey2015}   \\ 
& \xmark & --- & empirical results only & \algname{SparkNet} \citep{SparkNet2016}  \\ 
& \xmark & --- & empirical results only & \algname{FedAvg} \citep{FL2017-AISTATS} \\
\hline
\multirow{2}{*}{2. Homogeneous} 
& \cmark & bounded gradients & sublinear &   \algname{FedAvg} \citep{Li-local-bounded-grad-norms--ICLR2020}\\
& \cmark & bounded grad.\ diversity\tnote{\color{blue}(c)}  & linear but worse than \algname{GD} & \algname{LFGD} \citep{LocalDescent2019}  \\
\hline
\multirow{2}{*}{3. Sublinear} 
& \cmark & standard\tnote{\color{blue}(d)}  & sublinear &  \algname{LGD} \citep{localGD} \\
& \cmark & standard                                    & sublinear &  \algname{LSGD} \citep{localSGD-AISTATS2020}  \\
\hline
\multirow{3}{*}{4. Linear} 
& \cmark & standard & linear but worse than \algname{GD}  & \algname{Scaffold} \citep{SCAFFOLD} \\ 
& \cmark & standard & linear but worse than \algname{GD}  & \algname{S-Local-GD} \citep{LSGDunified2020} \\ 
& \cmark & standard & linear but worse than \algname{GD}  & \algname{FedLin} \citep{FEDLIN} \\
\hline
\multirow{2}{*}{5. Accelerated} 
& \cmark & standard & linear \& better than \algname{GD} &  \algname{ProxSkip}/\algname{Scaffnew} \citep{ProxSkip} \\ 
&\cellcolor{bgcolor2}\cmark & \cellcolor{bgcolor2}standard & \cellcolor{bgcolor2}linear \& \cellcolor{bgcolor2}better than \algname{GD} &  \cellcolor{bgcolor2}\algname{ProxSkip-VR} {\bf [THIS WORK]} \\
\hline
\end{tabular}
  \begin{tablenotes}
        {\tiny
        \item [{\color{blue}(a)}]  Since client sampling (CS) and data sampling (DS) can only {\em worsen}  theoretical communication complexity, our historical breakdown of the literature into 5 generations of LT methods focuses on the full client participation (i.e., no CS) and exact local gradient (i.e., no DS) setting. While some of the referenced methods incorporate CS and DS techniques, these are irrelevant for our purposes. Indeed, from the viewpoint of communication complexity, all these algorithms enjoy best theoretical performance  in the no-CS and no-DS regime.        
        \item [{\color{blue}(b)}] For the purposes of this table, we consider problem (\ref{eq:P}) in the {\em smooth} and {\em strongly convex} regime only. This is because the literature on LT methods struggles to understand even in this simplest (from the point of view of optimization) regime.  
        \item [{\color{blue}(c)}] {\em Bounded gradient diversity} is a uniform bound on a specific notion of gradient variance depending on client sampling probabilities. However, this assumption (as all homogeneity assumptions) is very restrictive. For example, it is not satisfied the standard class of smooth and strongly convex functions. 
        \item [{\color{blue}(d)}]  The notorious FL challenge of handling non-i.i.d.\ data by LT methods was solved  by \citet{localGD} (from the viewpoint of {\em optimization}). From generation 3 onwards, there was no need to invoke any data/gradient homogeneity assumptions.  Handling non-i.i.d.\ data remains a challenge from the point of view of {\em generalization}, typically by considering {\em personalized} FL models.         
        }
    \end{tablenotes}
    \end{threeparttable}
\end{table*}

We now offer several historical comments on the most important developments related to the {\em theoretical} understanding of LT.
To this end,  we have identified 5 distinct generations of LT methods, each with its unique challenges and characteristics. To make the narrative simple, and since we focus on this regime in our paper, we limit our overview to loss functions $\phi_i$ that are $\mu$-strongly convex and $L$-smooth. This is arguably the most  studied class of functions in continuous optimization \citep{NesterovBook}, and for this reason, it presents a valuable litmus test for any theory of LT.


\subsection{Generation 1: Heuristic Age} 
While LT ideas were used in several machine learning domains before \citep{Povey2015,SparkNet2016}, LT truly rose to prominence as a practically potent communication acceleration technique due to the seminal paper of \citet{FL2017-AISTATS}  which introduced the \algname{FedAvg} algorithm. However, no theory was provided in their work, nor in any prior work. LT-based heuristics, i.e., methods without any theoretical guarantees, dominated the initial development of the field up to, and including, the \algname{FedAvg} paper. 

\subsection{Generation 2: Homogeneous Age} 
The first theoretical results for LT methods offering explicit convergence rates relied on various data/gradient {\em homogeneity}\footnote{We use the term {\em homogeneity} to refer to various related assumptions used in the literature, including  {\em bounded gradient norms} \citep{Li-local-bounded-grad-norms--ICLR2020}, {\em bounded gradient variance} \citep{Li2019-local-homogeneous,Yu-local-homogeneous-2019} and {\em bounded gradient diversity} \citep{LocalDescent2019}.} 
assumptions.
The intuitive rationale behind such assumptions comes from the following thought process. In the extreme case when all the local functions $\phi_i$ are {\em identical} (this is often referred to as the {\em homogeneous} or {\em i.i.d.\ data} regime), there is a very simple approach to making \algname{GD}  communication-efficient: push the idea of LT to its extreme by running \algname{GD} on all clients, independently and in parallel, without any communication/synchronization/averaging whatsoever.
Extrapolating from this, it is reasonable to assume that as we increase heterogeneity, taking multiple local steps should still be beneficial as long as we do not take too many steps. Several authors analyzed various LT methods under such assumptions, and obtained rates \citep{LocalDescent2019,Yu-local-homogeneous-2019,Li2019-local-homogeneous,Li-local-bounded-grad-norms--ICLR2020}.  However, bounded dissimilarity assumptions are highly problematic. First, they do not seem to be satisfied even for some of the simplest function classes, such as strongly convex quadratics \citep{localGD,localSGD-AISTATS2020}, and moreover, it is well known that practical FL datasets are highly heterogeneous/non-i.i.d. \cite{FL2017-AISTATS,FL-big}. So, analyses relying on such strong assumptions are both mathematically questionable, and practically irrelevant.

\subsection{Generation 3: Sublinear Age}
The third generation of LT methods is characterized by the successful removal of the bounded dissimilarity assumptions from the convergence theory.  \citet{localGD} first achieved this breakthrough by studying the simplest LT method: local gradient descent (\algname{LGD}) (i.e., a simple combination of \algname{GD}  and LT). While works belonging to this generation elevated LT to the same theoretical footing as \algname{GD}  in terms of the assumptions, which marked an important milestone in our understanding of LT,  unfortunately, the obtained communication complexity theory of \algname{LGD}  is pessimistic when compared to vanilla \algname{GD}. Indeed, the inclusion of LT did {\em not} lead to an improvement upon the communication complexity of vanilla \algname{GD}. Moreover, while \algname{GD} enjoys a linear communication complexity (in the smooth and strongly convex regime), the communication complexity of \algname{LGD} is {\em sublinear}. In a follow-up work, \citet{localSGD-AISTATS2020} later analyzed \algname{LGD} in combination with DS as well. \citet{woodworth2020minibatch} and \citet{glasgow2022sharp} provided lower bounds for \algname{LGD} with DS showing that it is not better than minibatch \algname{SGD} in heterogeneous setting. See the work of \citet{LFPM} for a fixed-point theory viewpoint.

\subsection{Generation 4: Linear Age} 
The fourth generation of LT methods is characterized by the effort to design {\em linearly} converging variants of LT algorithms. In order to achieve this, it was important to tame the adverse effect of the so-called {\em client drift} \citep{SCAFFOLD}, which was identified as the culprit of the worse-than-\algname{GD} theoretical performance of the previous generation of LT methods. The first LT-based method that successfully tamed  client drift, and as a result obtained a linear convergence rate, was \algname{Scaffold} \citep{SCAFFOLD}. Several alternative approaches to obtaining the same effect were later proposed by \cite{LSGDunified2020} and  \citet{FEDLIN} . While obtaining a linear rate for LT methods under standard assumptions  was a major achievement, the communication complexity of these methods is still somewhat worse\footnote{Both \algname{GD}, and LT methods such as \algname{Scaffold} \citep{SCAFFOLD},  \algname{S-Local-GD} \citep{LSGDunified2020} and  \algname{FedLin} \citep{FEDLIN} enjoy the linear rate $\cO(\kappa \log \frac{1}{\varepsilon})$, where $\kappa$ is a condition number. However, this condition number is in general slightly worse for the LT methods.} than that of vanilla \algname{GD}, and is at best equal to that of \algname{GD}.

\subsection{Generation 5: Accelerated Age} Finally, the fifth generation of LT methods was initiated recently by \citet{ProxSkip} with their \algname{ProxSkip} method which enjoys {\em accelerated  communication complexity}. Acceleration comes from the LT steps  coupled with a new client drift reduction technique and a probabilistic approach to deciding whether communication takes place or not.  \citet{ProxSkip} first reformulate  \eqref{eq:P}  into the equivalent consensus form 
\begin{equation} \label{eq:Composite}\min_{x\in \R^{d}} f(x) + r(x),\end{equation}
where $d=Md'$, $x=(x_1,\dots,x_M)\in \R^{d}$, and 
\begin{equation}\label{eq:consensus} 
\squeeze f(x) \eqdef \sum\limits_{i=1}^M \frac{n}{n_i} \phi_i(x_i), \qquad r(x) = \begin{cases} 0 & \text{ if } x_1=\dots=x_M, \\ +\infty & \text{ otherwise.} \end{cases}\end{equation}
The \algname{ProxSkip} method is a randomized variant of proximal gradient descent (\algname{ProxGD})~\citep{Nesterov_composite2013,beck-book-first-order}  for solving \eqref{eq:Composite}, 
with the proximity operator of $r$, given by $$\prox_{r}(x) \eqdef \arg \min_{y} \left(r(y) + \frac{1}{2}\|y-x\|^2\right),$$ being	 evaluated in each iteration with probability $p\in (0,1]$ only. Remarkably, \citet{ProxSkip} showed that it is possible to choose $p$ as low as $1/\sqrt{\kappa}$, where $\kappa$ is the condition number of $f$, without this worsening the rate of its parent method \algname{ProxGD}. In summary, \algname{ProxSkip} lets the $M$ clients perform $\sqrt{\kappa}$ local gradient steps in expectation, followed by the evaluation of the prox of $r$, which in the case of the  consensus reformulation of \eqref{eq:P}  means averaging across all $M$ nodes, i.e., communication.

\section{ProxSkip-VR: A General Variance Reduction Framework for ProxSkip}\label{sec:main_theory}
\begin{algorithm*}[t]
	\caption{\algname{ProxSkip-VR}}
	\label{alg:ProxSkip-VR}
	\begin{algorithmic}[1]
		\STATE {\bf Parameters:} stepsize $\gamma > 0$, probability $p\in (0,1]$, initial iterate $x_0\in \R^d$, {\red initial control vector $y_0\in \R^d$}, {\color{blue} initial gradient shift $h_0 \in \R^d$}, number of iterations $T\geq 1$
		\FOR{$t=0,1,\dotsc,T-1$}
		\STATE {\red$g_t = g(x_t,y_t,\xi_t)$}\hfill $\diamond$  Sample  $\xi_t$ and construct an unbiased estimator of $\nabla f(x_t)$ 
		\STATE $\hat x_{t+1} = x_t - \gamma ( {\red g_t} - {\color{blue} h_t})$ \hfill $\diamond$ Take a gradient-type step adjusted via the {\color{blue} shift $h_t$}
		\STATE {\red Construct new control vector  $y_{t+1}$}
		\STATE Flip a coin $\theta_t \in \{0,1\}$ where $\mathop{\rm Prob}(\theta_t =1) = p$ \hfill $\diamond$ Decides whether to skip the prox or not
		\IF{$\theta_t=1$} 
		\STATE  $x_{t+1} = \prox_{\frac{\gamma}{p}r}\bigl(\hat x_{t+1} - \frac{\gamma}{p}{\color{blue}  h_t} \bigr)$ \hfill $\diamond$ Apply prox, but only with probability $p$
		\ELSE
		\STATE $x_{t+1} = \hat x_{t+1}$ \hfill $\diamond$ Skip the prox!
		\ENDIF
		\STATE ${\color{blue} h_{t+1}} = {\color{blue} h_t} + \frac{p}{\gamma}(x_{t+1} - \hat x_{t+1})$ \hfill $\diamond$ Update the {\color{blue} shift $h_t$}
		\ENDFOR
	\end{algorithmic}
\end{algorithm*}

In this work we contribute to the fifth generation of LT methods by extending the work of \citet{ProxSkip} to allow for a very large family    of gradient estimators, including variance reduced (VR) ones \citep{SVRG, SAGA, L-SVRG, DIANA}. 

Like \algname{ProxSkip}, our method  \algname{ProxSkip-VR} (Algorithm~\ref{alg:ProxSkip-VR}) is aimed to solve the composite problem \eqref{eq:Composite} in a more general setting (see Assumptions~\ref{ass:L-smoothness}--\ref{ass:Reg}), with the special structure \eqref{eq:consensus} coming from the consensus reformulation  being a special case only. Our method differs from \algname{ProxSkip} in that we replace the gradient $\nabla f(x_t)$ by an unbiased estimator $g_t = g(x_t,{\red y_t},\xi_t)$, where $\xi_t$ is the source of randomness controlling unbiasedness and  {\red $y_t$ is a control vector} whose role is to progressively reduce the variance of the estimator, so that $$\Exp{g_t \;|\; x_t, {\red y_t}} = \nabla f(x_t).$$

There are several motivations behind this endeavor. First, it is a-priori not clear whether the novel proof technique employed by \citet{ProxSkip} can be combined with the proof techniques used in the analysis of VR methods, and hence it is scientifically significant to investigate the possibility of such a merger of two strands of the literature. We show  that this is possible. Second, marrying VR estimators with \algname{ProxSkip} can lead to novel system architectures  which are more elaborate than the simplistic client-server architecture (see Section~\ref{sec:tree}).  Lastly, while researchers contributing to generations 1--4 of LT methods were preoccupied with trying to close the gap on \algname{GD} in terms of communication efficiency, they {\em ignored} the number of the local steps appearing in their algorithms, and reported their bounds primarily in terms of the number of communication rounds. Bounds reported this way make complete sense in the scenario when the cost of local work (e.g., one \algname{SGD} step w.r.t.\ a single data point), say $\delta$, is negligible compared to the cost of communication, which can w.l.o.g.\ assume to be 1, and when the number of local steps is small. With the advent of the fifth generation of LT methods, we can (to a large degree) stop worrying about communication efficiency, and can now ask more refined questions, such as: \begin{quote}{\em Are there gradient estimators which, when combined with \algname{ProxSkip}, lead to faster algorithms in terms of the total cost, which includes the communication cost as well as the cost of local training?} \end{quote} We give an affirmative answer to the question in Sections~\ref{sec:tree} and \ref{sec:experiments}.


\subsection{Standard assumptions}

We assume throughout that $f$ is differentiable, and let $D_f(x,y) \eqdef f(x) - f(y) - \langle \nabla f(y), x-y \rangle $ denote the Bregman divergence of $f$. Throughout the work we make the following assumptions:

\begin{assumption}[$L$-smoothness]
	\label{ass:L-smoothness}
	There exists $L > 0$ such that $2D_f(x,y) \leq L \left\|x-y\right\|^{2}$ for all $x,y \in \mathbb{R}^{d}
$.  \end{assumption}

\begin{assumption}[$\mu$-convexity]
	\label{ass:mu-strongly-convex}
	There exists $\mu > 0$ such that $\mu \left\|x-y\right\|^{2} \leq 2D_f(x,y)$ for all $x,y \in \mathbb{R}^{d}$. 	
  \end{assumption}

\begin{assumption}
	\label{ass:Reg}
	The regularizer $r:\R^d\to \R\cup \{+\infty\}$ is proper, closed and convex. \end{assumption}

Under the above assumptions, \eqref{eq:Composite} has a unique minimizer $x_{\star}$. Let $h_{\star} \eqdef \nabla f(x_{\star})$.


\subsection{Modelling variance reduced gradient estimators}

Our next assumption, initially introduced by~\citet{gorbunov2020unified}, postulates several parametric inequalities characterizing the behavior and ultimately the quality of a gradient estimator.  Similar assumptions appeared later in~\citep{LSGDunified2020, gorbunov2020linearly}.
\begin{assumption}
	\label{sigma_t}
	Let $\{x_t\}$ be iterates produced by \algname{ProxSkip-VR}. First, we assume that the stochastic gradients $g_{t}=g(x_t,y_t,\xi_t)$ are unbiased for all $t\geq 0$, namely \begin{equation}\label{eq:unbiased}\Exp{g_t \;|\; x_t, y_t} =\nabla f(x_{t}) .\end{equation}
Second, we assume that there exist non-negative constants $A, B, C, \tilde{A}, \tilde{B}, \tilde{C}$, with $\tilde{B}<1$, and a nonnegative mapping $y_t \mapsto \sigma(y_t) \eqdef \sigma_t$  such that the following two relations hold for all $t\geq 0$,
\begin{eqnarray}
	\Exp{ \left\|g_{t}-\nabla f(x_{*})\right\|^{2} \;|\; x_t, y_t} &\leq& 2 A D_{f} (x_{t}, x_{*})+B \sigma_{t} +C,\label{eq:sigma-1}\\
	\Exp{ \sigma_{t+1} \;|\; x_t, y_t } &\leq & 2 \tilde{A} D_{f}(x_{t}, x_{*}) +\tilde{B} \sigma_{t} +\tilde{C}.\label{eq:sigma_2}
\end{eqnarray}
\end{assumption}

Assumption~\ref{sigma_t} covers a very large collection of  gradient estimators, including an infinite variety of subsampling/minibatch estimators, gradient sparsification and quantization estimators, and their combinations; see \citep{gorbunov2020unified} for examples. VR estimators are characterized by $C=\tilde{C}=0$; most non-VR estimators by $\tilde{A}=\tilde{B}=\tilde{C}=B=0$ and $C>0$~\citep{gower2019sgd}. 

\subsection{Main result}

We are now ready to formulate our main result.

\begin{theorem}\label{thm:main}
Let Assumptions~\ref{ass:mu-strongly-convex} and \ref{ass:Reg} hold, and let $g_t$ be a gradient estimator satisfying Assumption~\ref{sigma_t}. If $B>0$, choose any  $\MM >\nicefrac{B}{(1-\tilde{B})}$ and $\beta= \nicefrac{(B + \MM\tilde{B})}{\MM}$. If $B=0$, let $\MM=0$ and $\beta=\tilde{B}$. Choose stepsize  $
	0<\gamma \leq \min \left\{\nicefrac{1}{\mu}, \nicefrac{1}{(A+\MM \tilde{A}) }\right\}.$
	Then the iterates of \algname{ProxSkip-VR} for any $p\in (0,1]$ satisfy
	\begin{equation}\label{eq:main_thm}
	\squeeze	\Exp{ \Psi_{T} } \leq \max \left\{(1-\gamma \mu)^{T},\beta^{T},(1-p^2)^T\right\} \Psi_{0}+\frac{\left(C+\MM \tilde{C}\right) \gamma^{2}}{\min \left\{\gamma \mu,p^2, 1 - \beta \right\}},
	\end{equation}
	where the Lyapunov function  is defined by $$\squeeze	\Psi_{t} \eqdef \|x_{t} - x_{\star}\|^2 + \frac{\gamma^2}{p^2}\|h_t - h_{\star}\|^2 +\gamma^{2}  \MM \sigma_{t}.$$
\end{theorem}

\begin{table*}[t]
    \centering
    \scriptsize
    \caption{\footnotesize Special cases of \algname{ProxSkip-VR}, depending on the choice of the gradient estimator $g_t$.}
    \label{tab:comparison2}
    \begin{threeparttable}
\begin{tabular}{lllc}
 \bf Estimator of $\nabla f$ & \begin{tabular}{c}\bf Communication Complexity\\\bf of \algname{ProxSkip-VR} \end{tabular}&\begin{tabular}{c}\bf Iteration Complexity\\\bf of \algname{ProxSkip-VR} \end{tabular}&\bf \begin{tabular}{c}\bf Corollaries\\\bf of Theorem~\ref{thm:main} \end{tabular} \\
\hline
 \algname{GD}\tnote{\color{blue}(b)} & $\mathcal{O}\left(\sqrt{\nicefrac{L}{\mu}}\log\nicefrac{1}{\varepsilon}\right)$ &$\mathcal{O}\left(\nicefrac{L}{\mu}\log\nicefrac{1}{\varepsilon}\right)$&Theorem~\ref{thm:proxskip}  \\ 
\hline
 \algname{SGD}\tnote{\color{blue}(c)} & $\mathcal{O} \left(\left(\sqrt{\nicefrac{A}{\mu}}+\sqrt{\nicefrac{2 C}{\varepsilon \mu^{2}}}\right) \log \nicefrac{1}{\varepsilon}\right)$  &$\mathcal{O} \left(\left(\nicefrac{A}{\mu}+\nicefrac{2 C}{\varepsilon \mu^{2}}\right) \log \nicefrac{1}{\varepsilon}\right)$&Theorem~\ref{thm:sproxskip}  \\ 
 \hline
\cellcolor{bgcolor2}\algname{HUB} {\bf [NEW]} &\cellcolor{bgcolor2}$\mathcal{O}\left(\sqrt{\nicefrac{L_{\max}}{\mu}\left(1+\nicefrac{\omega}{\tau}\right)}\log\nicefrac{1}{\varepsilon}\right)$ &\cellcolor{bgcolor2}$\mathcal{O}\left(\nicefrac{L_{\max}}{\mu}\left(1+\nicefrac{\omega}{\tau}\right)\log\nicefrac{1}{\varepsilon}\right)$ &\cellcolor{bgcolor2}Theorem~\ref{thm:QLSVRG}  \\ 
\cellcolor{bgcolor2}\algname{LSVRG} {\bf [NEW]} &\cellcolor{bgcolor2}$\mathcal{O}\left(\sqrt{\nicefrac{L(\tau)}{\mu}}\log\nicefrac{1}{\varepsilon}\right)$ &\cellcolor{bgcolor2}$\mathcal{O}\left(\nicefrac{L(\tau)}{\mu}\log\nicefrac{1}{\varepsilon}\right)$&\cellcolor{bgcolor2}Corollary~\ref{thm:lsvrg-proxskip}  \\ 
 \cellcolor{bgcolor2}\algname{Q} {\bf [NEW]} & \cellcolor{bgcolor2}$\mathcal{O}\left(\sqrt{\nicefrac{L_{\max}}{\mu}\left(1+\nicefrac{\omega}{M}\right)}\log\nicefrac{1}{\varepsilon}\right)$ &\cellcolor{bgcolor2}$\mathcal{O}\left(\nicefrac{L_{\max}}{\mu}\left(1+\nicefrac{\omega}{M}\right)\log\nicefrac{1}{\varepsilon}\right)$ &\cellcolor{bgcolor2}Corollary~\ref{thm:rand-diana-proxskip}  \\ 
\hline
\end{tabular}
  \begin{tablenotes}
        {\tiny
        \item [{\color{blue}(a)}]  Any estimator satisfying Assumption~\ref{sigma_t}
        \item [{\color{blue}(b)}]   \algname{ProxSkip-VR} with the \algname{GD} estimator reduces to the \algname{ProxSkip} method of \citet{ProxSkip}
         \item [{\color{blue}(c)}]     \algname{ProxSkip-VR} with the \algname{SGD} estimator satisfying Assumption~\ref{Expected_smoothness} reduces to the \algname{SProxSkip} method of \citet{ProxSkip}
         \item [{\color{blue}(d)}]   $L(\tau) \eqdef \frac{m-\tau}{\tau(m-1)} L_{\max}+\frac{m(\tau-1)}{\tau(m-1)} L$, where $\tau$ is the mini-batch size and $m$ is the number of clients belonging to one hub 
        }
    \end{tablenotes}
    \end{threeparttable}
\end{table*}
 
\subsection{Two examples of gradient estimators}
Here we give two illustrating examples of  estimators satisfying 
Assumption~\ref{sigma_t}. 
\begin{theorem}[\algname{GD} estimator]
	\label{thm:proxskip}
Let  Assumption~\ref{ass:L-smoothness}, \ref{ass:mu-strongly-convex} and \ref{ass:Reg} hold. Then for the trivial estimator $g_t=\nabla f(x_t)$,  Assumption~\ref{sigma_t} holds with the following parameters:
\begin{align*}
	A = L, \quad B = 0, \quad C = 0, \quad \tilde{A} = 0, \quad \tilde{B} = 0,  \quad \tilde{C} = 0, \quad \sigma_t \equiv 0.
\end{align*}
Choose a stepsize satisfying $	0<\gamma \leq  \nicefrac{1}{L}.$ Then the iterates of \algname{ProxSkip-VR} for any  $p\in (0,1]$ satisfy 
	\begin{equation}\label{eq:GD-xx}
	\Exp{\Psi_{T}} \leq \max \left\{(1-\gamma \mu)^{T},(1-p^2)^T\right\} \Psi_{0},
\end{equation}
where $$\squeeze	\Psi_{t} \eqdef \|x_{t} - x_{\star}\|^2 + \frac{\gamma^2}{p^2}\|h_t - h_{\star}\|^2.$$ Let $\gamma = \nicefrac{1}{L}$ and $p = \sqrt{\nicefrac{\mu}{L}}$ then the communication and iteration complexities of \algname{ProxSkip-VR} are 
$$\squeeze \# comms = \mathcal{O}\left(\sqrt{\frac{L}{\mu}}\log\frac{1}{\varepsilon}\right), \qquad \# iters = \mathcal{O}\left(\frac{L}{\mu}\log\frac{1}{\varepsilon}\right),$$ 
respectively.
\end{theorem}

This recovers the result obtained by~\citet{ProxSkip} for their \algname{ProxSkip} method.  The next assumption holds for virtually all (non-VR) estimators based on subsampling~\citep{gower2019sgd}. 
   \begin{assumption}[Expected smoothness]
   	\label{Expected_smoothness}
   	We say that an unbiased estimator $g(x;\xi):\R^d\to\R^d$ of the gradient $\nabla f(x)$ satisfies the expected smoothness inequality if there exists $A^{\prime\prime} >0$ such that
   	\begin{align*}
   		\Exp{\|g(x;\xi) - g(x_{\star};\xi)\|^2} \leq 2A^{\prime\prime}D_f(x,x_{\star}),\quad \forall x\in \mathbb{R}^d.
   	\end{align*}	   	
   \end{assumption}

\begin{theorem}[\algname{SGD} estimator]
	\label{thm:sproxskip}
	Let	$g(x,\xi)$ satisfy Assumption~\ref{Expected_smoothness} and define $g_t \eqdef g(x_t,\xi_t)$, where $\xi_t$ is chosen independently at time $t$.  Then Assumption~\ref{sigma_t} holds with the following parameters:
	\begin{align*}
		A = 2A^{\prime\prime}, \quad B = 0, \quad C = 2{\rm Var}(g(x_{\star},\xi)), \quad \tilde{A} = 0, \quad \tilde{B} = 0,  \quad \tilde{C} = 0, \quad \sigma_t \equiv 0.
	\end{align*}
 Moreover, assume that Assumption~\ref{ass:mu-strongly-convex} holds.	 Choose stepsize $	0<\gamma \leq \min \left\{\nicefrac{1}{\mu}, \nicefrac{1}{A}\right\}.$ Then the iterates of \algname{ProxSkip-VR} for any  probability $p\in (0,1]$ satisfy 
	\begin{equation}
		\squeeze
		\Exp{\Psi_{T}} \leq \max \left\{(1-\gamma \mu)^{T},(1-p^2)^T\right\} \Psi_{0} + \gamma^2 \frac{2{\rm Var}(g(x_{\star},\xi))}{\min\left\lbrace \gamma\mu,p^2 \right\rbrace},\label{eq:GD-yy}
	\end{equation}
	where the Lyapunov function is defined by $$\squeeze	\Psi_{t} \eqdef \|x_{t} - x_{\star}\|^2 + \frac{\gamma^2}{p^2}\|h_t - h_{\star}\|^2.$$ If we choose $\gamma=\min \left\{\nicefrac{1}{A}, \nicefrac{\varepsilon \mu}{2 C}\right\}$ and $p=\sqrt{\gamma \mu}$ then the communication and iteration complexities of  \algname{ProxSkip-VR}  are 
	$$\squeeze \# comms = \mathcal{O} \left(\left(\sqrt{\frac{A}{\mu}}+\sqrt{\frac{C}{\varepsilon \mu^{2}}}\right) \log \frac{1}{\varepsilon}\right), \qquad \# iters = \mathcal{O} \left(\left(\frac{A}{\mu}+\frac{C}{\varepsilon \mu^{2}}\right) \log \frac{1}{\varepsilon}\right),$$ 
	respectively.
\end{theorem}
This recovers the result obtained by~\citet{ProxSkip} for their stochastic variant of \algname{ProxSkip}, which they call \algname{SProxSkip}.

\section{New FL Architecture: Regional Hubs Connecting the Clients to the Server} \label{sec:tree}

We  now illustrate the versatility of our \algname{ProxSkip-VR} framework by designing a new ``FL architecture'' and proposing an algorithm that can efficiently operate in this setting.

In particular, we consider the situation where the clients are clustered (e.g., based on region), and where {\em a hub} is placed in between each cluster and the central server. Clients communicate with their regional hub only, which can communicate with the central server (see Figure~\ref{logo}). There are $M$ hubs, hub $i$ handles $n_i$ clients, and client $j$ associated with hub $i$ owns loss function $\phi_{ij}$. 

 \begin{figure}[!h]
	\centering
	\includegraphics[width=3in]{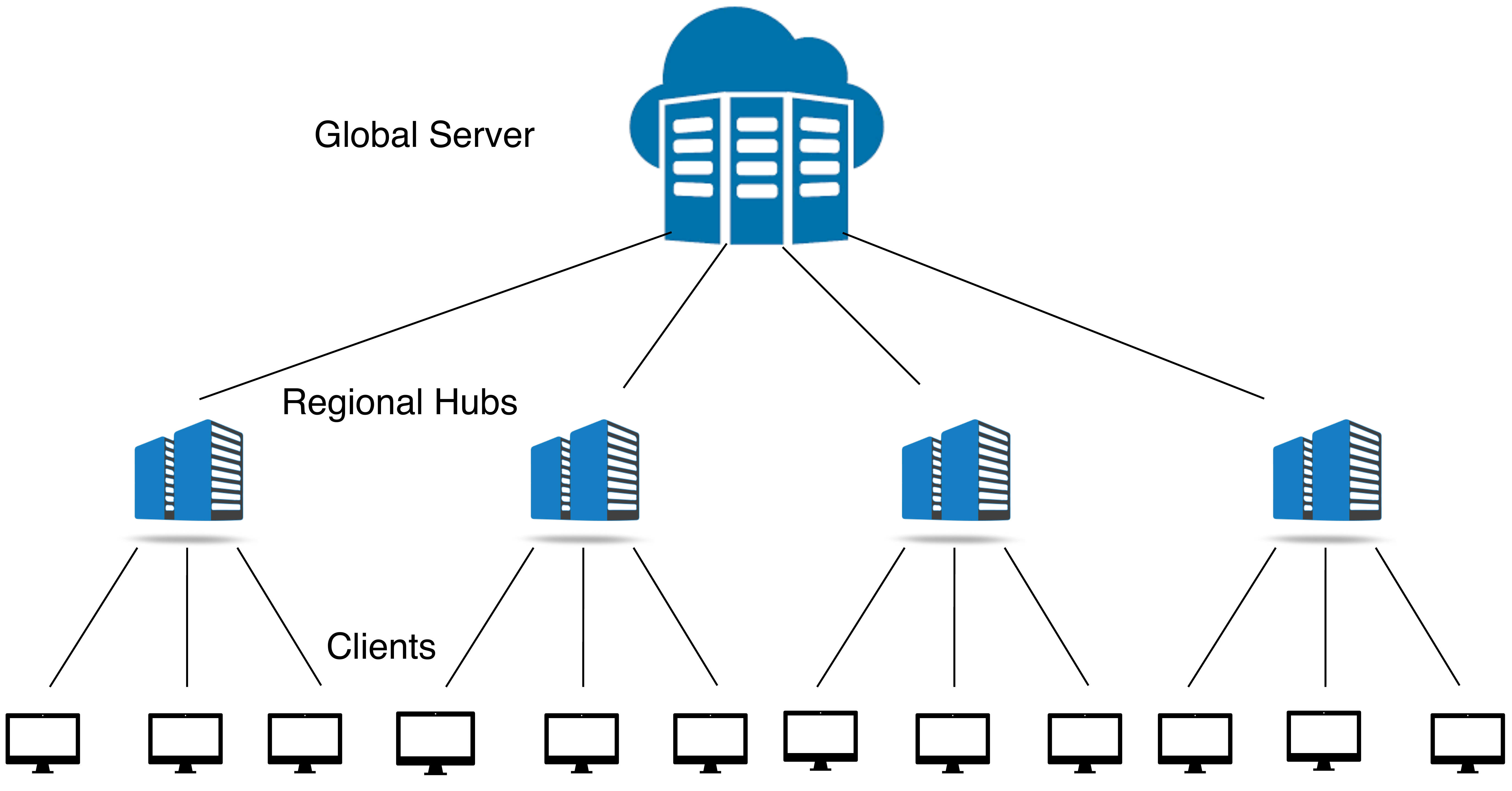}
	\caption{Server-hubs-clients FL architecture with 4 hubs and 12 clients.}
	\label{logo}
\end{figure}

Mathematically, this can be modeled by problem \eqref{eq:P}. In this situation, we care about two sources of communication cost: the server and the hubs, and between the hubs and the clients. We propose to handle this via {\em local training (LT)}  between the server and the hubs, and via {\em client sampling (CS)} and  {\em compressed communication (CC)}  between the hubs and the clients. Algorithmically, from the server-hubs perspective, we are applying a particular variant of \algname{ProxSkip-VR}  to \eqref{eq:Composite}--\eqref{eq:consensus}, where $\phi_i$ is the aggregate loss handles by hub $i$. This takes care of communication efficiency between the server and the hub. Note also that we need not worry about partial participation of hubs, as these are designed to be always available. However, in this situation, it is costly for hub $i$ to compute the gradient of $\phi_i$  as this involves communication with all the clients it handles. 

\subsection{Handling of  client sampling (CS) and compressed communication (CC)}
In order to alleviate this burden, we propose a combination of CS and CC. However, we need to be very careful about how to do this. Indeed, both CS and CC, even when applied in isolation, and without \algname{ProxSkip} in the mix, can lead to a substantial slowdown in convergence. For example, one will typically lose linear convergence in the strongly convex regime. However, techniques for preserving linear convergence in the presence of CS and CC exist: this is what variance reduction strategies are designed to do. For example,  \algname{LSVRG} \citep{hofmann2015variance,L-SVRG}
is a VR technique for reducing the variance due to CS, and \algname{DIANA} \citep{DIANA} is a VR technique for reducing the variance due to CC. However, we are not aware of any VR method that combines CS (applied first) and CC (applied second). 

We now propose such a technique. In iteration $t$, every hub $i\in \{1,2,\dots,M\}$ selects a random subset $\cS_t^i\subseteq \{1,2,\dots,n_i\}$ of the clients it handles of cardinality $\tau_i$, chosen uniformly at random, and estimates the hub gradient via
 \begin{equation}
 	\label{eq:hub_grad}
 \squeeze	\nabla \phi_i(x_t) \approx g_t^i \eqdef \frac{1}{|\cS_t^i|} \sum \limits_{j \in \cS_t^i} \cQ_t^{ij} \left( \nabla \phi_{ij}(x_t) - \nabla \phi_{ij}(y_t) \right) + \nabla \phi_i(y_t),
 \end{equation}
where $\cQ_t^{ij}:\R^{d'}\to\R^{d'}$ is a randomized compression (e.g., sparsification or quantization) operator~\citep{alistarh2017qsgd,DCGD,DIANA2,Cnat,Artemis2020}, i.e., a mapping satisfying 	
 \begin{align}
 	\label{compress}
		\Exp{Q_t^{ij}(x)}=x, \quad \Exp{ \|Q_t^{ij}(x)-x\|^{2} } \leq \omega\left\|x\right\|^{2}, \quad \forall x\in \R^{d'},
	\end{align}  
and the control vector $y_t$ is updated probabilistically as follows:
 \begin{equation}\label{eq:LSVRG-step}
	y_{t+1}=\left\{\begin{array}{lll}
		x_{t} & \text { with probability } & q \\
		y_{t} & \text { with probability } & 1-q
	\end{array}\right..
\end{equation}
The global gradient estimator (a vector in $\R^{Md'}$), which we call \algname{HUB}, is construcuted as a concatenation of the above hub estimators:
\begin{equation}\squeeze \nabla f(x_t) \eqdef \left(\frac{n_i}{n} \nabla \phi_i(x_t)\right)_{i=1}^M \approx g_t \eqdef g(x_t,y_t,\xi_t)\eqdef \left(\frac{n_i}{n} g_t^i\right)_{i=1}^M,\label{eq:U*G(SY*(S*Yddd}\end{equation}
where $\xi_t$ represents the combined randomness from the compressors $\{\cQ_t^{ij}\}$ and random sets $\{\cS_t^i\}$. 

In order to analyze \algname{ProxSkip-VR} in the consensus form, from now on we assume that $n_i=m = \nicefrac{n}{M}$ and $\tau_i=\tau \in \{1,2,\dots,m\}$ for all $i$, and rely on a slightly different, more general reformulation: 
\begin{equation*}
\squeeze	\min \limits_{x \in \mathbb{R}^{d}} \frac{1}{m}\sum\limits_{j=1}^{m}\widetilde{\phi}_j(x)+r(x), \quad \widetilde{\phi}_j(x)\eqdef\frac{1}{M}\sum\limits_{i=1}^{M}  \phi_{ij}\left(x_{i}\right), \quad r(x)\eqdef \begin{cases}0 & \text { if } x_{1}=\cdots=x_{M}, \\ +\infty & \text { otherwise.}\end{cases}
\end{equation*}

 Our proposed method \algname{ProxSkip-HUB} is \algname{ProxSkip-VR} combined with the novel \algname{HUB} estimator \eqref{eq:U*G(SY*(S*Yddd}, applied to the above reformulation; see Algorithm~\ref{alg:ProxSkip-HUB}. 
 
 \begin{algorithm*}[!th]
	\caption{\algname{ProxSkip-HUB}}
	\label{alg:ProxSkip-HUB}
	\begin{algorithmic}[1]
		\STATE {\bf Input}: stepsize $\gamma > 0$, probabilities $p>0$, $q>0$, initial iterate $x_0\in \R^d$, initial shift ${\red y_0}\in \R^d$, initial control variate ${\color{blue} h_0} \in \R^d$, number of iterations $T\geq 1$
		\FOR{$t=0,1,\dotsc,T-1$}
		\STATE broadcast $x_t$ to all clients
		\FOR{$i \in S_t$}
		\STATE $ \hat{\Delta}_{t}^{i}=Q\left(\nabla f_i(x_t) - \nabla f_i({\red y_t})\right) $ \hfill $\diamond$ Apply compression operator
		\ENDFOR
		\STATE $\hat{\Delta}_{t}=\frac{1}{\tau} \sum_{i \in S_t} \hat{\Delta}_{t}^{i}$
		\STATE $\hat{g}=\hat{\Delta}_{t} + \nabla f({\red y_t})$
		
		\STATE $\hat x_{t+1} = x_t - \gamma (\hat{g}_t - {\color{blue} h_t})$ \hfill $\diamond$ Take a gradient-type step adjusted via the control variate ${\red h_t}$
		\STATE Flip a coin $\theta_t \in \{0,1\}$ where $\mathop{\rm Prob}(\theta_t =1) = p$ \hfill $\diamond$ To decide whether to skip the prox or not
		\IF{$\theta_t=1$} 
		\STATE  $x_{t+1} = \prox_{\frac{\gamma}{p}\psi}\bigl(\hat x_{t+1} - \frac{\gamma}{p}{\color{blue} h_t} \bigr)$ \hfill $\diamond$ Apply prox, but only very rarely! (with probability $p$)
		\ELSE
		\STATE $x_{t+1} = \hat x_{t+1}$ \hfill $\diamond$ Skip the prox!
		\ENDIF
		\STATE ${\red h_{t+1}} = {\red h_t} + \frac{p}{\gamma}(x_{t+1} - \hat x_{t+1})$ \hfill $\diamond$ Update the control variate ${\color{blue} h_t}$
		
		\STATE $y_{t+1}=\left\{\begin{array}{lll}
			x_{t} & \text { with probability } & q \\
			y_{t} & \text { with probability } & 1-q
		\end{array}\right.$ \hfill $\diamond$ Update the shift ${\red y_t}$
		\ENDFOR
	\end{algorithmic}
\end{algorithm*}

\subsection{Theory for ProxSkip-HUB}
In the following result we first claim that the above estimator satisfies Assumption~\ref{sigma_t} with $C=\tilde{C}=0$ (i.e., it is variance-reduced), and the rest of the claim follows by application of our general theorem Theorem~\ref{thm:main}.

\begin{theorem}\label{thm:QLSVRG}
Assume that $\nabla \widetilde{\phi}_j$ is $L_j$-smooth for all $j$ and let Assumptions~\ref{ass:mu-strongly-convex} and \ref{ass:Reg} hold. Then for the gradient estimator \eqref{eq:U*G(SY*(S*Yddd}, Assumption~\ref{sigma_t} holds with the following constants:
\begin{align*}
\squeeze 		
A = 4\left(L(\tau)+\frac{\omega}{\tau}L_{\max}\right) , 
\quad B = 4\left(1+\frac{\omega}{\tau}\right), \quad C = 0, \quad \tilde{A} = q L_{\max},\quad \tilde{B} = 1-q, \quad \tilde{C} = 0,
\end{align*}
and $\sigma_{t} \eqdef \sigma(y_t)$, $\sigma(y) \eqdef \frac{1}{m}\sum_{j=1}^{m} \|\nabla \widetilde{\phi}_j(y) - \nabla \widetilde{\phi}_j(x_{\star})\|^2$,	
 $L_{\max} \eqdef \max_j L_j $.  Set $\MM=\nicefrac{2B}{(1-\tilde{B})}$ and $
	0<\gamma \leq \min \left\{\nicefrac{1}{\mu}, \nicefrac{1}{(A+\MM\tilde{A}) }\right\}.$
Then the iterates of \algname{ProxSkip-VR} for any $p\in (0,1]$ satisfy 
		\begin{align*}
	\squeeze 	\Exp{\Psi_{T}} \leq \max \left\{(1-\gamma \mu)^{T},(1-p^2)^T,\left(1-\frac{q}{2}\right)^{T}\right\} \Psi_{0},
	\end{align*}
	where the Lyapunov function is defined by $$\squeeze	\Psi_{t} \eqdef \|x_{t} - x_{\star}\|^2 + \frac{\gamma^2}{p^2}\|h_t - h_{\star}\|^2 + \gamma^{2} \frac{8}{q}\left(1+\frac{\omega}{\tau}\right) \sigma_{t}.$$
\end{theorem}

We now consider two special cases. In the first, we specialize to the no compression regime, and in the second, to the full participation regime.

\begin{corollary}[No compression]
			\label{thm:lsvrg-proxskip}
	If we do not use compression (i.e., $\omega = 0$), then the communication and iteration complexities are $$\squeeze \# comms = \mathcal{O}\left(\sqrt{\frac{L_{\max}}{\mu}}\log \frac{1}{\varepsilon}\right), \qquad \# iters = \mathcal{O}\left(\frac{L_{\max}}{\mu}\log \frac{1}{\varepsilon}\right),$$ respectively. However, if we use the estimator~\eqref{eq:hub_grad} in Theorem~\ref{thm:main} directly, then the communication and iteration complexities are
	$$\squeeze \# comms = \mathcal{O}\left(\sqrt{\frac{L(\tau)}{\mu}}\log \frac{1}{\varepsilon}\right), \qquad \# iters = \mathcal{O}\left(\frac{L(\tau)}{\mu}\log \frac{1}{\varepsilon}\right),$$ where $L(\tau)\eqdef \frac{m-\tau}{\tau(m-1)} L_{\max}+\frac{m(\tau-1)}{\tau(m-1)} L$.
\end{corollary}
Notice that $L_{\max}\geq L$, and that $L(\tau)=L_{\max}$ for $\tau=1$ and $L(\tau) = L$ for $\tau=m$. Moreover, $L(\tau)$ decreases as the minibatch size $\tau$ increases. 

\begin{corollary}[No client sampling]
		\label{thm:rand-diana-proxskip}
	If we do not use client sampling (i.e., $\tau = m$), then the communication and iteration complexities are 
 $$\squeeze \# comms = \mathcal{O}\left(\sqrt{\frac{L_{\max}}{\mu}\left(1+\frac{\omega}{M}\right)}\log \frac{1}{\varepsilon}\right), \qquad \#iters = \mathcal{O}\left(\frac{L_{\max}}{\mu}\left(1+\frac{\omega}{M}\right)\log \frac{1}{\varepsilon}\right),$$ respectively.
\end{corollary}

Assume $r(x)\equiv 0$. If $\cQ_t^{ij}(x)\equiv x$, then $\omega = 0$, and we restore  the well-known \algname{LSVRG} method~\citep{hofmann2015variance,kovalev2020don}, assuming that the functions $\phi_{ij}$ have the same smoothness constant. We can recover the same rate exactly as well, but with a slightly more refined analysis, one in which we do not need to work with compressors (\algname{LSVRG} does not involve any), which makes for a tighter analysis. On the other hand, if $\tau = m$, we restore the rate of the well-known \algname{DIANA}~\citep{DIANA,DIANA2} method and its sibling \algname{Rand-DIANA}~\citep{Shifted}.

\section{Experiments}\label{sec:experiments}

To illustrate the predictive power of our theory, it suffices to  consider  $L_2$-regularized logistic regression  in the distributed setting \eqref{eq:P}, with 
$$\squeeze \phi_i(x)=\frac{1}{n_i} \sum \limits_{j=1}^{n_i} \log \left(1+\exp \left(-b_{ij} a_{ij}^{\top} x\right)\right)+\frac{\lambda}{2}\|x\|^{2},$$ where $n_i$ is the number of data points per worker $a_{ij}\in \R^{d'}$ and $b_{ij} \in \{-1, +1\}$ are the data samples and labels. We choose $n_i=m = \nicefrac{n}{M}$ for all $i$. We set the regularization parameter $\lambda = 5\cdot10^{-4}L$ by default, where $L$ is the smoothness constant of $f$. We conduct several experiments on the \dataset{w8a} dataset from LibSVM library~\citep{chang2011libsvm}.

\subsection{ProxSkip-LSVRG vs baselines}
In Figure~\ref{fig:exp_comp_total_cost} (first row), we compare various LT baselines\footnote{With the exception of \algname{LocalSGD}, all  use client drift correction.}  for three choices of mini-batch sizes ($\tau=16, 32, 64$) with our method \algname{ProxSkip-VR} combined with the \algname{LSVRG} estimator, which is a special case of the \algname{HUB} estimator when $\cQ_t^{ij}(x)\equiv x$ for all $i,j$. We see that our method outperforms all other methods significantly due to its communication-acceleration properties.

\begin{figure}[!htbp]
	\centering
	\begin{subfigure}[b]{0.3\textwidth}
		\centering
		\includegraphics[trim=20 10 40 40, clip, width=\textwidth]{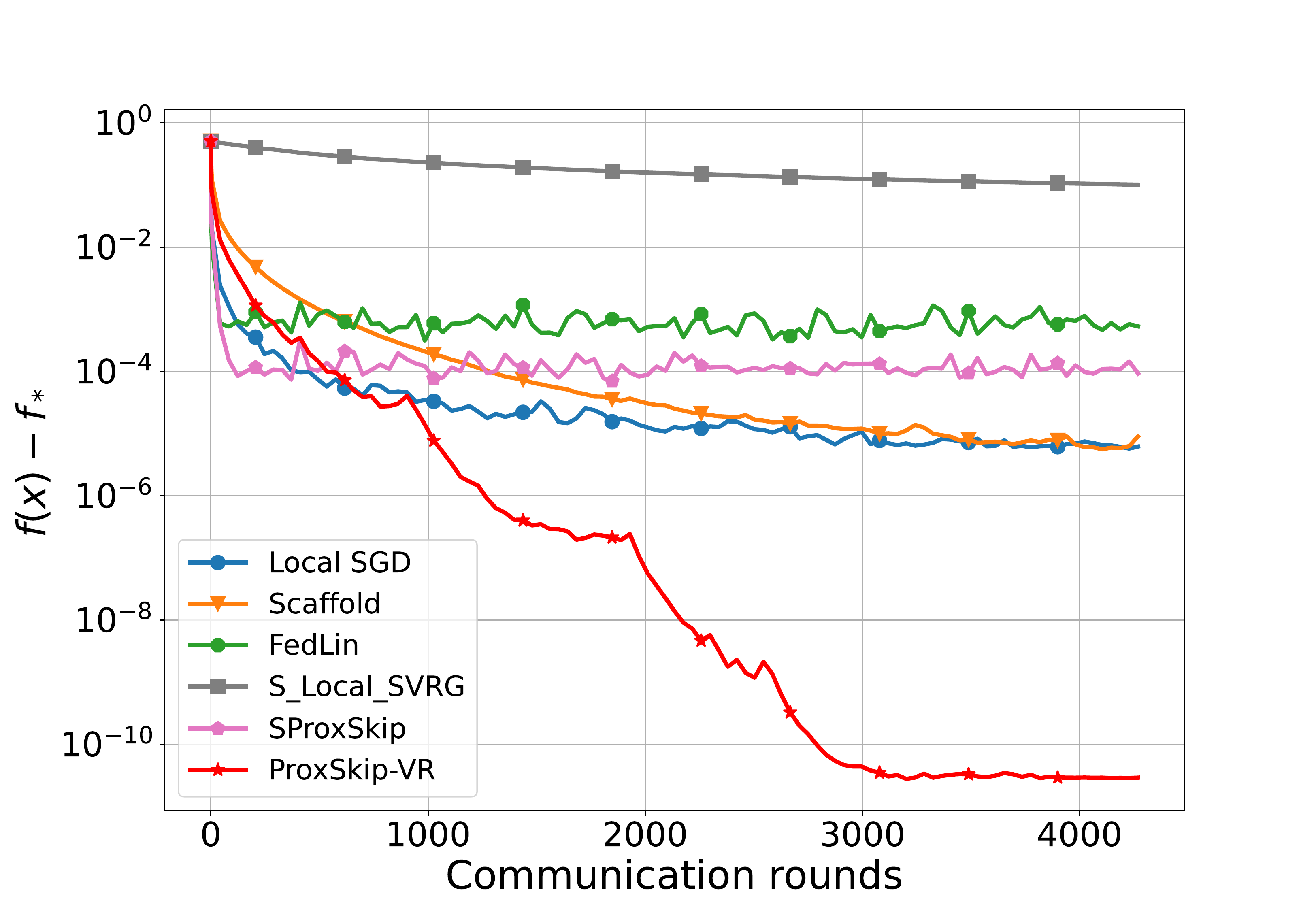}
		\caption{$\tau=16$}
	\end{subfigure}
	\hfill 
	\begin{subfigure}[b]{0.3\textwidth}
		\centering
		\includegraphics[trim=20 10 40 40, clip, width=\textwidth]{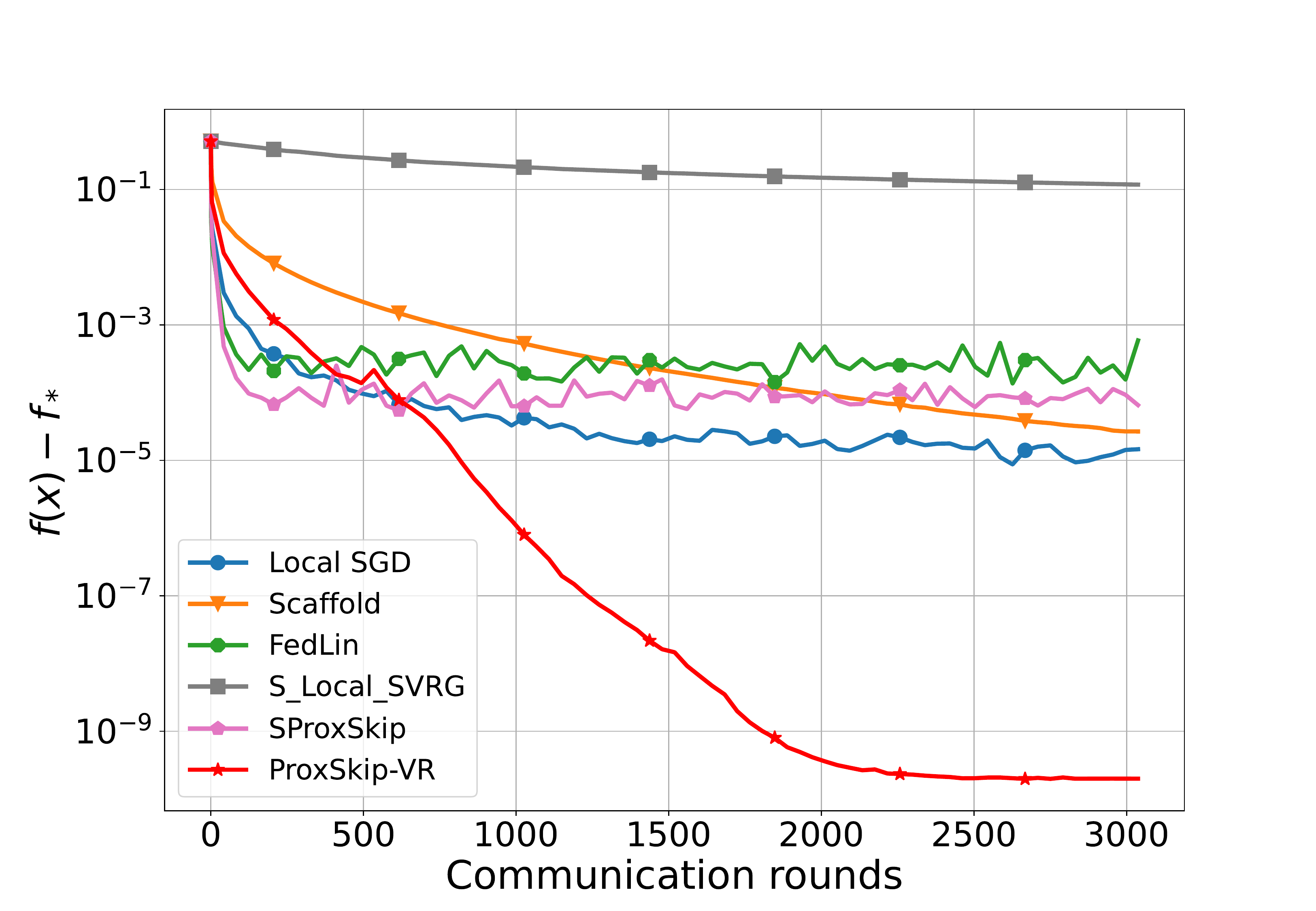}
		\caption{$\tau=32$}
		\end{subfigure}
		\hfill
	\begin{subfigure}[b]{0.3\textwidth}
		\centering
		\includegraphics[trim=20 10 40 40, clip, width=\textwidth]{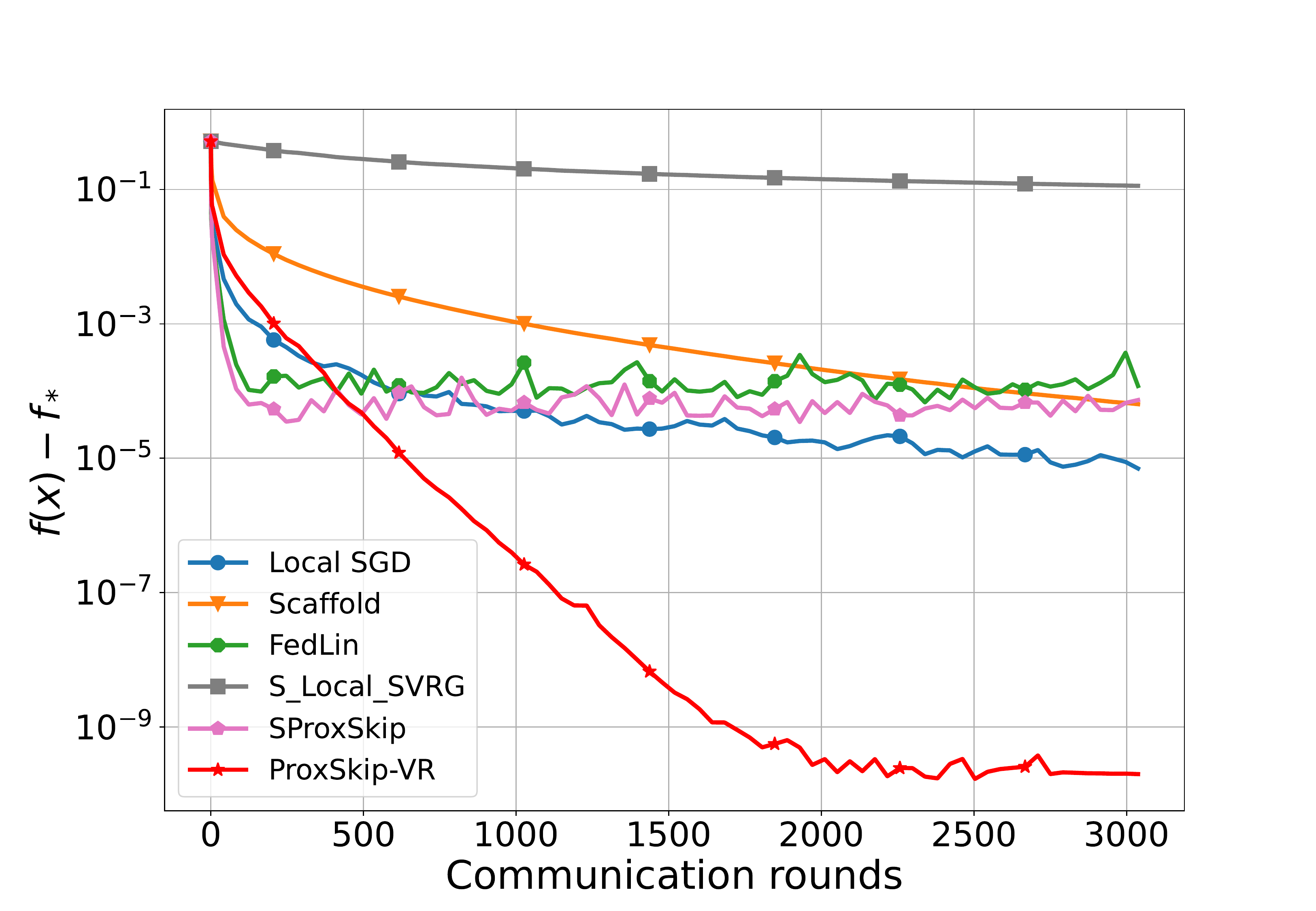}
		\caption{$\tau=64$}
	\end{subfigure}
	\begin{subfigure}[b]{0.3\textwidth}   
		\centering
		\includegraphics[trim=0 10 40 20, clip, width=\textwidth]{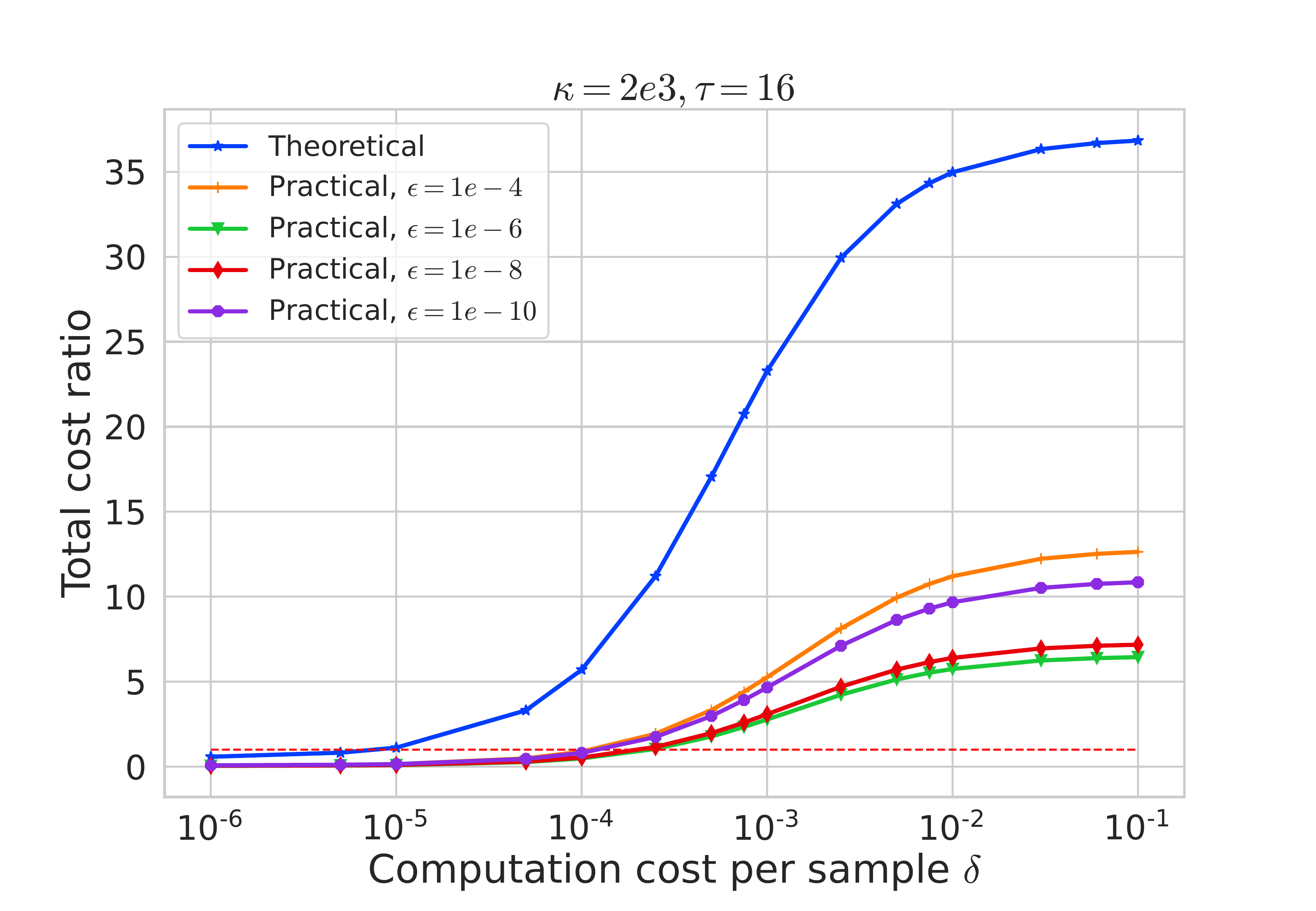}
	\end{subfigure}
	\hfill  
	\begin{subfigure}[b]{0.3\textwidth}
		\centering
		\includegraphics[trim=0 10 40 20, clip, width=\textwidth]{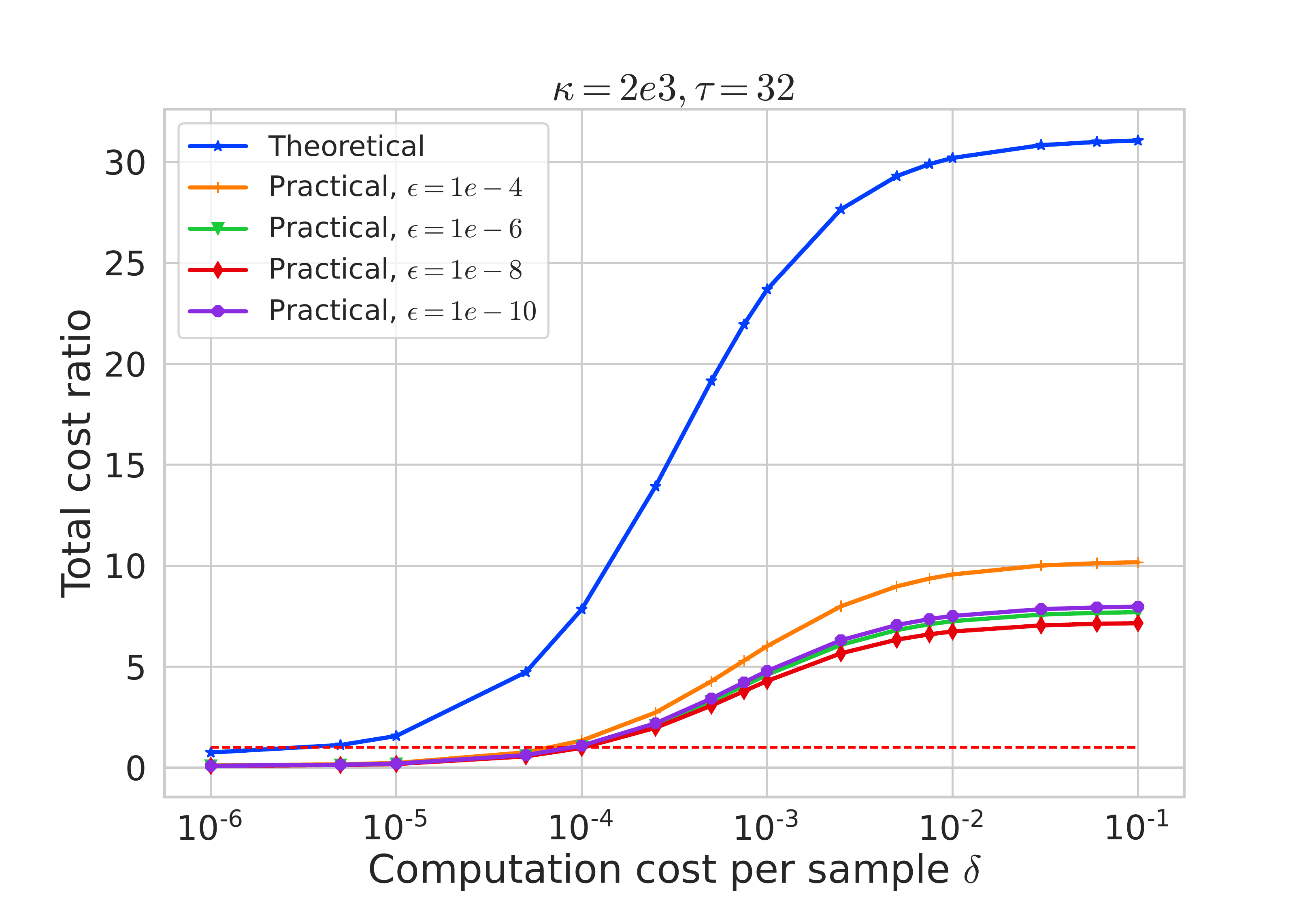}
	\end{subfigure}   
	\hfill  
	\begin{subfigure}[b]{0.3\textwidth}
		\centering
		\includegraphics[trim=0 10 40 20, clip, width=\textwidth]{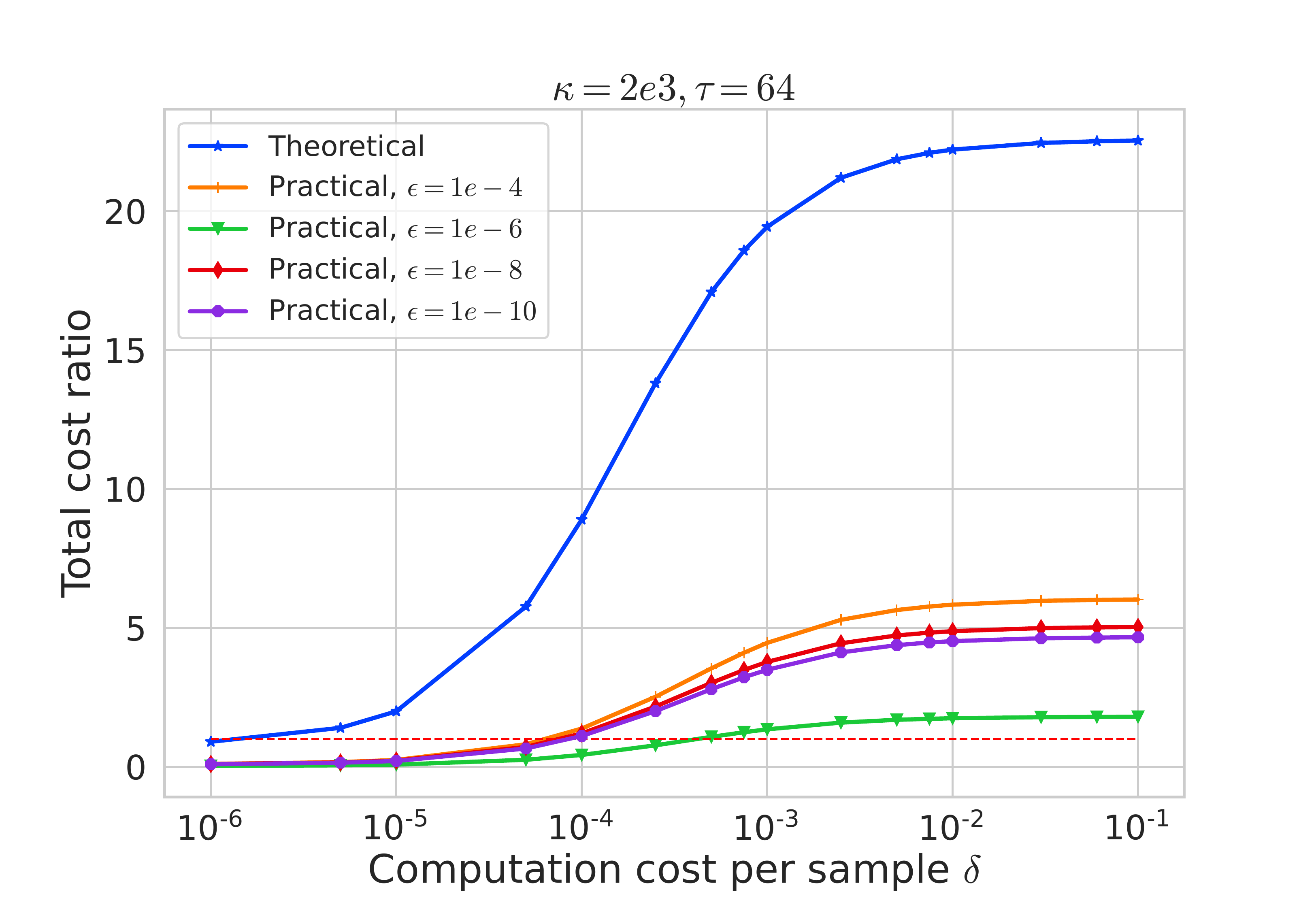}
	\end{subfigure} 
	\caption{The top row shows the convergence results compared with baselines and the second row is the total cost ratio of \algname{ProxSkip} over our \algname{ProxSkip-VR}.}
	\label{fig:exp_comp_total_cost}
\end{figure}

\subsection{The total cost of ProxSkip-LSVRG can be significantly smaller than the total cost of ProxSkip, in theory and practice}

Next, we derive the total cost, which includes communication cost (assumed to be $1$, for normalization purposes), and computational cost (assumed to be  $\delta$; and equal to the cost of performing one \algname{SGD} step with a single data point). Let us consider the total cost of \algname{ProxSkip-VR} in the case of the \algname{LSVRG} estimator: in each iteration we compute 1 stochastic gradient, and with probability $q$ we compute the exact gradient. We do not need to compute a second stochastic gradient since we can use memory and the relation $y_{t+1} = x_t$. The total cost for \algname{ProxSkip-VR} is equal to 
$$\text{Cost}(\text{\algname{ProxSkip-VR}}) \eqdef T_{\text{comm.}}(\text{\algname{ProxSkip-VR}}) + \delta   \left(q m + (1-q)\tau+\tau\right)T_{\text{iter}}({\text{\algname{ProxSkip-VR}}}).$$ 
\algname{ProxSkip} requires full/exact gradient computation at each iteration, so the total cost of \algname{ProxSkip} is 
$$\text{Cost}(\text{\algname{ProxSkip}}) \eqdef T_{\text{comm.}}(\text{\algname{ProxSkip}}) + \delta m  T_{\text{iter}}({\text{\algname{ProxSkip}}}).$$ 
Using Theorems~\ref{thm:proxskip} and \ref{thm:QLSVRG} and the value  $\squeeze  L(\tau) \eqdef \frac{m-\tau}{\tau (m-1)}L_{\max} + \frac{m(\tau - 1)}{\tau (m-1)} L$
of the expected smoothness constant for sampling with minibatch size $\tau$ uniformly at random,
 we get the following expression for the cost ratio, expressed as a function of $\delta$: 
\begin{align}\label{eq:09uf09ufdff} 
		 \squeeze   \text{Cost ratio}(\delta) \eqdef \frac{\text{Cost}(\text{\algname{ProxSkip}})}{\text{Cost}(\text{\algname{ProxSkip-VR}})} = \frac{\sqrt{\mu L} + mL \delta}{\sqrt{\mu L(\tau)} + \left( 2m\mu + (2L(\tau)-2\mu)\tau\right)\delta}.
\end{align}
We can easily calculate the limits of this expression: 
\begin{align*}
	\squeeze
	   \text{Cost ratio} (\delta=0) = \sqrt{\frac{L}{L(\tau)}},      \quad\quad\quad    \text{Cost ratio} (\delta\to \infty)= \frac{m L}{2 (m \mu+\left(L(\tau)- \mu\right) \tau)}.
\end{align*}
Since $L(\tau) \geq L$, the cost ratio is below 1 when $\delta=0$ and $\tau<m$, which means that \algname{ProxSkip} is better than \algname{ProxSkip-VR}. This is to be expected since $\delta=0$ means that we {\em ignore} the cost of local computation entirely, which offers advantage to the former method. On the other hand, the cost ratio is an increasing function of $\delta$, and often reaches the threshold of 1 with a small value of $\delta$; in \ref{fig:exp_comp_total_cost} (second row), this threshold is reached for $\delta \in [10^{-6},10^{-5}]$ in all three plots. 

In Figure~\ref{fig:exp_comp_total_cost} (second row), we depict the theoretical cost ratio according to \eqref{eq:09uf09ufdff} and the corresponding experimental cost ratio obtained by an actual run of both methods to achieve $\varepsilon$-accuracy, with $\varepsilon = 10^{-6} $ and  $\varepsilon = 10^{-8}$. Remarkably, the experimental results follows the same pattern as our theoretical prediction. The empirical curves appear lower because we use approximations for $L_{\max}$, $L$ and $\mu$. As we can see, starting from $\delta = 10^{-4}$, \algname{ProxSkip-VR} starts to outperform \algname{ProxSkip}. As $\delta$ grows, the advantage of variance reduction embedded in \algname{ProxSkip-VR} over vanilla \algname{ProxSkip} grows. These results suggest that variance reduction is of practical utility in terms of the total cost, even for small values of $\delta$, and its effectiveness grows with $\delta$.   


\subsection{Additional experiments with ProxSkip-VR}

We conduct further experiments to validate the efficiency of our proposed method \algname{ProxSkip-VR}. We instantiate our variance reduction design with \algname{LSVRG}and compare our method with several baselines across various  datasets (\dataset{w8a}/\dataset{a9a}), different number of workers (10/20), and different batch sizes (16/32/64). All results in Figures~\ref{fig:053},~\ref{fig:054},~\ref{fig:051},~\ref{fig:052} show that \algname{ProxSkip-VR}  achieves linear convergence and outperforms the baselines. 

\begin{figure}[!htbp]
	\centering
	\begin{subfigure}[b]{0.32\textwidth}
		\centering
		\includegraphics[trim=20 10 40 40, clip, width=\textwidth]{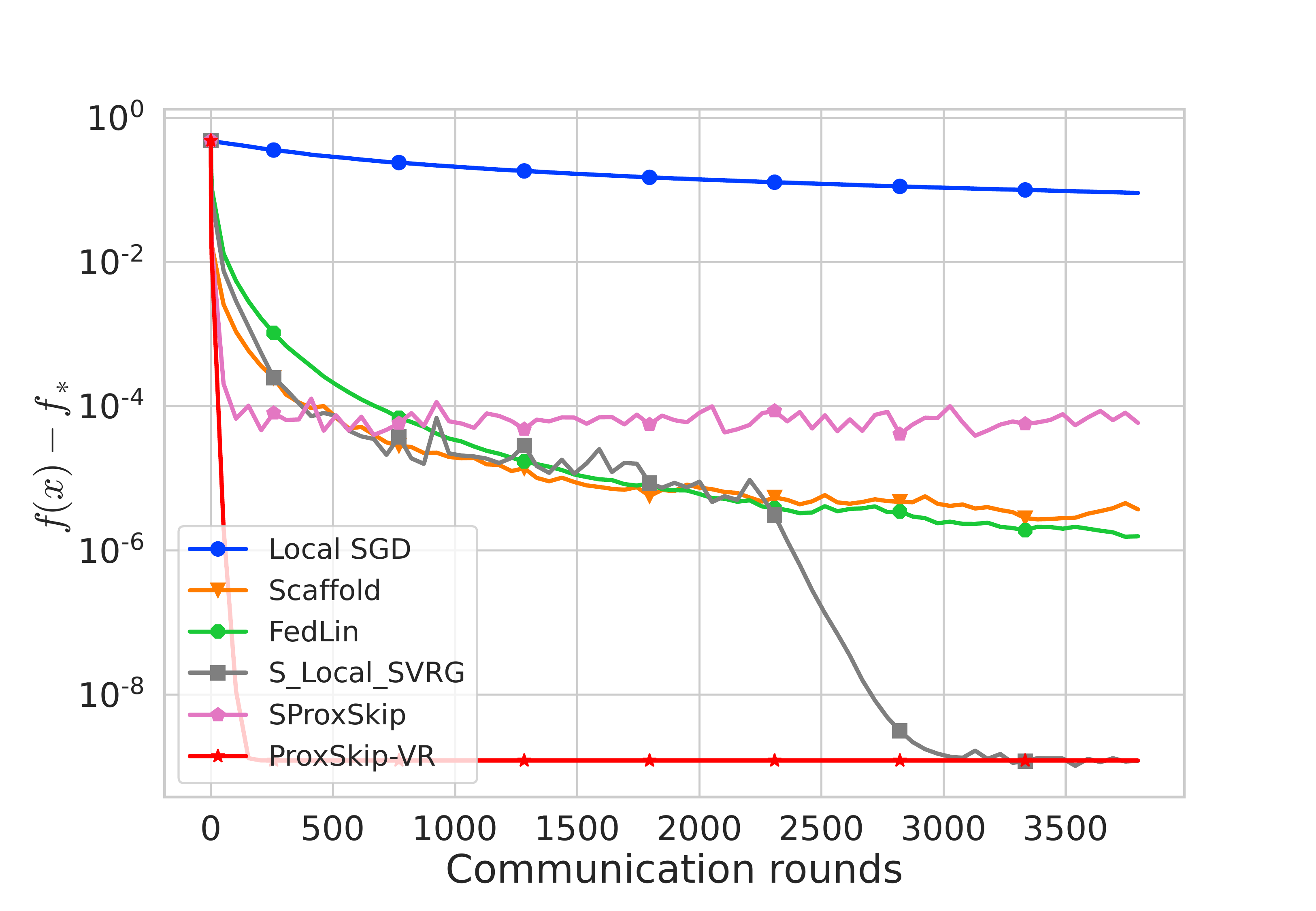}
		\caption{$\tau=16$.}
	\end{subfigure}
	\hfill
	\begin{subfigure}[b]{0.32\textwidth}
		\centering
		\includegraphics[trim=20 10 40 40, clip, width=\textwidth]{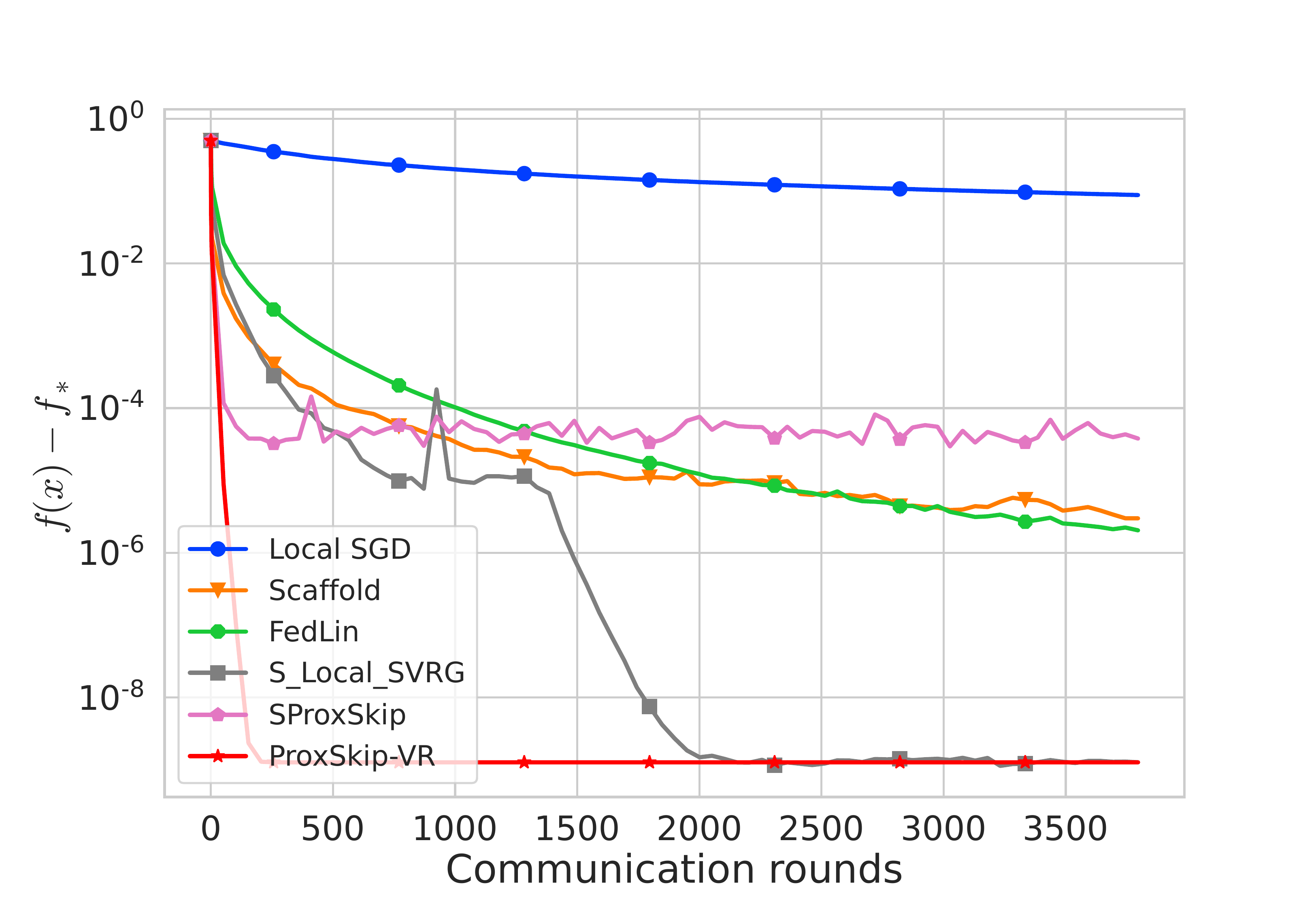}
		\caption{$\tau=32$.}
	\end{subfigure}
	\hfill
	\begin{subfigure}[b]{0.32\textwidth}
		\centering
		\includegraphics[trim=20 10 40 40, clip, width=\textwidth]{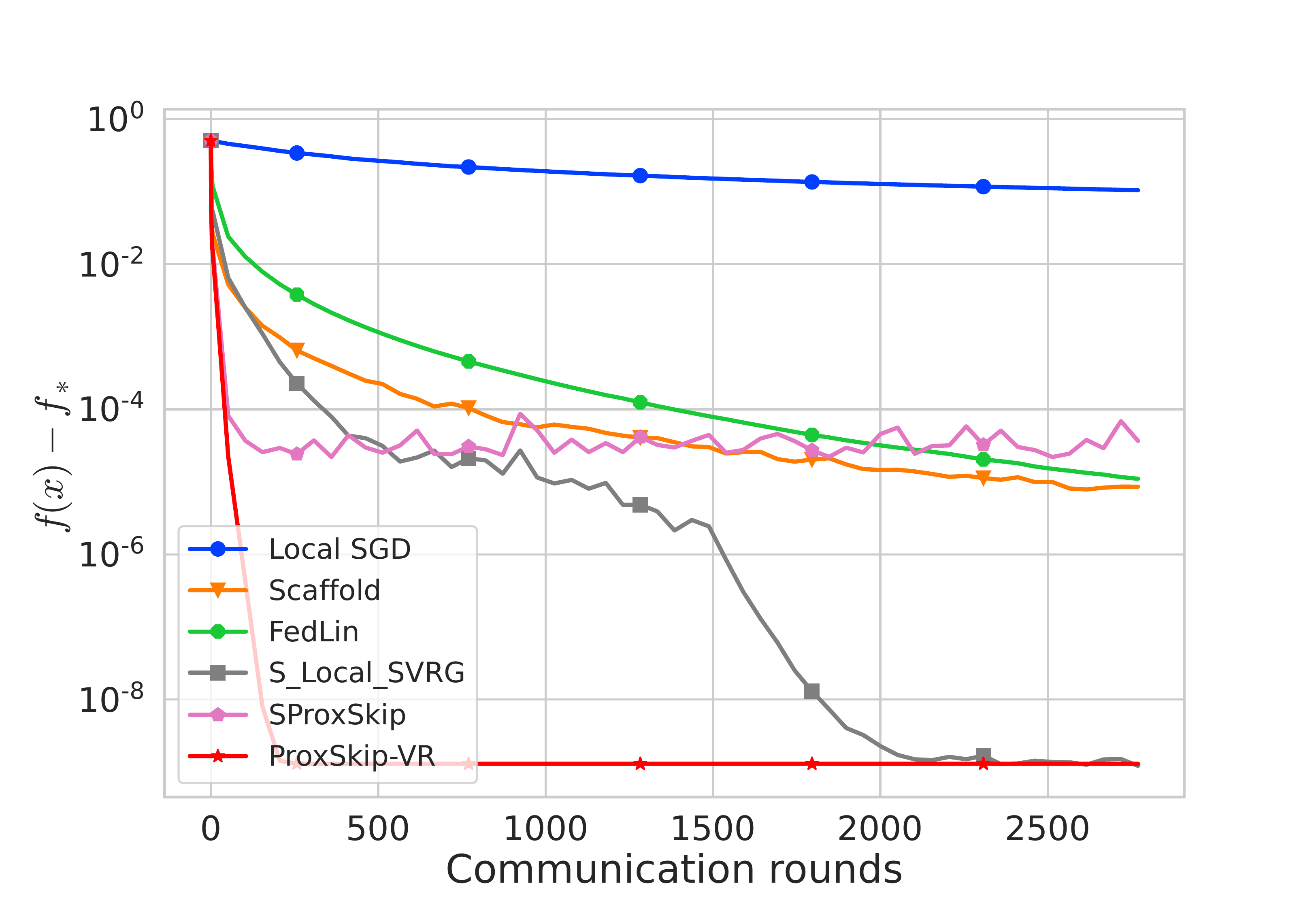}
		\caption{$\tau=64$.}
	\end{subfigure}
	\caption{Convergence results with 20 distributed workers on \dataset{w8a} dataset, $\kappa=1e3$.}
	\label{fig:053}
\end{figure} 

\begin{figure}[!htbp]
	\centering
	\begin{subfigure}[b]{0.32\textwidth}
		\centering
		\includegraphics[trim=20 10 40 40, clip, width=\textwidth]{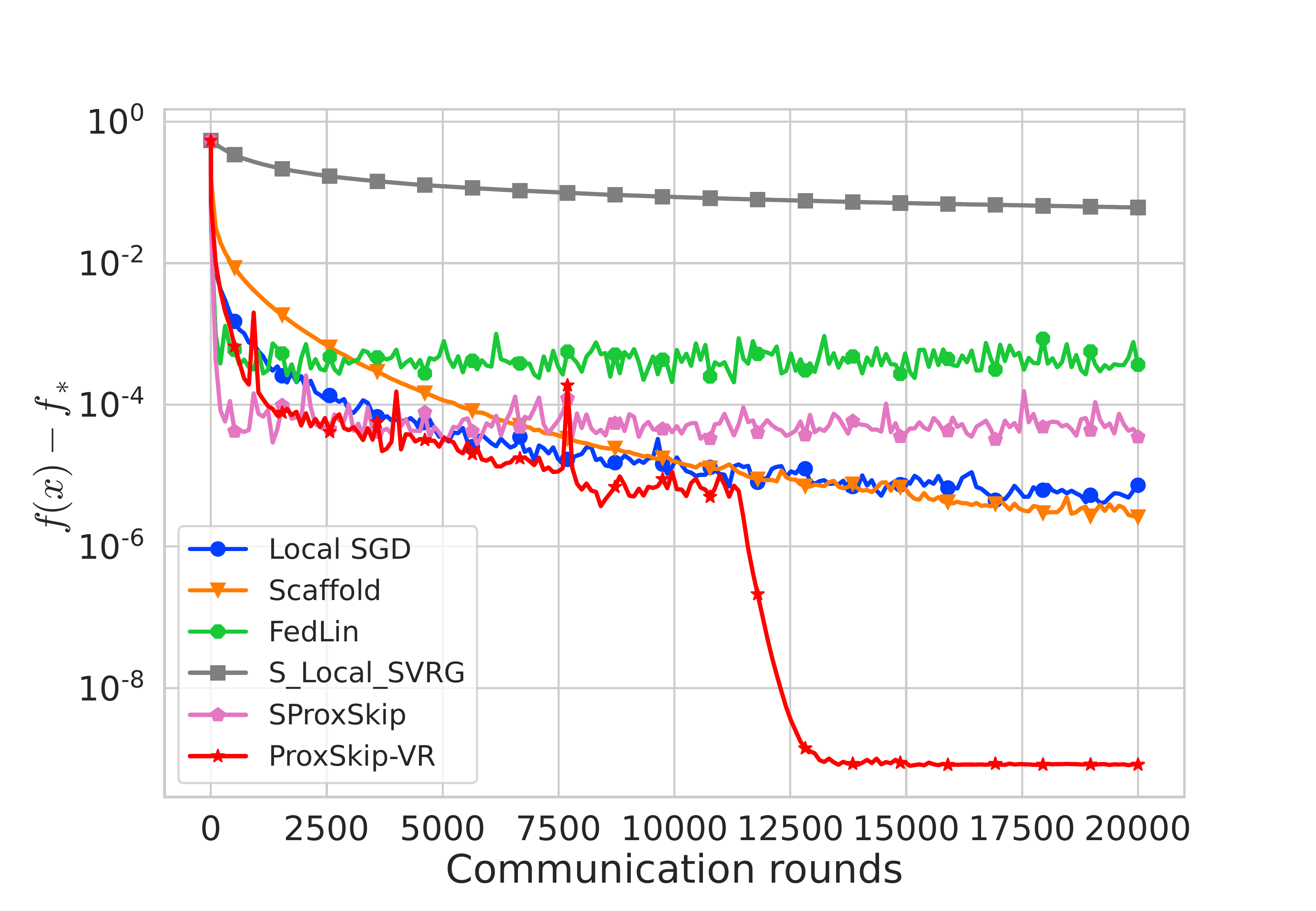}
		\caption{$\tau=16$.}
	\end{subfigure}
	\hfill
	\begin{subfigure}[b]{0.32\textwidth}
		\centering
		\includegraphics[trim=20 10 40 40, clip, width=\textwidth]{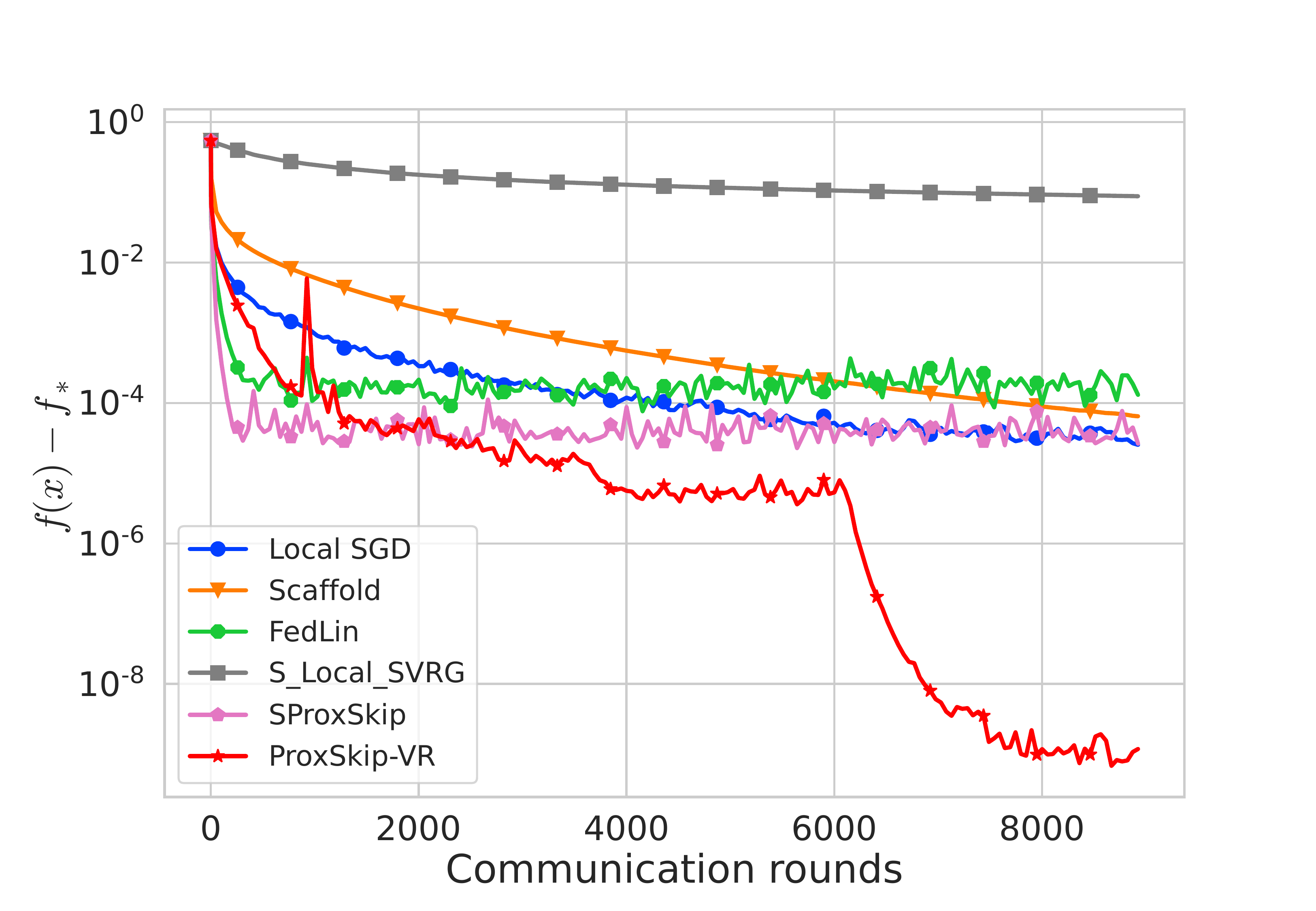}
		\caption{$\tau=32$.}
	\end{subfigure}
	\hfill
	\begin{subfigure}[b]{0.32\textwidth}
		\centering
		\includegraphics[trim=20 10 40 40, clip, width=\textwidth]{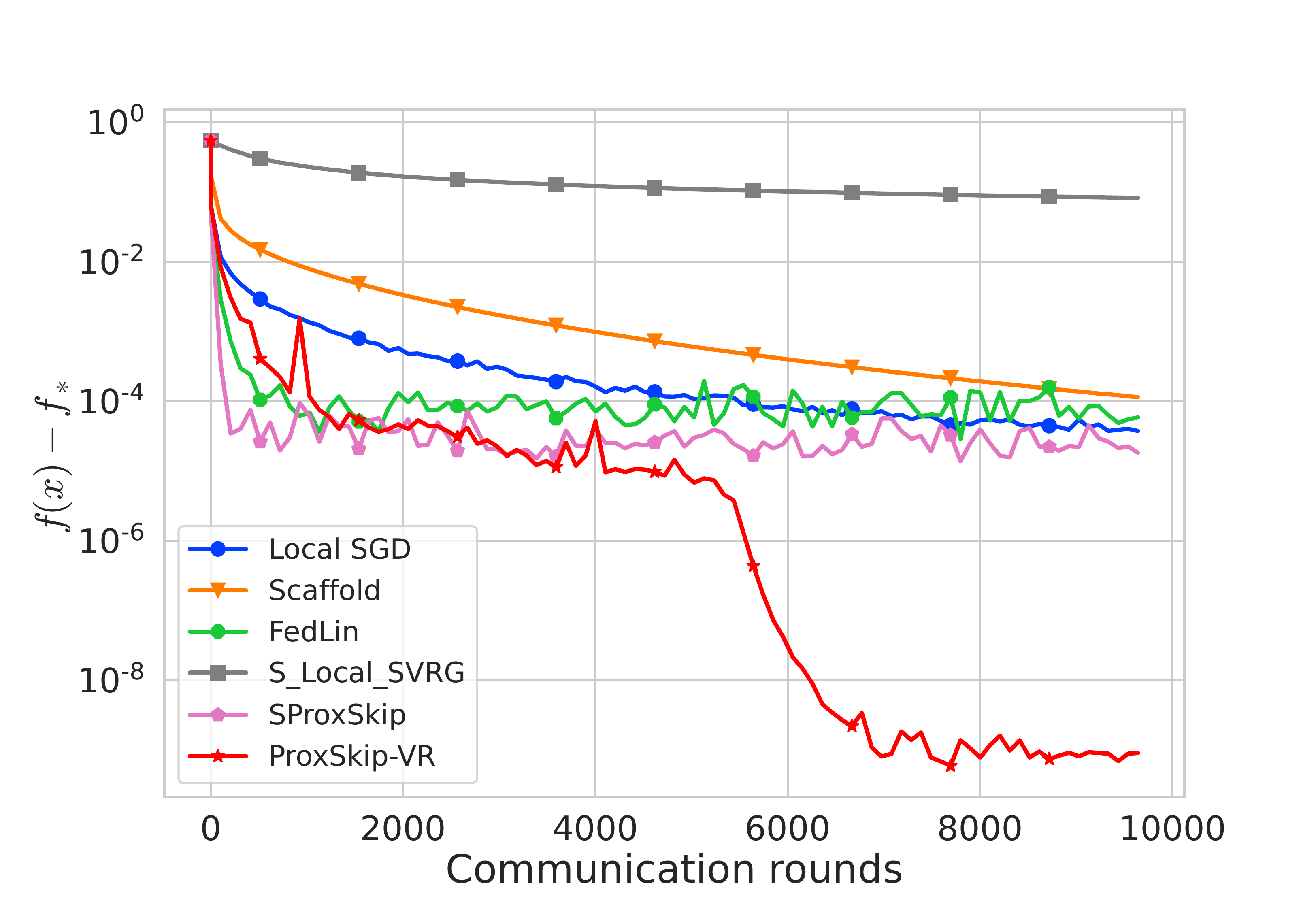}
		\caption{$\tau=64$.}
	\end{subfigure}
	\caption{Convergence results with 20 distributed workers on \dataset{w8a} dataset, $\kappa=1e4$.}
	\label{fig:054}
\end{figure}

\begin{figure}[!htbp]
	\centering
	\begin{subfigure}[b]{0.32\textwidth}
		\centering
		\includegraphics[trim=20 10 40 40, clip, width=\textwidth]{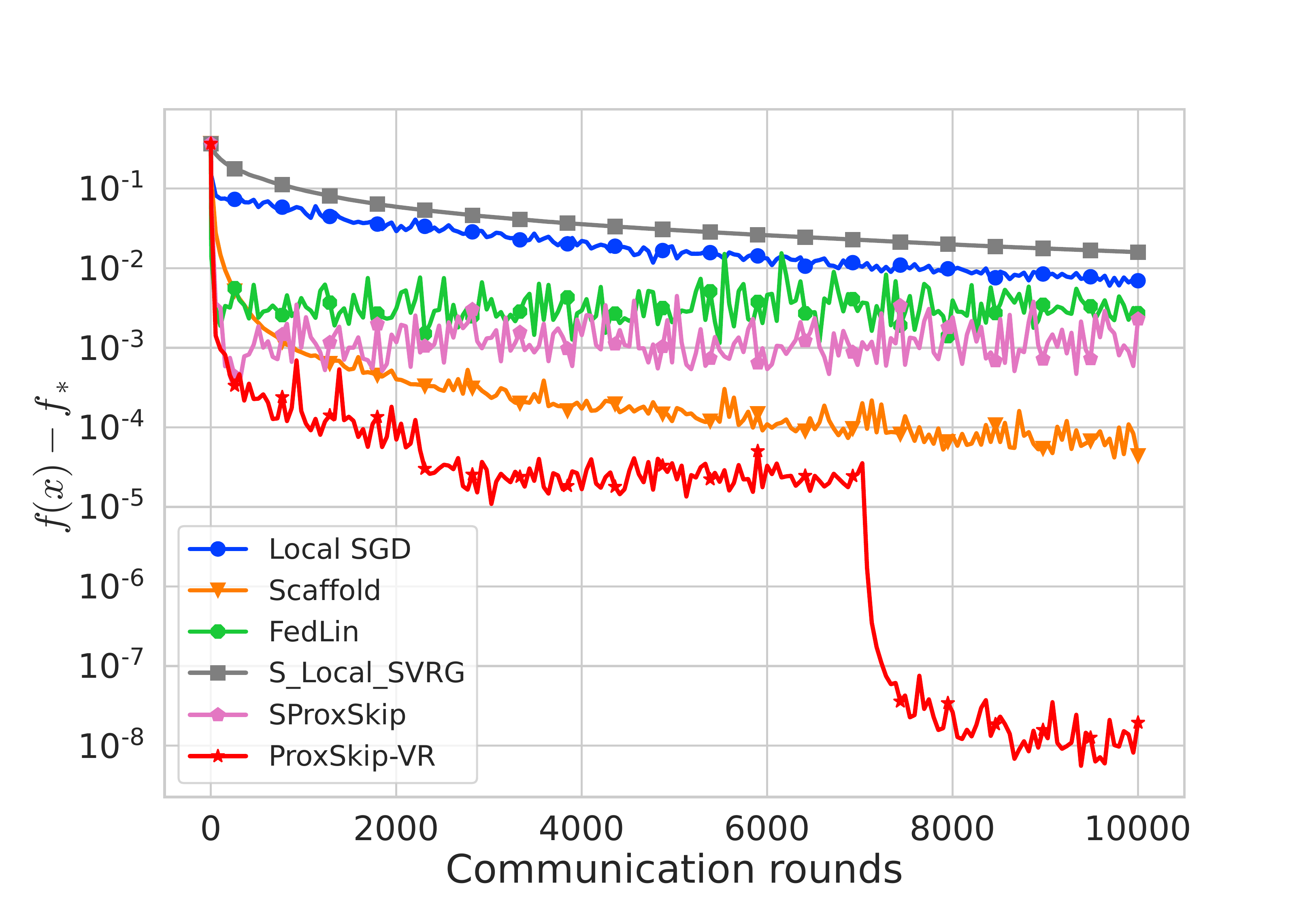}
		\caption{$\tau=16$.}
	\end{subfigure}
	\hfill
	\begin{subfigure}[b]{0.32\textwidth}
		\centering
		\includegraphics[trim=20 10 40 40, clip, width=\textwidth]{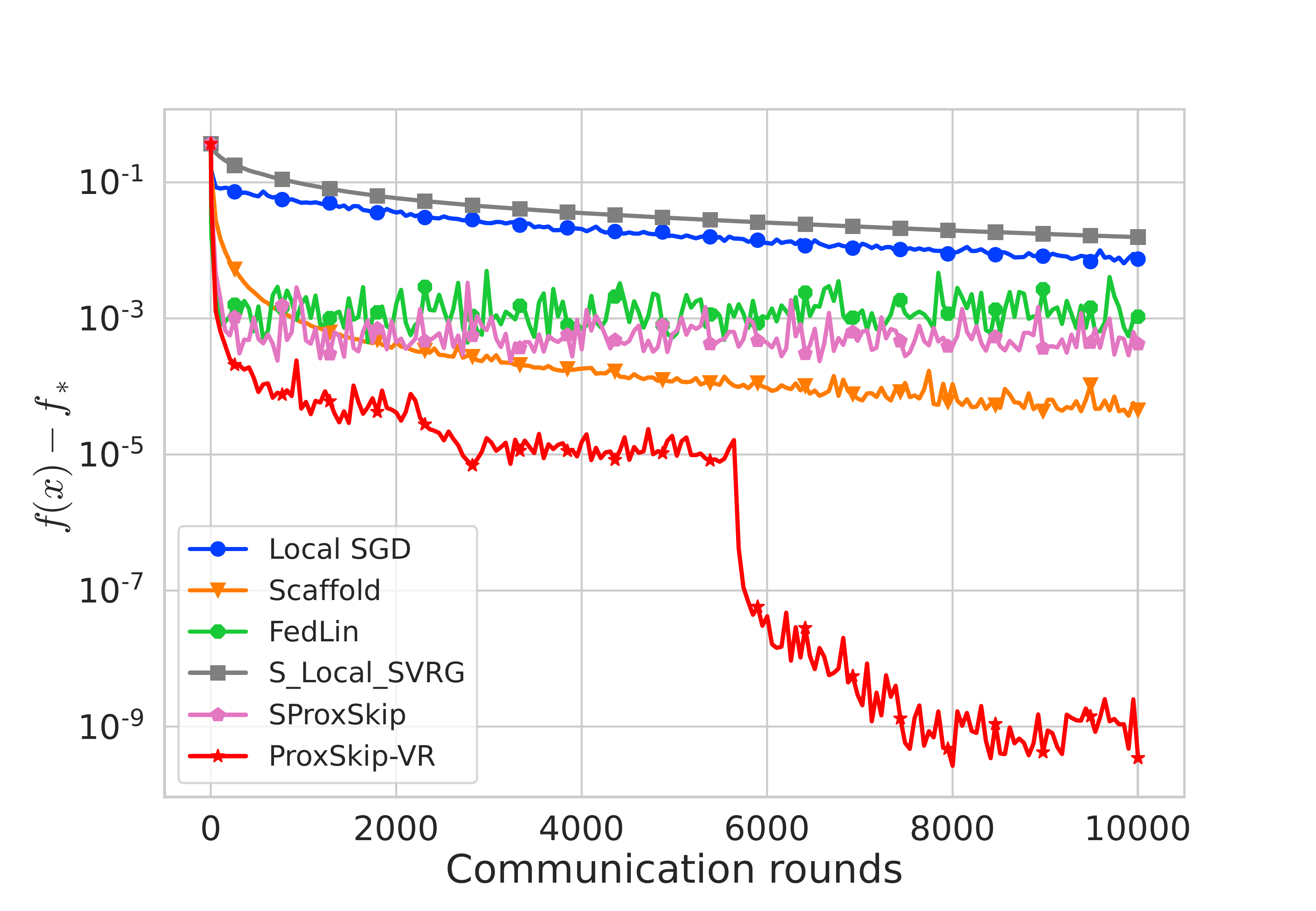}
		\caption{$\tau=32$.}
	\end{subfigure}
	\hfill
	\begin{subfigure}[b]{0.32\textwidth}
		\centering
		\includegraphics[trim=20 10 40 40, clip, width=\textwidth]{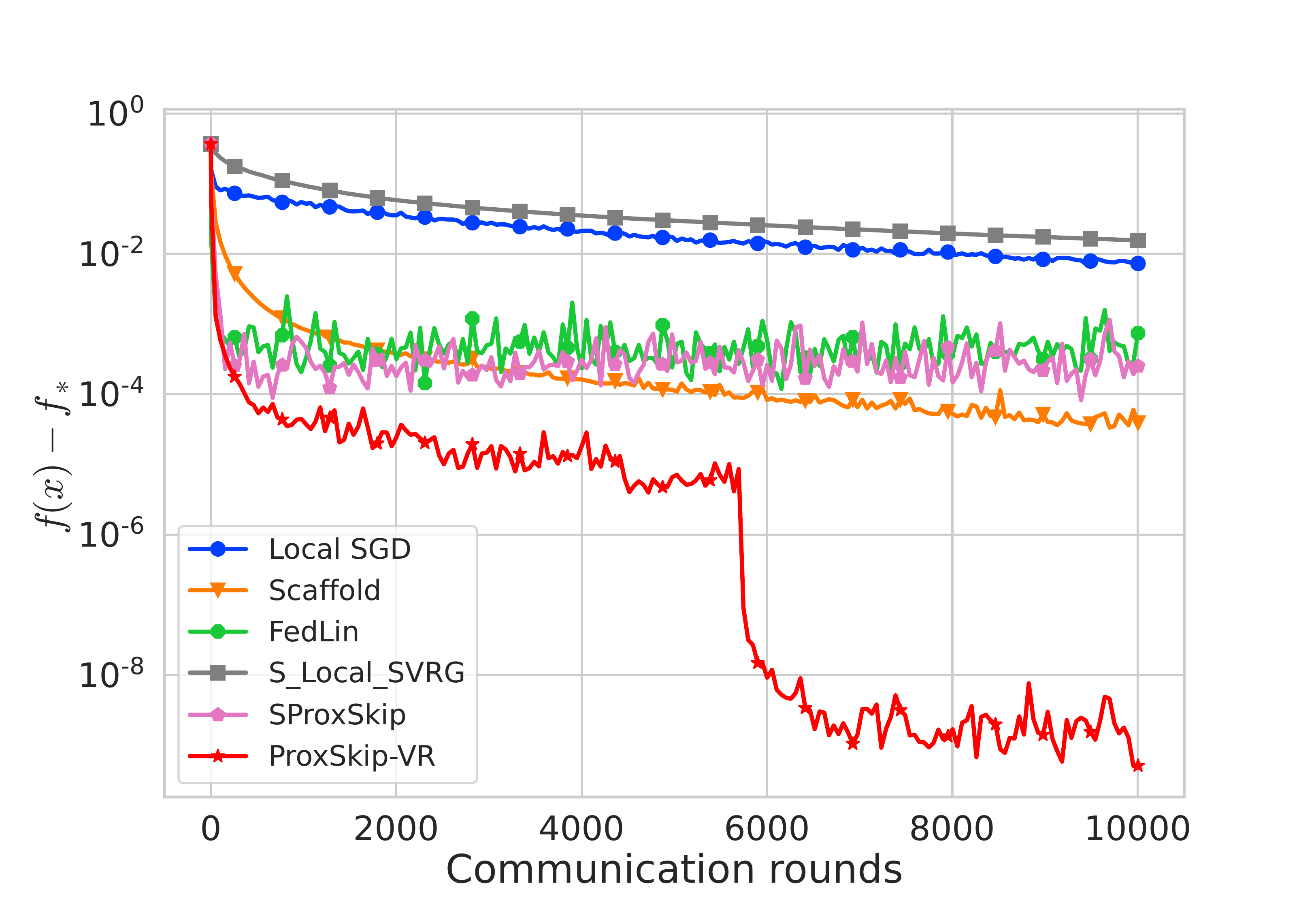}
		\caption{$\tau=64$.}
	\end{subfigure}
	\caption{Convergence results with 10 distributed workers on \dataset{a9a} dataset, $\kappa=1e4$.}
	\label{fig:051}
\end{figure} 

\begin{figure}[!htbp]
	\centering
	\begin{subfigure}[b]{0.32\textwidth}
		\centering
		\includegraphics[trim=20 10 40 40, clip, width=\textwidth]{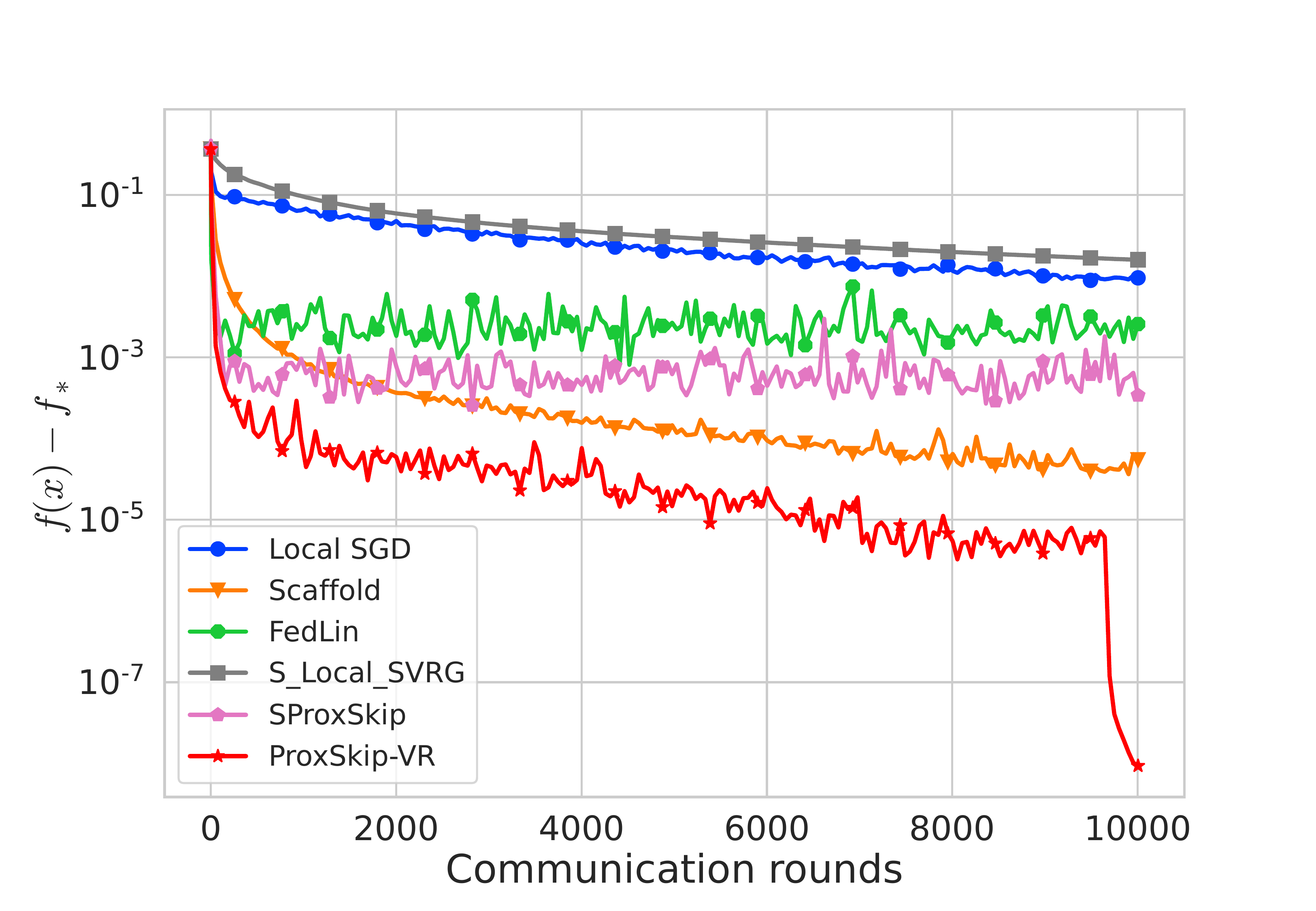}
		\caption{$\tau=16$.}
	\end{subfigure}
	\hfill
	\begin{subfigure}[b]{0.32\textwidth}
		\centering
		\includegraphics[trim=20 10 40 40, clip, width=\textwidth]{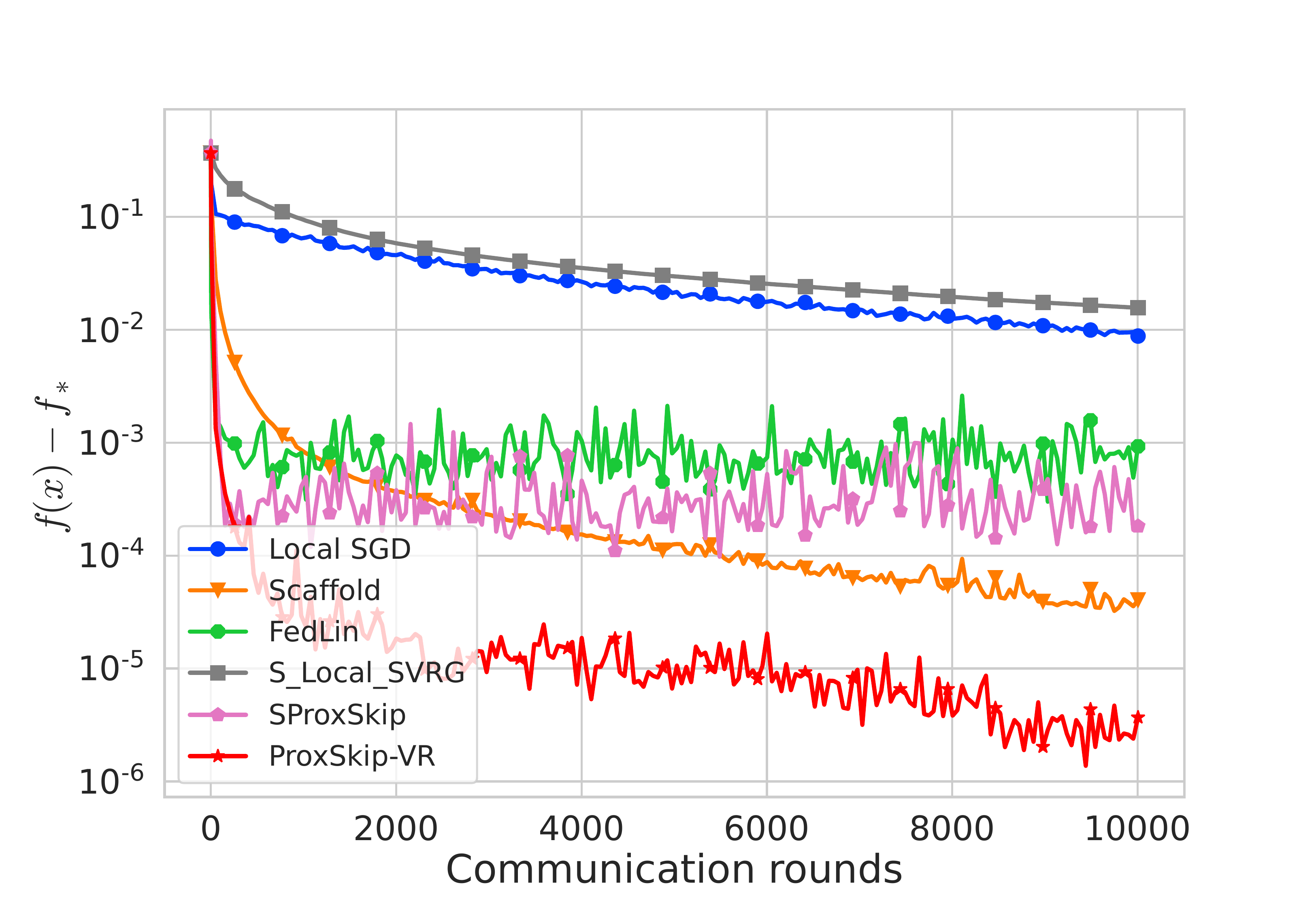}
		\caption{$\tau=32$.}
	\end{subfigure}
	\hfill
	\begin{subfigure}[b]{0.32\textwidth}
		\centering
		\includegraphics[trim=20 10 40 40, clip, width=\textwidth]{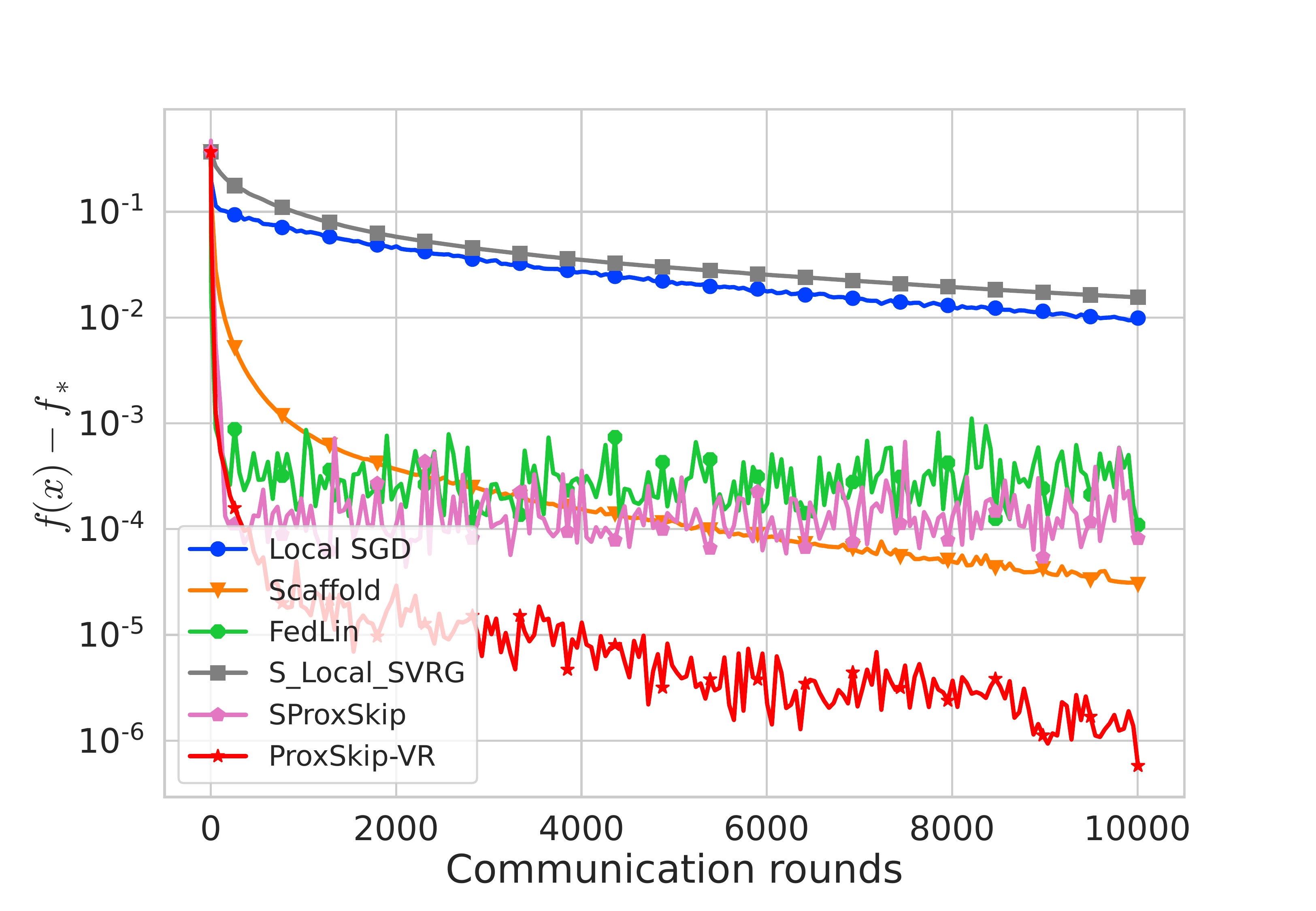}
		\caption{$\tau=64$.}
	\end{subfigure}
	\caption{Convergence results with 20 distributed workers on \dataset{a9a} dataset, $\kappa=1e4$.}
	\label{fig:052}
\end{figure} 

\subsection{Additional experiments of cost ratio of ProxSkip over ProxSkip-VR}

Here we report on more experiments related to the total cost ratio of \algname{ProxSkip} over \algname{ProxSkip-VR} in Figures~\ref{fig:055},~\ref{fig:056},~\ref{fig:057},~\ref{fig:058}. When \algname{GD} is used as the subroutine, i.e., when \algname{ProxSkip-VR} reduces to \algname{ProxSkip}, the ratio is equal to one, and  is depicted by the red horizontal dashed line. Any cost ratio above one means that \algname{ProxSkip-VR} benefits over \algname{ProxSkip} (e.g., a cost ratio of $10$ means $10\times$ speedup in favor of our method). As seen in the plots,  acceleration of our method over \algname{ProxSkip} is clearly visible, and improves as the local computation cost $\delta$ per sample increases. For large values of $\delta$ (i.e., $\delta\approx 10^{-1}$), the acceleration can reach $20\times$ to $85\times$. Note also that acceleration is more significant for small minibatch sizes. That is, it is better for $\tau=16$ than for $\tau=32$, which is better than in the $\tau=64$ case.  This means that in terms of total cost, it is beneficial for the workers to use smaller minibatch sizes, i.e., it is beneficial to be as far from the full batch regime employed by \algname{ProxSkip} as possible. 

\begin{figure}[!htbp]
	\centering
	\begin{subfigure}[b]{0.32\textwidth}
		\centering
		\includegraphics[trim=20 10 40 40, clip, width=\textwidth]{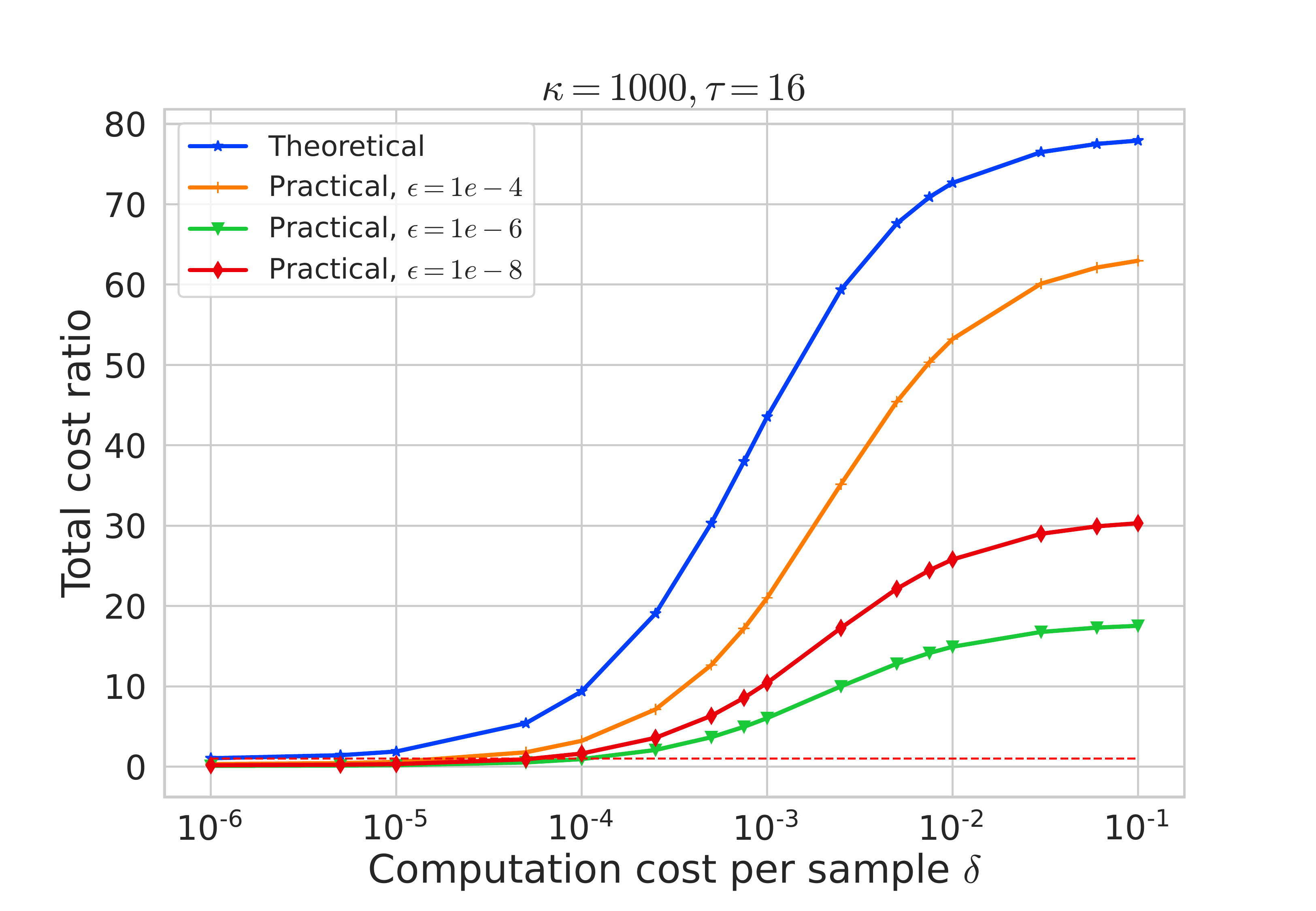}
		\caption{$\tau=16$.}
	\end{subfigure}
	\hfill
	\begin{subfigure}[b]{0.32\textwidth}
		\centering
		\includegraphics[trim=20 10 40 40, clip, width=\textwidth]{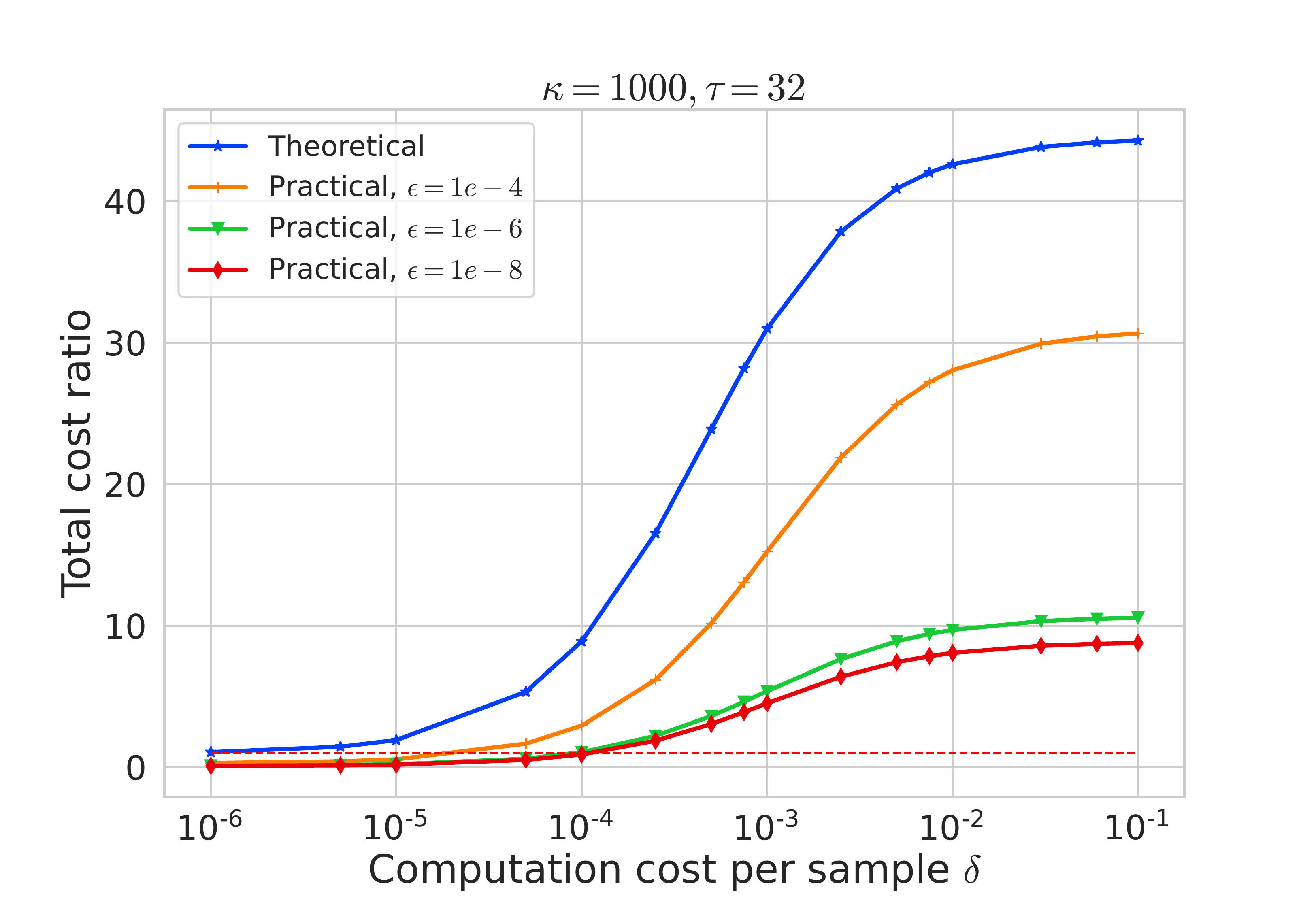}
		\caption{$\tau=32$.}
	\end{subfigure}
	\hfill
	\begin{subfigure}[b]{0.32\textwidth}
		\centering
		\includegraphics[trim=20 10 40 40, clip, width=\textwidth]{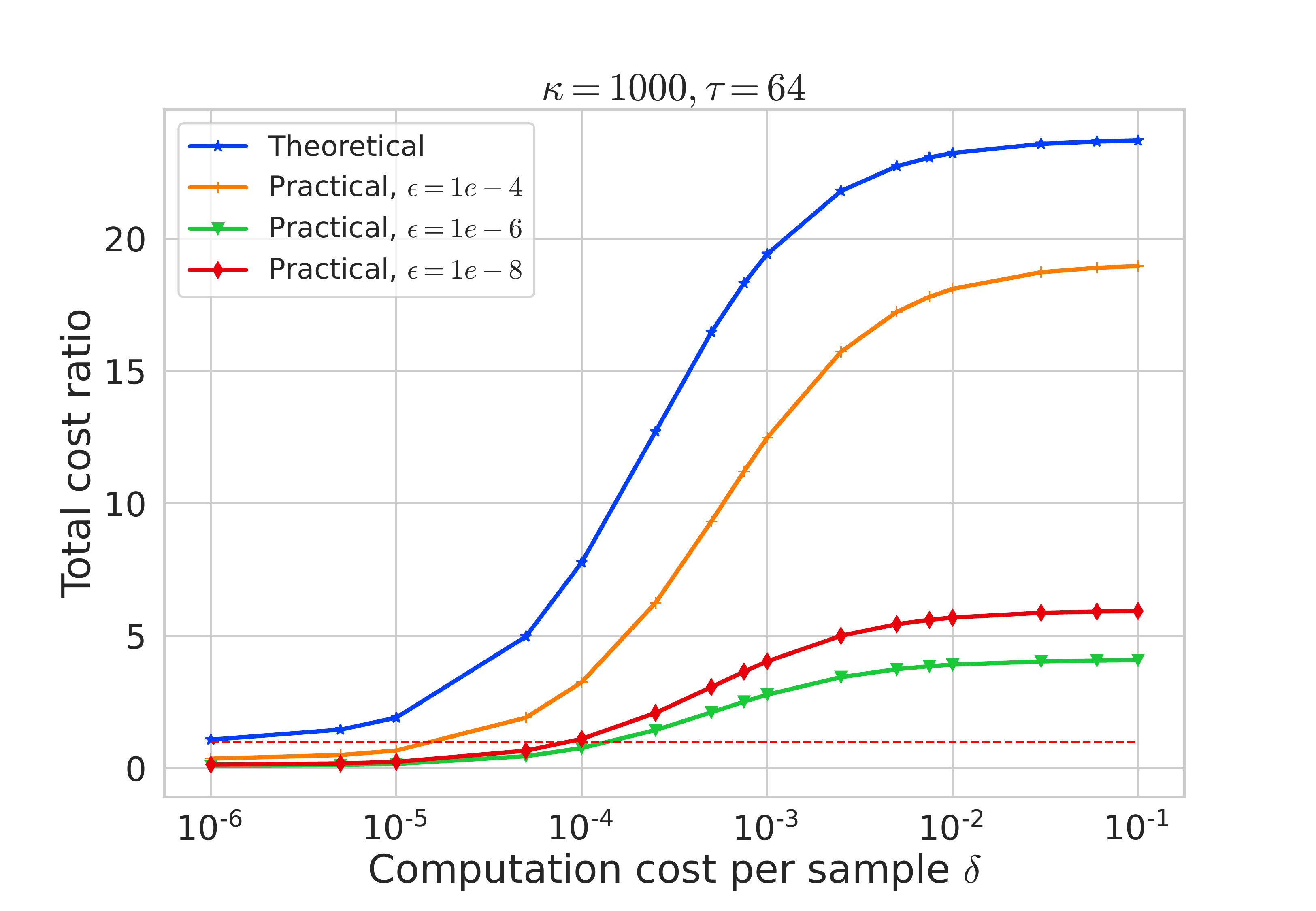}
		\caption{$\tau=64$.}
	\end{subfigure}
	\caption{Acceleration with 10 distributed workers on \dataset{a9a} dataset, $\kappa=1e3$.}
	\label{fig:055}
\end{figure} 

\begin{figure}[!htbp]
	\centering
	\begin{subfigure}[b]{0.32\textwidth}
		\centering
		\includegraphics[trim=20 10 40 40, clip, width=\textwidth]{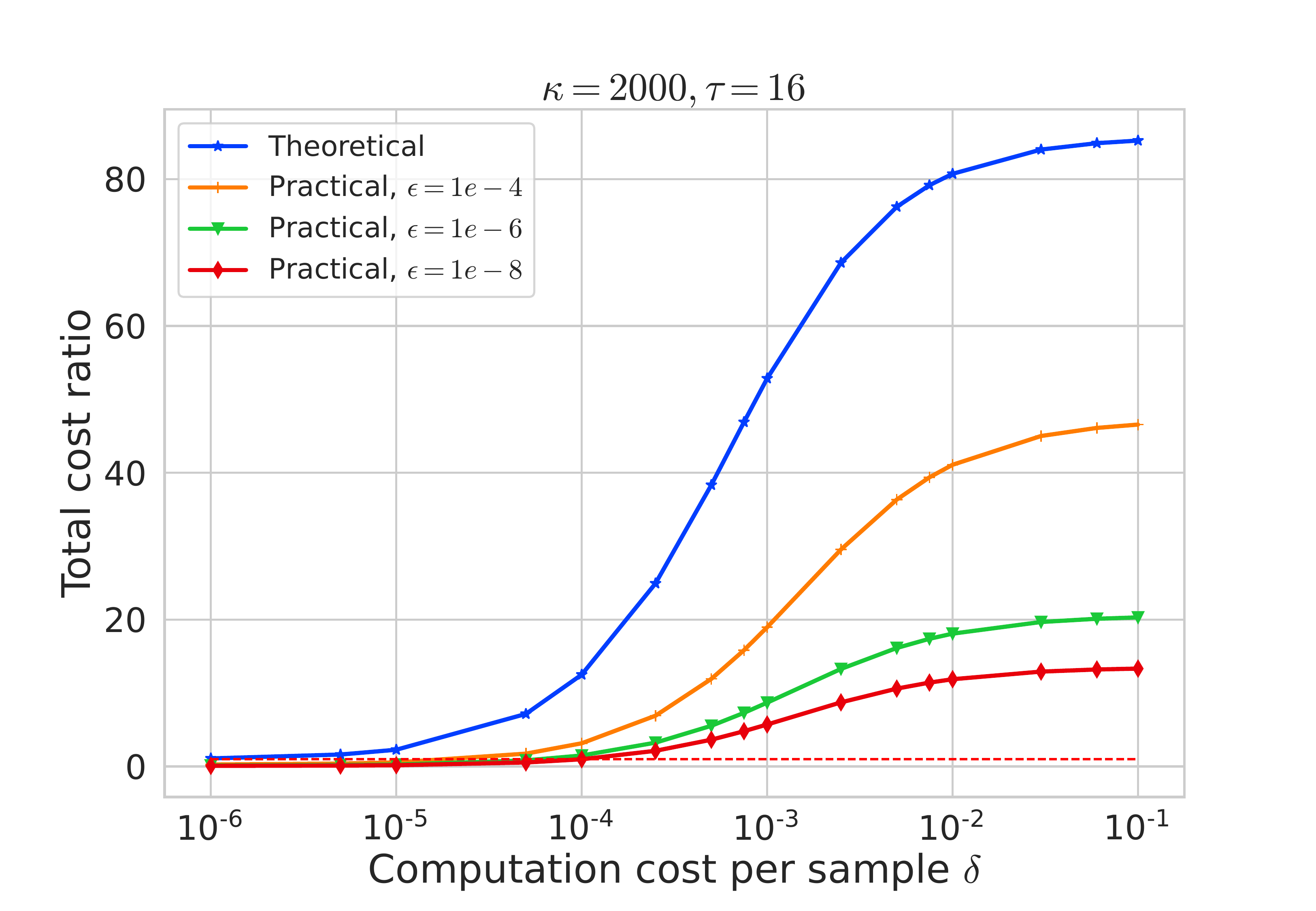}
		\caption{$\tau=16$.}
	\end{subfigure}
	\hfill
	\begin{subfigure}[b]{0.32\textwidth}
		\centering
		\includegraphics[trim=20 10 40 40, clip, width=\textwidth]{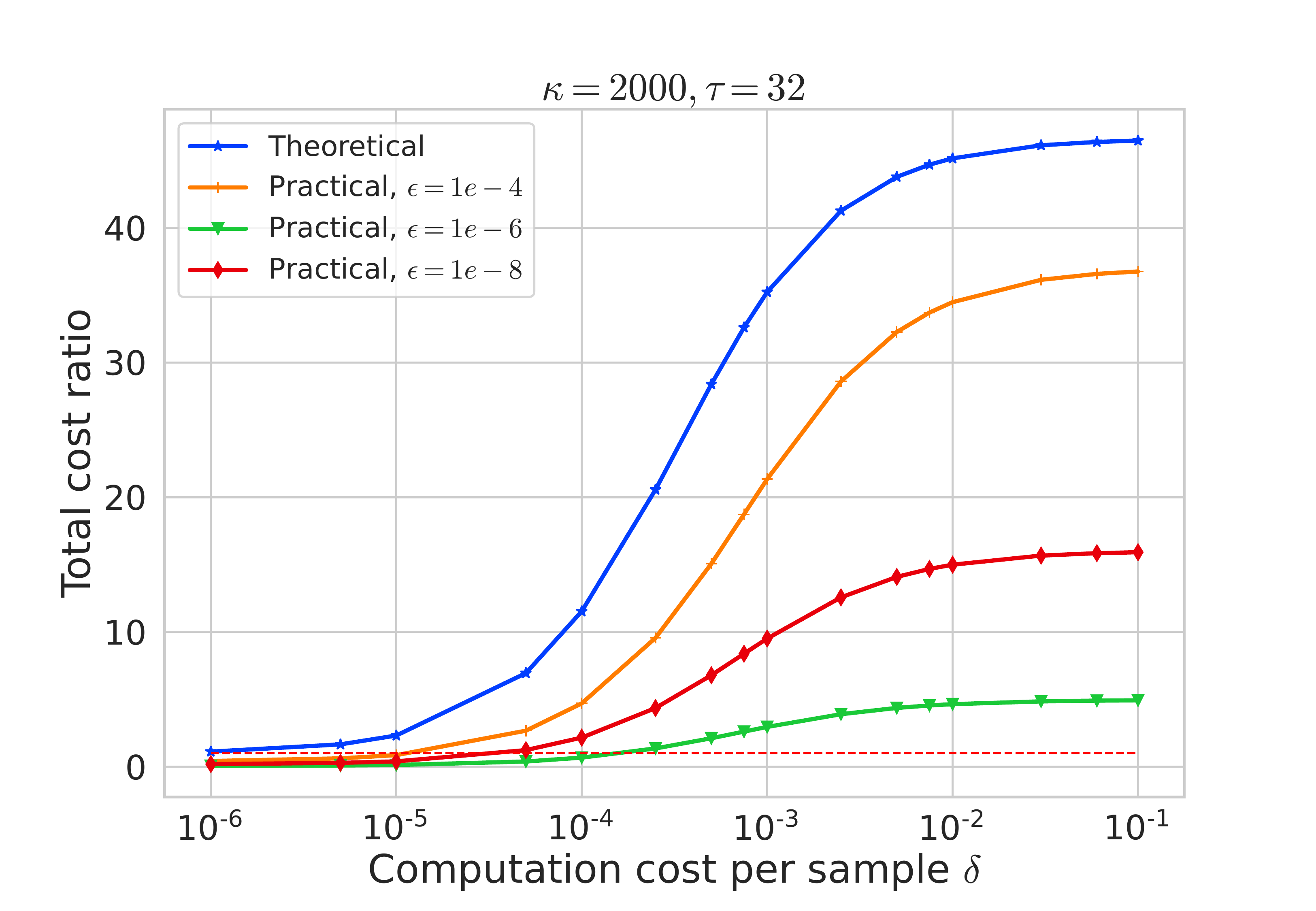}
		\caption{$\tau=32$.}
	\end{subfigure}
	\hfill
	\begin{subfigure}[b]{0.32\textwidth}
		\centering
		\includegraphics[trim=20 10 40 40, clip, width=\textwidth]{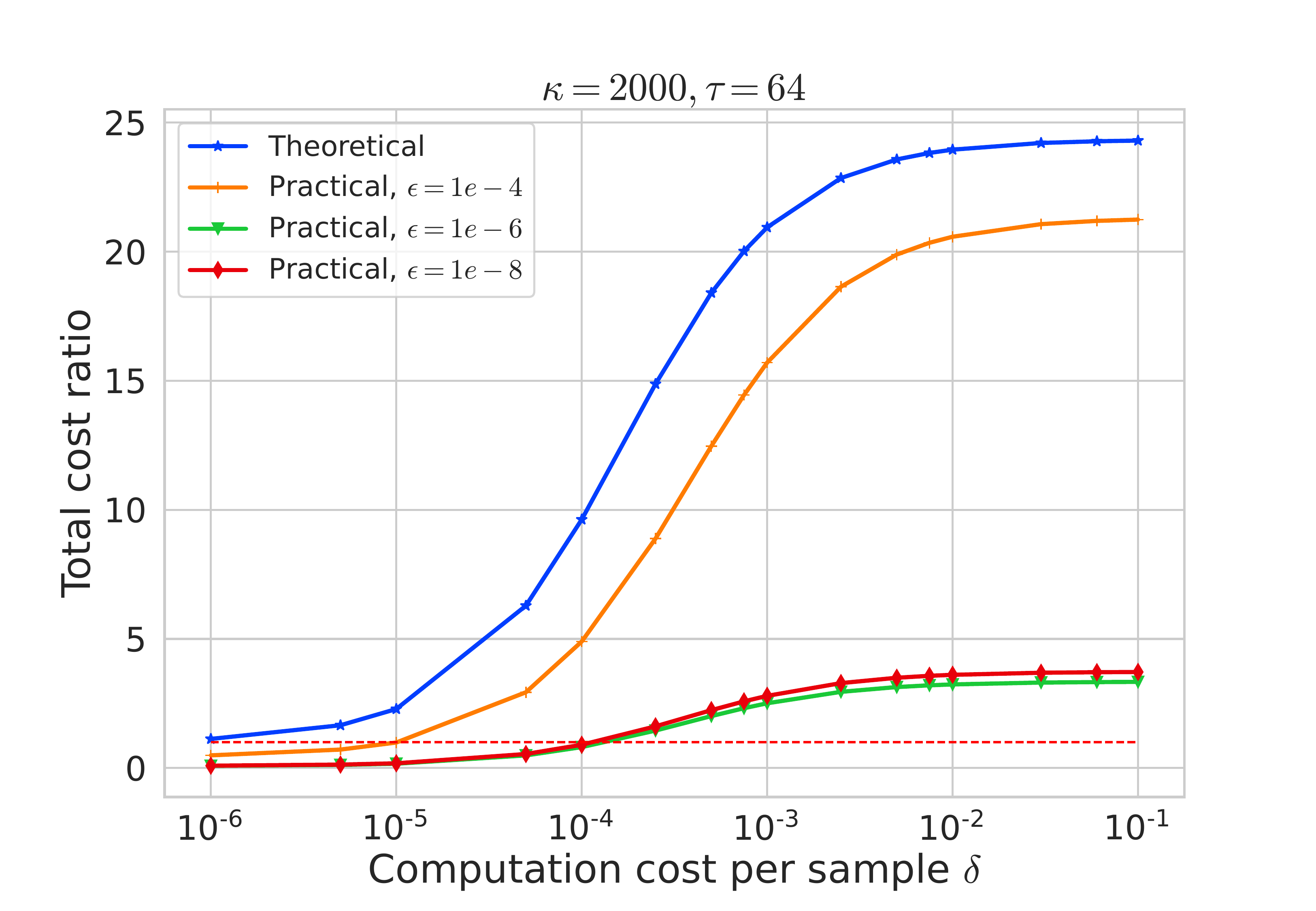}
		\caption{$\tau=64$.}
	\end{subfigure}
	\caption{Acceleration with 10 distributed workers on \dataset{a9a} dataset, $\kappa=2e3$.}
	\label{fig:056}
\end{figure} 

\begin{figure}[!htbp]
	\centering
	\begin{subfigure}[b]{0.32\textwidth}
		\centering
		\includegraphics[trim=20 10 40 40, clip, width=\textwidth]{img/0019_a9a_n10_bs16_cosize10_cc1_kappa_kappa1000.0_error_0.0001.txt.pdf}
		\caption{$\tau=16$.}
	\end{subfigure}
	\hfill
	\begin{subfigure}[b]{0.32\textwidth}
		\centering
		\includegraphics[trim=20 10 40 40, clip, width=\textwidth]{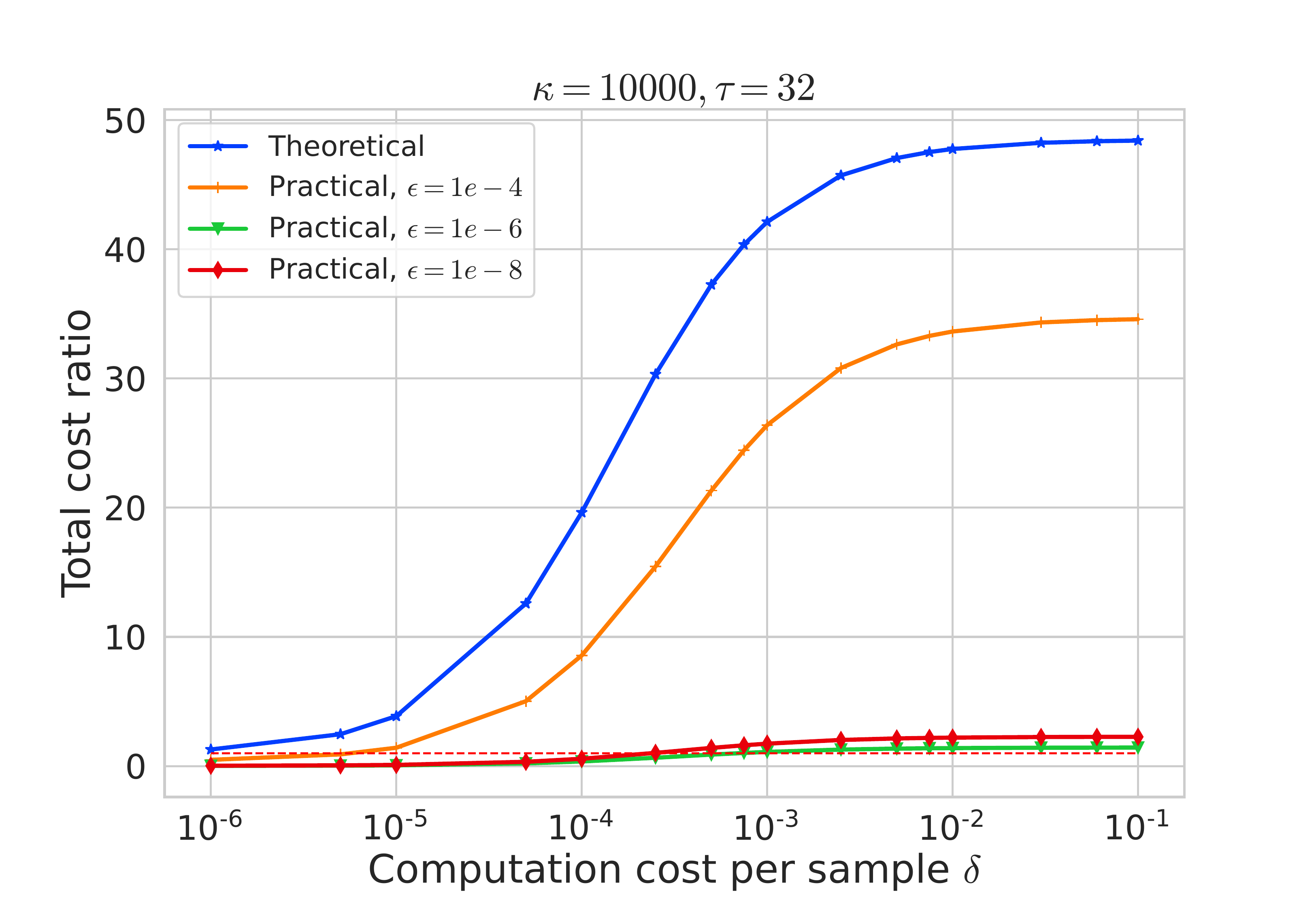}
		\caption{$\tau=32$.}
	\end{subfigure}
	\hfill
	\begin{subfigure}[b]{0.32\textwidth}
		\centering
		\includegraphics[trim=20 10 40 40, clip, width=\textwidth]{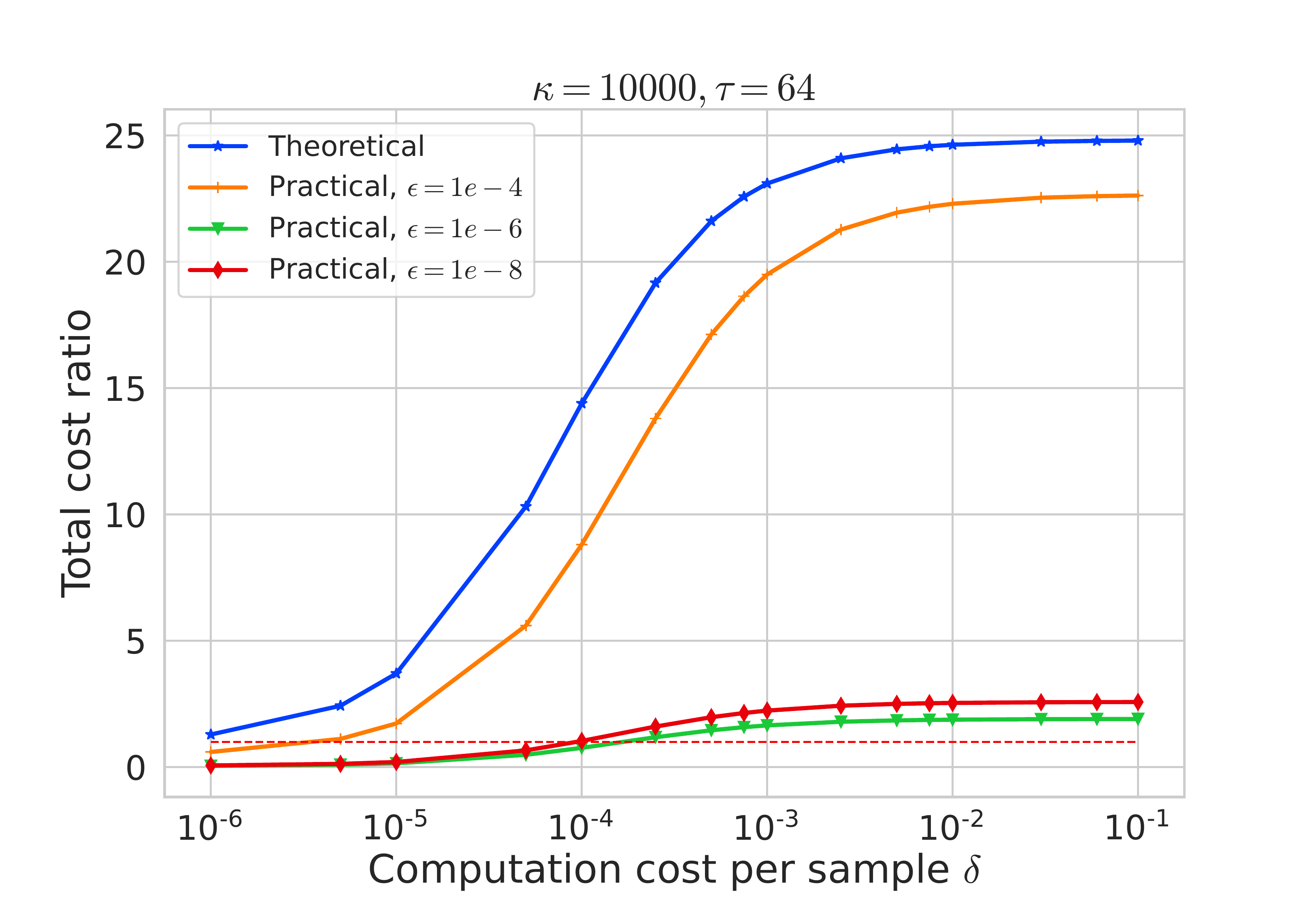}
		\caption{$\tau=64$.}
	\end{subfigure}
	\caption{Acceleration with 10 distributed workers on \dataset{a9a} dataset, $\kappa=1e4$.}
	\label{fig:057}
\end{figure} 

\begin{figure}[!htbp]
	\centering
	\begin{subfigure}[b]{0.32\textwidth}
		\centering
		\includegraphics[trim=20 10 40 40, clip, width=\textwidth]{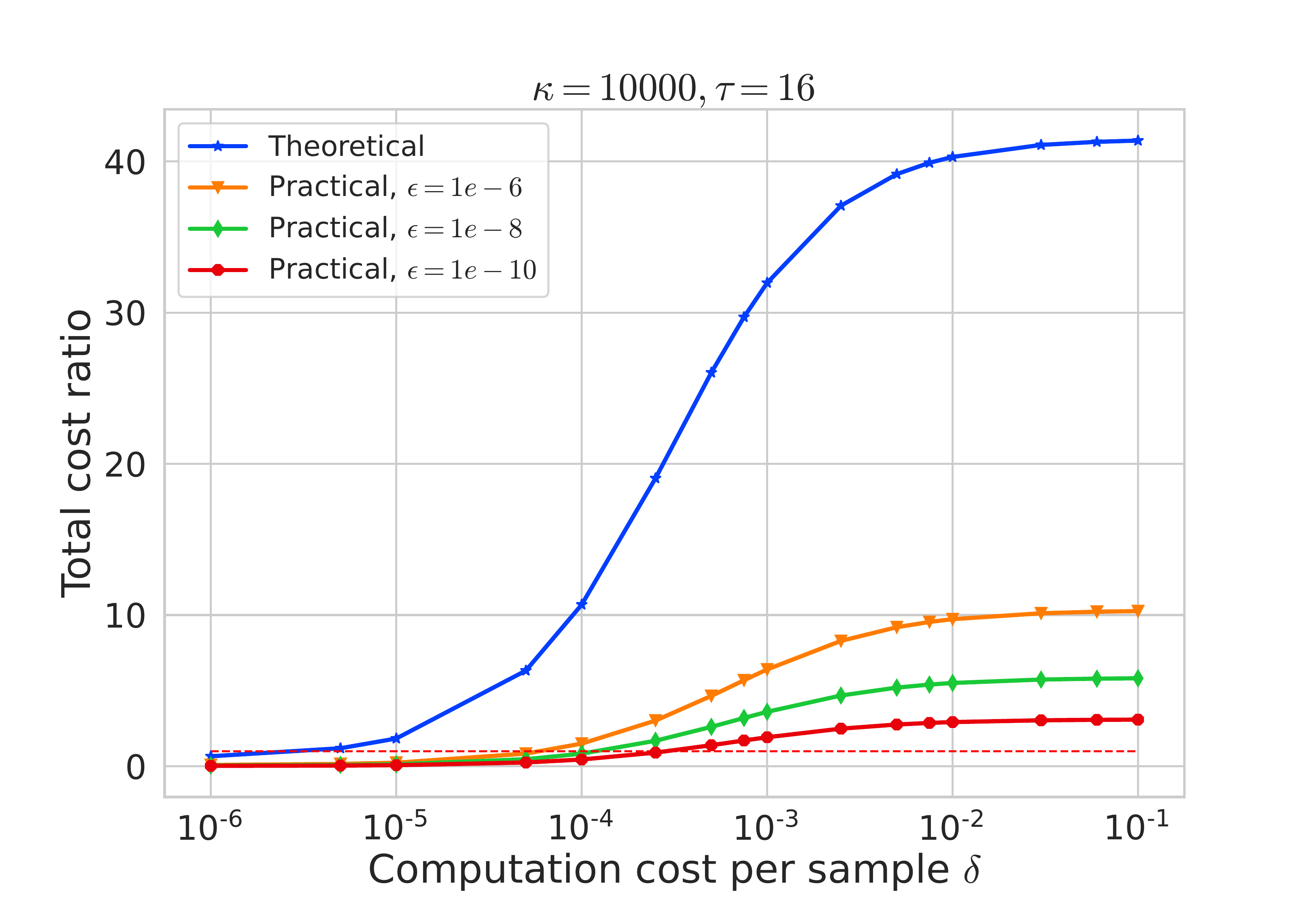}
		\caption{$\tau=16$.}
	\end{subfigure}
	\hfill
	\begin{subfigure}[b]{0.32\textwidth}
		\centering
		\includegraphics[trim=20 10 40 40, clip, width=\textwidth]{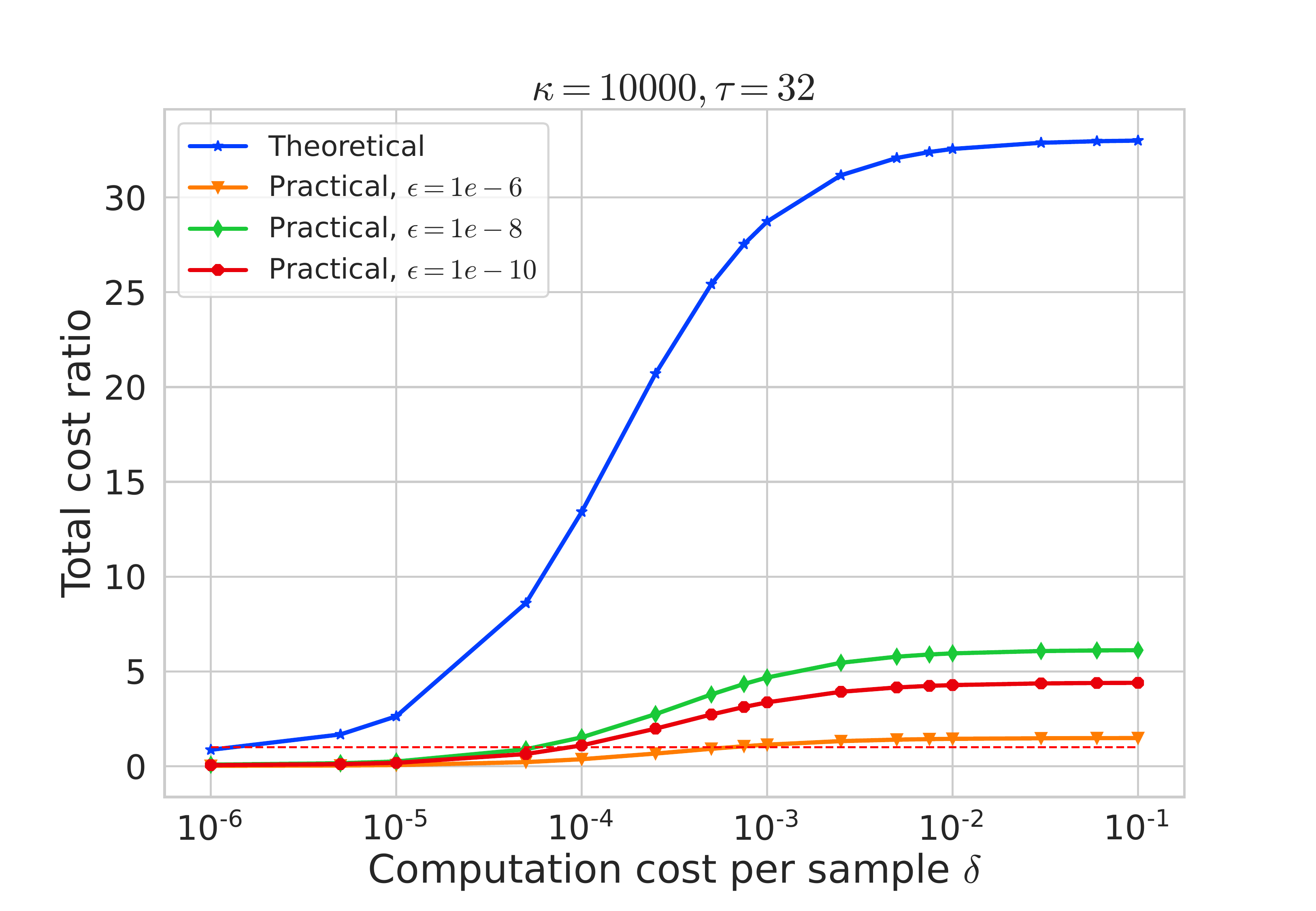}
		\caption{$\tau=32$.}
	\end{subfigure}
	\hfill
	\begin{subfigure}[b]{0.32\textwidth}
		\centering
		\includegraphics[trim=20 10 40 40, clip, width=\textwidth]{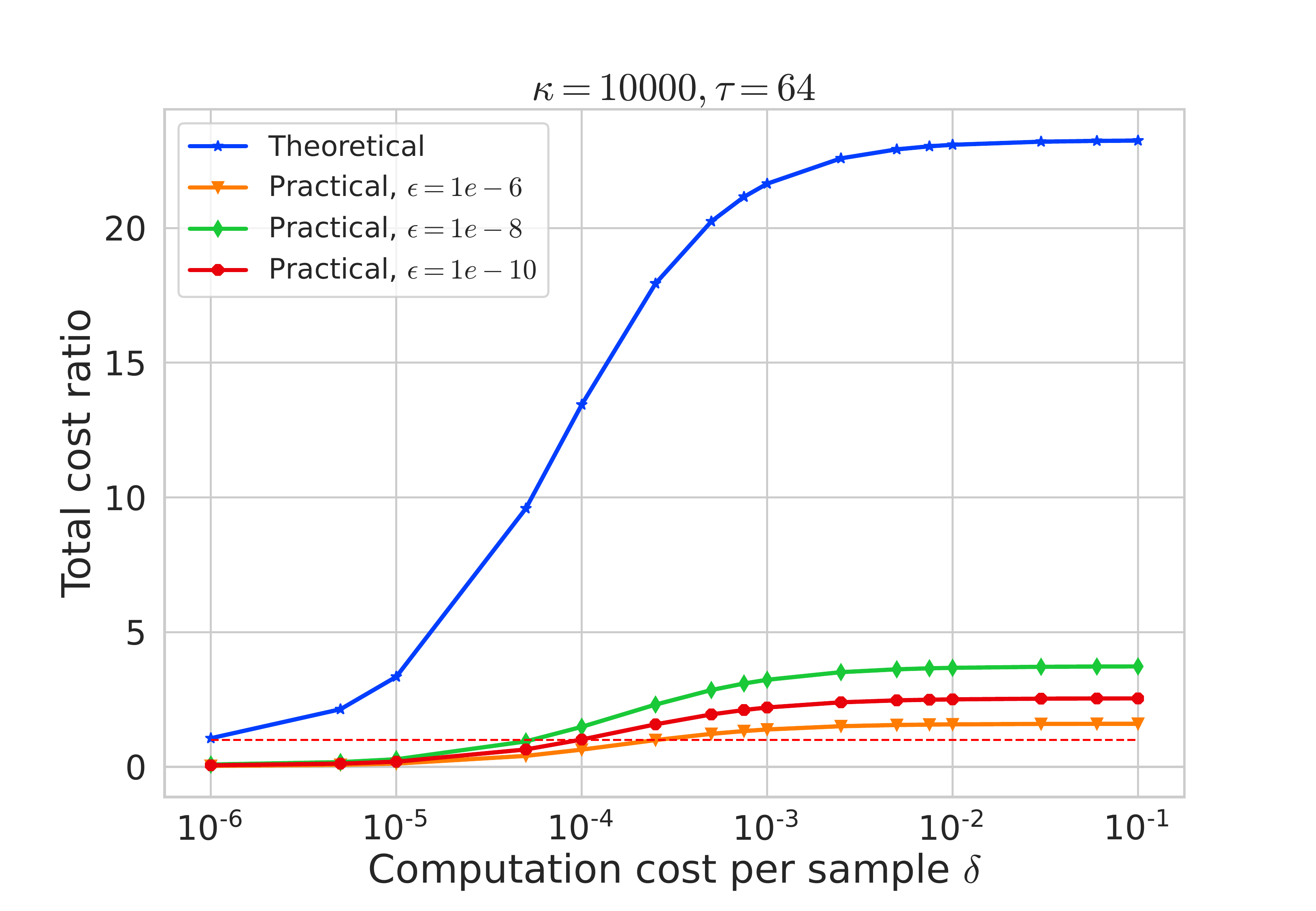}
		\caption{$\tau=64$.}
	\end{subfigure}
	\caption{Acceleration with 10 distributed workers on \dataset{w8a} dataset, $\kappa=1e4$.}
	\label{fig:058}
\end{figure} 

\subsection{Experiments with ProxSkip-HUB}
In this work we introduced a new FL architecture: regional hubs connecting the clients to the server; see Section~\ref{sec:tree}.  For conceptual simplicity, and in order to facilitate fair comparison with \algname{ProxSkip-LSVRG}, we assume that the number of hubs equals to  the number of clients, and that each client owns a single datapoint only. We compare \algname{ProxSkip-HUB} with \algname{ProxSkip-LSVRG} to check whether communication compression leads to any benefits in terms of  total costs. Theoretically, and similarly to our analysis in Section~\ref{sec:experiments}, the total cost for \algname{ProxSkip-LSVRG} is 
$$\text{Cost}(\text{\algname{ProxSkip-LSVRG}}) \eqdef T_{\text{comm.}}(\text{\algname{ProxSkip-LSVRG}})\\
+ \delta   \left(q m + (1-q)\tau+\tau\right)T_{\text{iter}}({\text{\algname{ProxSkip-LSVRG}}}).$$


Recall that we assume the communication cost from every worker/hub to the master is equal to 1, and the computation cost per sample is equal to $\delta$. Here we generalize to the multi-level structure. We assume that the communication cost from every client to hub is equal to $\delta'$. We choose the Rand-$k$ sparsification for \algname{ProxSkip-HUB}; this compressor selects $k$-entries of the gradient vector, uniformly at random from the full $d$-dimensional gradient. The total cost of \algname{ProxSkip-HUB} is
\begin{equation}
	\begin{aligned}
		\text{Cost}(\text{\algname{ProxSkip-HUB}}) &:= T_{\text{comm.}} (\text{\algname{ProxSkip-HUB}})\\
		&\quad + \delta' \left( q m+  \frac{k}{d} \left( (1-q) \tau+\tau\right)  \right) T_{\text{iter}} (\text{\algname{ProxSkip-HUB}}).\\
	\end{aligned}
\end{equation}

Our experimental results are summarized Figure~\ref{fig:060}; we use the values $\delta = \delta’ = 10^{-2}$. Clearly, and thanks to communication compression,   \algname{ProxSkip-HUB} has  benefit in terms of the total cost compared to \algname{ProxSkip-LSVRG}, and can reach up to three degrees of magnitude! 

\begin{figure}[!htbp]
	\centering
	\begin{subfigure}[b]{0.32\textwidth}
		\centering
		\includegraphics[trim=20 10 40 40, clip, width=\textwidth]{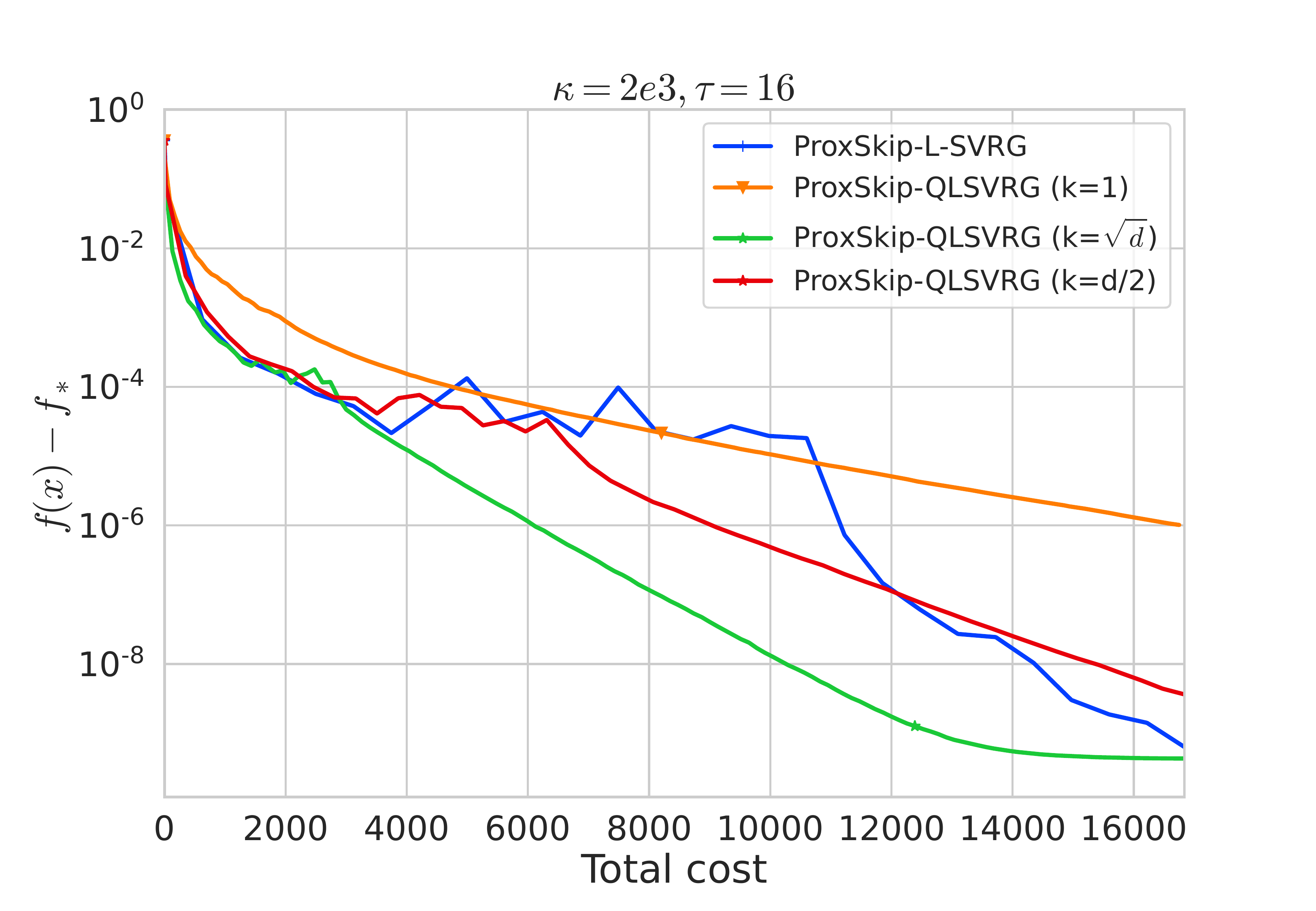}
		\caption{$\tau=16$.}
	\end{subfigure}
	\hfill
	\begin{subfigure}[b]{0.32\textwidth}
		\centering
		\includegraphics[trim=20 10 40 40, clip, width=\textwidth]{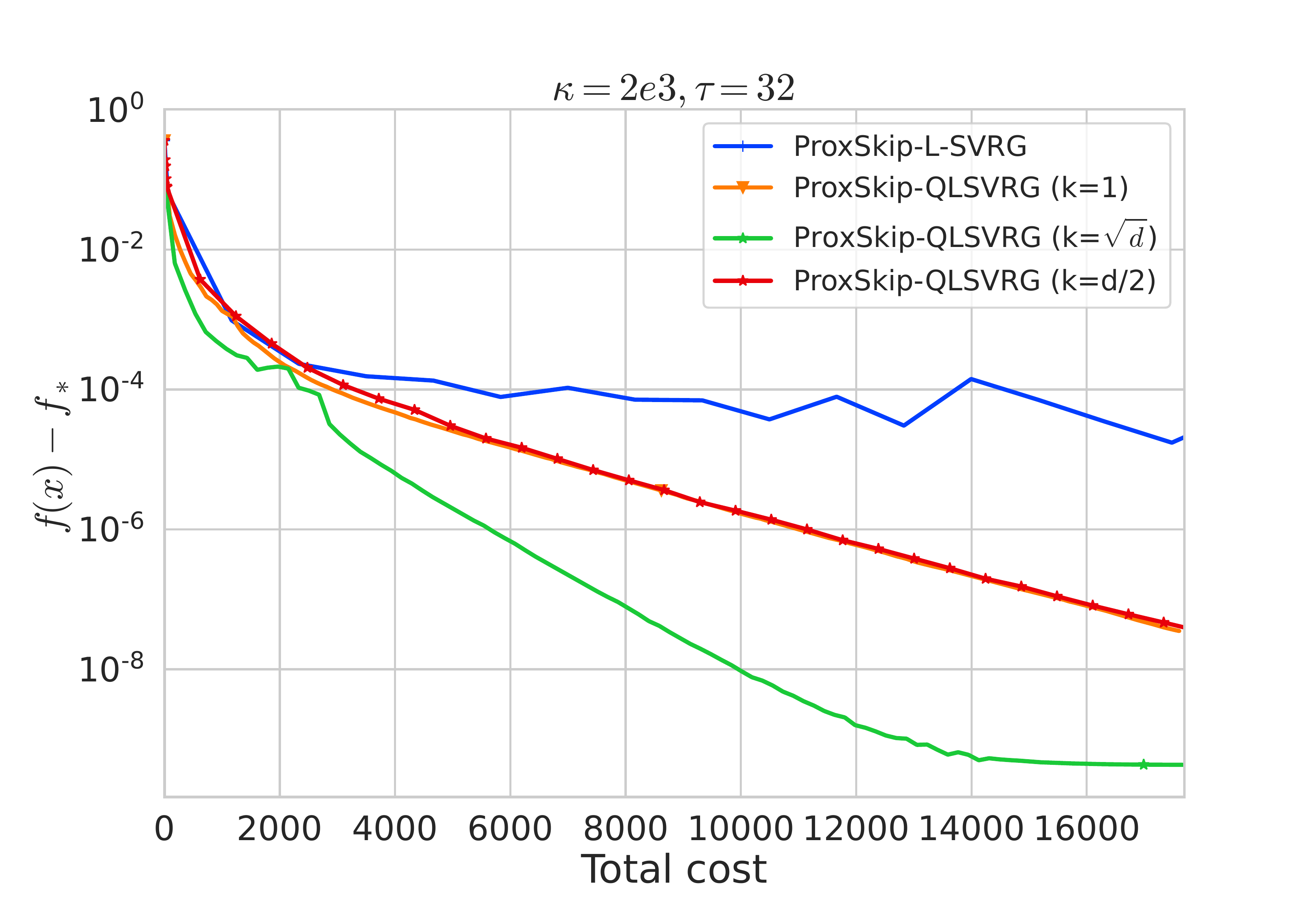}
		\caption{$\tau=32$.}
	\end{subfigure}
	\hfill
	\begin{subfigure}[b]{0.32\textwidth}
		\centering
		\includegraphics[trim=20 10 40 40, clip, width=\textwidth]{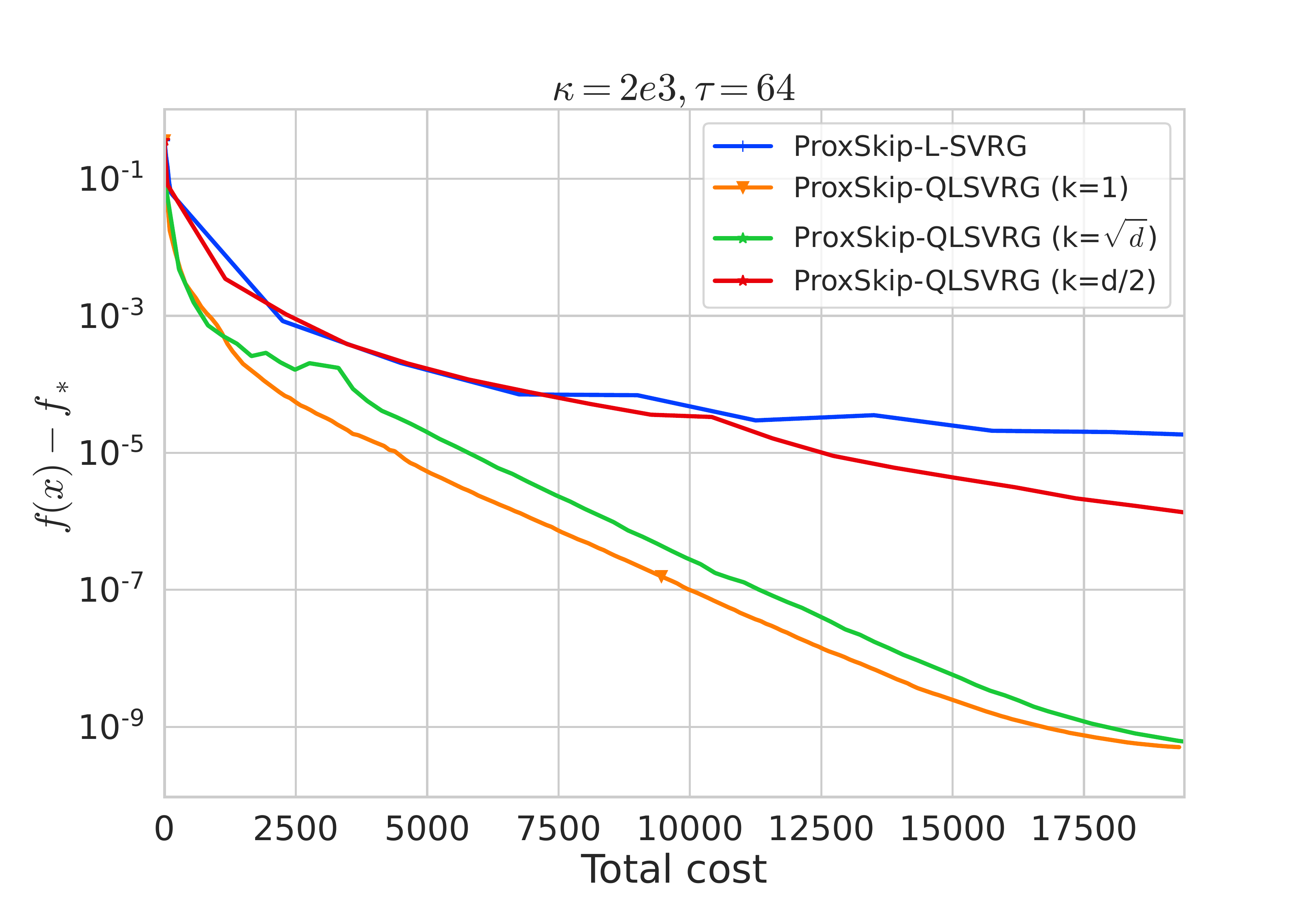}
		\caption{$\tau=64$.}
	\end{subfigure}
	\caption{Convergence results with different batch sizes and sparsification parameter  $k$ on \dataset{a9a}, $\kappa=2e3$.}
	\label{fig:060}
\end{figure}

\section*{Acknowledgements}

We would like to thank Eduard Gorbunov for useful discussions related to some aspects of the theory.

\bibliography{VRPS}

\clearpage

\appendix

\part*{Appendix}

\section{Basic Facts}

\subsection{Bregman divergence, $L$-smoothness and $\mu$-strong convexity}
The Bregman divergence of a differentiable function $f: \mathbb{R}^d \rightarrow \mathbb{R}$ is defined by
\begin{align}
	\label{eq:bregman_divergence}
	D_{f}(x, y):=f(x)-f(y)-\langle\nabla f(y), x-y\rangle.
\end{align}
It is easy to see that 
\begin{align}
	\label{eq:sum_bregman}
\langle\nabla f(x)-\nabla f(y), x-y\rangle=D_{f}(x, y)+D_{f}(y, x), \quad \forall x, y \in \mathbb{R}^{d}.
\end{align}
For an $L$-smooth and $\mu$-strongly convex function $f: \mathbb{R}^d \rightarrow \mathbb{R}$, we have
\begin{align}
	\label{eq:smooth-norms}
\frac{\mu}{2}\|x-y\|^{2} \leq D_{f}(x, y) \leq \frac{L}{2}\|x-y\|^{2}, \quad \forall x, y \in \mathbb{R}^{d}
\end{align}
and
\begin{align}
	\label{eq:smooth-grad}
	\frac{1}{2 L}\|\nabla f(x)-\nabla f(y)\|^{2} \leq D_{f}(x, y) \leq \frac{1}{2 \mu}\|\nabla f(x)-\nabla f(y)\|^{2}, \quad \forall x, y \in \mathbb{R}^{d}.
\end{align}

\subsection{Firm-nonexpansiveness of the proximity operator}
Given $\psi: \mathbb{R}^{d} \rightarrow \mathbb{R}$, we define $\psi^{*}(y):=\sup _{x \in \mathbb{R}^{d}}\{\langle x, y\rangle-\psi(x)\}$ to be its Fenchel conjugate. The proximity operator of $\psi^{*}$ satisfies for any $\tau>0$
\begin{align}
	\label{eq:fenhel}
	\text{ if } u=\operatorname{prox}_{\tau \psi^{*}}(y), \quad \text{ then } \quad u \in y-\tau \partial \psi^{*}(u).
\end{align}
If Assumption~\ref{ass:Reg} is satisfied, then firm nonexpansiveness of the proximity operator implies~\citep{ProxSkip}  that
\begin{align}
	\label{eq:firm-prox}
	\left\|\operatorname{prox}_{\frac{\gamma}{p} \psi}(x)-\operatorname{prox}_{\frac{\gamma}{p} \psi}(y) \right\|^{2}+ \left\| \left(x- \operatorname{prox}_{\frac{\gamma}{p} \psi}(x) \right) - \left(y-\operatorname{prox}_{\frac{\gamma}{p} \psi}(y)\right) \right\|^{2} \leq\|x-y\|^{2},
\end{align}
for all $x, y \in \mathbb{R}^{d}$ and any $\gamma, p>0$.

\subsection{Young's inequality} For any two vectors $ a, b \in \mathbb{R}^{d}$, we have 
\begin{equation}
	\label{youngs}
	\|a+b\|^{2} \leq 2\|a\|^{2}+2\|b\|^{2}.
\end{equation}

\subsection{Jensen’s inequality} For a convex function $h : \R^d \leftarrow \R$ and any vectors $x_1, \ldots , x_n \in \R^d$, we have 
\begin{align}
	h\left(\frac{1}{n} \sum_{i=1}^{n} x_{i}\right) \leq \frac{1}{n} \sum_{i=1}^{n} h\left(x_{i}\right).
\end{align}
Applying this to the squared norm, $h(x) = \|x\|^2,$ we get
\begin{align}
	\label{jensen}
	\left\|\frac{1}{n} \sum_{i=1}^{n} y_{i}\right\|^{2} \leq \frac{1}{n} \sum_{i=1}^{n}\left\|y_{i}\right\|^{2}.
\end{align}

\clearpage
\section{Analysis of ProxSkip-VR}
In this section we provide the proof of  Theorem~\ref{thm:main}.
\subsection{Main lemma of ProxSkip}
 We start from  Lemma~\ref{lem:main_lemma} initially introduced by~\citet{ProxSkip}; for completeness we provide the whole proof. Let us define two additional sequences:
 \begin{align}
 	\label{eq:seqs}
 \hat{w}_{t}=x_{t}-\gamma \hat{g}_{t}\left(x_{t}\right), \qquad  \hat{w}_{\star}=x_{\star}-\gamma \hat{g}_{t}\left(x_{\star}\right)  .
 \end{align}
\begin{lemma}
	\label{lem:main_lemma}
	If Assumption~\ref{ass:Reg} holds, $\gamma > 0$ and $0 < p \leq 1$, then the iterates of \algname{ProxSkip-VR} satisfy
	\begin{align}
		\label{eq:lemma1}
		\Exp{\|x_{t+1} - x_\star\|^2+\frac{\gamma^2}{p^2}\|h_{t+1} - h_{\star}\|^2} \leq\left\|\hat{w}_{t}-w_{\star}\right\|^{2}+\left(1-p^{2}\right) \frac{\gamma^{2}}{p^{2}}\left\|h_{t}-h_{\star}\right\|^{2}.
	\end{align}
\end{lemma}
\begin{proof}
	 In order to simplify, let us define two points:
	 \begin{align}
	 	\label{eq:two_points}
	 	x\eqdef\hat{x}_{t+1}-\frac{\gamma}{p} h_{t}, \quad y\eqdef x_{\star}-\frac{\gamma}{p} h_{\star}.
	 \end{align}
 \paragraph{STEP 1 (Optimality conditions).} Using the first-order optimality conditions for $f + \psi$ and using $h_\star \eqdef \nabla f (x_\star)$, we obtain the following fixed-point identity for $x_\star$:
 \begin{align}
 	\label{eq:opt}
 	x_{\star}=\operatorname{prox}_{\frac{\gamma}{p} \psi}\left(x_{\star}-\frac{\gamma}{p} h_{\star}\right)=	\operatorname{prox}_{\frac{\gamma}{p} \psi}(y) .
 \end{align}

\paragraph{STEP 2 (Recalling the steps of the method).}
Recall that the vectors $x_t$ and $h_t$ are in Algorithm~\ref{alg:ProxSkip-VR} updated as follows:
\begin{align}
 	\label{eq:x_upd}
	x_{t+1}=\left\{\begin{array}{lll}
		\operatorname{prox}_{\frac{\gamma}{p} \psi}(x) & \text { with probability } & p \\
		\hat{x}_{t+1} & \text { with probability } & 1-p
	\end{array}\right.
\end{align}
and 
\begin{align}
		 	\label{eq:h_upd}
	h_{t+1}=h_{t}+\frac{p}{\gamma}\left(x_{t+1}-\hat{x}_{t+1}\right) = \left\{\begin{array}{lll}
		h_{t}+\frac{p}{\gamma}\left(	\operatorname{prox}_{\frac{\gamma}{p} \psi}(x) -\hat{x}_{t+1}\right) & \text { with probability } p \\
		h_{t} & \text { with probability } 1-p
	\end{array}\right..
\end{align}
Let us consider the expected value $V_{t+1}\eqdef \Exp{\left\|x_{t+1}-x_{\star}\right\|^{2}+\frac{\gamma^{2}}{p^{2}}\left\|h_{t+1}-h_{\star}\right\|^{2}}$:
\begin{eqnarray}
	V_{t+1} & \stackrel{(\ref{eq:x_upd})+(\ref{eq:h_upd})}{=} &  p\left(\left\|	\operatorname{prox}_{\frac{\gamma}{p} \psi}(x)-x_{\star}\right\|^{2}+\frac{\gamma^{2}}{p^{2}}\left\|h_{t}+\frac{p}{\gamma}\left(	\operatorname{prox}_{\frac{\gamma}{p} \psi}(x)-\hat{x}_{t+1}\right)-h_{\star}\right\|^{2}\right) \notag \\
	&&\quad +(1-p)\left(\left\|\hat{x}_{t+1}-x_{\star}\right\|^{2}+\frac{\gamma^{2}}{p^{2}}\left\|h_{t}-h_{\star}\right\|^{2}\right)\notag \\
	&\stackrel{(\ref{eq:opt})}{=}& p\left( \left\|	\operatorname{prox}_{\frac{\gamma}{p} \psi}(x)-	\operatorname{prox}_{\frac{\gamma}{p} \psi}(y) \right\|^{2}+\left\|\frac{\gamma}{p} h_{t}+	\operatorname{prox}_{\frac{\gamma}{p} \psi}(x)-\hat{x}_{t+1}-\frac{\gamma}{p} h_{\star}\right\|^{2}\right)\notag \\
	&&\quad +(1-p)\left(\left\|\hat{x}_{t+1}-x_{\star}\right\|^{2}+\frac{\gamma^{2}}{p^{2}}\left\|h_{t}-h_{\star}\right\|^{2}\right)\notag \\
	&\stackrel{(\ref{eq:two_points})+(\ref{eq:opt})}{=} & p\left(\left\|	\operatorname{prox}_{\frac{\gamma}{p} \psi}(x)-	\operatorname{prox}_{\frac{\gamma}{p} \psi}(y)\right\|^{2}+\left\|	\operatorname{prox}_{\frac{\gamma}{p} \psi}(x)-x+y-	\operatorname{prox}_{\frac{\gamma}{p} \psi}(y)\right\|^{2}\right)\notag \\
	&&\quad +(1-p)\left(\left\|\hat{x}_{t+1}-x_{\star}\right\|^{2}+\frac{\gamma^{2}}{p^{2}}\left\|h_{t}-h_{\star}\right\|^{2}\right).\label{eq:9u0fd9h0fd9h}
\end{eqnarray}
\paragraph{STEP 4 (Applying firm nonexpansiveness).}Applying firm nonexpansiveness of the proximal operator~\eqref{eq:firm-prox}, this leads to the inequality
\begin{eqnarray*}
V_{t+1}& \stackrel{\eqref{eq:9u0fd9h0fd9h}+(\ref{eq:firm-prox})}{\leq} & p\|x-y\|^{2} +(1-p)\left(\left\|\hat{x}_{t+1}-x_{\star}\right\|^{2}+\frac{\gamma^{2}}{p^{2}}\left\|h_{t}-h_{\star}\right\|^{2}\right) \\ \notag
	& \stackrel{(\ref{eq:two_points})}{=}&  p\left\|\hat{x}_{t+1}-\frac{\gamma}{p} h_{t}-\left(x_{\star}-\frac{\gamma}{p} h_{\star}\right)\right\|^{2} +(1-p)\left(\left\|\hat{x}_{t+1}-x_{\star}\right\|^{2}+\frac{\gamma^{2}}{p^{2}}\left\|h_{t}-h_{\star}\right\|^{2}\right) .
\end{eqnarray*}
\paragraph{STEP 5 (Simple algebra).} Next, we expand the squared norm and collect the terms, obtaining
\begin{eqnarray}
V_{t+1}& \leq & p\left\|\hat{x}_{t+1}-x_{\star}\right\|^{2} +p \frac{\gamma^{2}}{p^{2}}\left\|h_{t}-h_{\star}\right\|^{2}-2 \gamma\left\langle\hat{x}_{t+1}-x_{\star}, h_{t}-h_{\star}\right\rangle \notag\\
&&\quad +(1-p)\left(\left\|\hat{x}_{t+1}-x_{\star}\right\|^{2}  +\frac{\gamma^{2}}{p^{2}}\left\|h_{t}-h_{\star}\right\|^{2}\right) \notag \\
	&=& \left\|\hat{x}_{t+1}-x_{\star}\right\|^{2}-2 \gamma\left\langle\hat{x}_{t+1}-x_{\star}, h_{t}-h_{\star}\right\rangle +\frac{\gamma^{2}}{p^{2}}\left\|h_{t}-h_{\star}\right\|^{2}.\label{eq:last1}	
\end{eqnarray}
Finally, note that by our definition of $\hat{w}_t$, we have the identity $\hat{x}_{t+1}=\hat{w}_{t}+\gamma h_{t}$. Therefore, the first two terms above can be rewritten as
\begin{align}
		\label{eq:last2}
\notag	\left\|\hat{x}_{t+1}-x_{\star}\right\|^{2}-2 \gamma\left\langle\hat{x}_{t+1}-x_{\star}, h_{t}-h_{\star}\right\rangle&=\left\|\hat{w}_{t}-w_{\star}+\gamma\left(h_{t}-h_{\star}\right)\right\|^{2}\\
\notag	&\quad -2 \gamma\left\langle \hat{w}_{t}-w_{\star}+\gamma\left(h_{t}-h_{\star}\right), h_{t}-h_{\star}\right\rangle \\
\notag	&=\left\|\hat{w}_{t}-w_{\star}\right\|^{2}+2 \gamma\left\langle \hat{w}_{t}-w_{\star}, h_{t}-h_{\star}\right\rangle\\
\notag	&\quad +\gamma^{2}\left\|h_{t}-h_{\star}\right\|^{2} 
	-2 \gamma\left\langle \hat{w}_{t}-w_{\star}, h_{t}-h_{\star}\right\rangle\\
\notag	&\quad -2 \gamma^{2}\left\|h_{t}-h_{\star}\right\|^{2} \\
	&=\left\|\hat{w}_{t}-w_{\star}\right\|^{2}-\gamma^{2}\left\|h_{t}-h_{\star}\right\|^{2}.
\end{align}
Finally, plugging \eqref{eq:last2} into \eqref{eq:last1}, we get:
	\begin{align*}
	V_{t+1} \leq\left\|\hat{w}_{t}-w_{\star}\right\|^{2}+\left(1-p^{2}\right) \frac{\gamma^{2}}{p^{2}}\left\|h_{t}-h_{\star}\right\|^{2}.
\end{align*}
\end{proof}
\subsection{Main lemma }
This lemma allows us to obtain a useful recursion for variance-reduced stochastic estimators used in our \algname{ProxSkip-VR} algorithm. 
\begin{lemma}
	Let Assumptions~\ref{ass:mu-strongly-convex} and \ref{sigma_t} hold. Then the iterates of \algname{ProxSkip-VR} satisfy 
	\begin{align*}
\notag	\Exp{\|\hat{w}_t - w_\star\|^2}  \leq  (1-\gamma\mu)\|x_t - x_\star\|^2 - 2\gamma D_f(x_t,x_\star)\left(1-\gamma A\right) + \gamma^2B\sigma_{t} + \gamma^2C.
	\end{align*}
\end{lemma}
\begin{proof}
	We start from the definitions of the auxiliary sequence $\hat{w}_t$ (see \eqref{eq:seqs}):
	\begin{eqnarray}
		\label{32}
\notag		\|\hat{w}_t - w_\star\|^2 &\overset{\eqref{eq:seqs}}{=}& \| x_t - \gamma\hat{g}_t - (x_\star - \gamma\nabla f(x_\star) ) \|^2\\
	\notag	&=&\|  (x_t - x_\star ) - \gamma(\hat{g}_t - \nabla f(x_\star))  \|^2\\
			& =& \|x_t - x_\star\|^2 - 2\gamma\left\langle x_t - x_\star, \hat{g}_t - \nabla f(x_\star) \right\rangle + \gamma^2 \| \hat{g}_t - \nabla f(x_\star) \|^2.
	\end{eqnarray}
Taking expectation in~\eqref{32} and using unbiasedness of $\hat{g}_t$ (see \eqref{eq:unbiased} in Assumption~\ref{sigma_t}), we get 
\begin{align}
	\label{33}
	\Exp{\|\hat{w}_t - w_\star\|^2}  \overset{\eqref{32}+\eqref{eq:unbiased}}{=} \|x_t - x_\star\|^2 - 2\gamma\left\langle x_t - x_\star, \nabla f(x_t) - \nabla f(x_\star) \right\rangle + \gamma^2 \Exp{\| \hat{g}_t - \nabla f(x_\star) \|^2}.
\end{align}
Let us now consider the inner product in~\eqref{33}. Using~\eqref{eq:smooth-norms} and~\eqref{eq:sum_bregman}, we obtain
\begin{align}
		\label{34}
	\Exp{\|\hat{w}_t - w_\star\|^2}  \leq (1-\gamma\mu)\|x_t - x_\star\|^2 - 2\gamma D_f(x_t,x_\star) + \gamma^2 \Exp{\| \hat{g}_t - \nabla f(x_\star) \|^2}.
\end{align}
To bound the last term in \eqref{34}, we can apply Assumption~\ref{sigma_t}:
\begin{align}
	\label{35}
	\Exp{\| \hat{g}_t - \nabla f(x_\star) \|^2} \leq 2AD_f(x_t,x_\star)+B\sigma_{t} + C.
\end{align}
Plugging~\eqref{35} into~\eqref{34} gives us
\begin{align}
	\label{eq:le2}
\notag	\Exp{\|\hat{w}_t - w_\star\|^2}  &\leq (1-\gamma\mu)\|x_t - x_\star\|^2 - 2\gamma D_f(x_t,x_\star) + \gamma^2 \left(2AD_f(x_t,x_\star)+B\sigma_{t} + C\right)\\
	& \leq  (1-\gamma\mu)\|x_t - x_\star\|^2 - 2\gamma D_f(x_t,x_\star)\left(1-\gamma A\right) + \gamma^2B\sigma_{t} + \gamma^2C,
\end{align}
which is what we set out to prove.
\end{proof}

\subsection{Proof of Theorem~\ref{thm:main}}


\begin{proof}
	Using definition of the Lyapunov function $\Psi_t$, and the tower property of conditional expectation, we obtain
	\begin{eqnarray}
\notag		\Exp{\Psi_{t+1}} &=& 	\Exp{\|x_{t+1} - x_{\star}\|^2+\frac{\gamma^2}{p^2}\|h_{t+1} - h_{\star}\|^2+\MM\gamma^2\sigma_{t+1}} \notag \\
\notag		 &\stackrel{\eqref{eq:lemma1}}{\leq} & \Exp{\|\hat{w}_t - w_\star \|^2} + (1-p^2)\frac{\gamma^2}{p^2}\|h_t - h_\star\|^2 \notag  \\
	&& \quad +\MM \gamma^2 \left( 2\tilde{A}D_f(x_t,x_\star) + \tilde{B}\sigma_{t} + \tilde{C} \right) \notag \\
		&\stackrel{\eqref{eq:le2}}{\leq}	&	(1-\gamma\mu)\|x_t - x_\star\|^2 - 2\gamma D_f(x_t,x_\star)\left(1-\gamma A\right) + \gamma^2B\sigma_{t} + \gamma^2C	\notag  \\
			&& \quad + (1-p^2)\frac{\gamma^2}{p^2}\|h_t - h_\star\|^2 +\MM \gamma^2 \left( 2\tilde{A}D_f(x_t,x_\star) + \tilde{B}\sigma_{t} + \tilde{C} \right) \notag \\
		&\leq& 	(1-\gamma\mu)\|x_t - x_\star\|^2 - 2\gamma D_f(x_t,x_\star)\left(1-\gamma (A+\MM\tilde{A})\right) \notag \\
			&& \quad + \gamma^2\MM \sigma_{t}\left(\frac{B+\MM \tilde{B}}{\MM}\right) + \gamma^2(C+\MM\tilde{C})+(1-p^2)\frac{\gamma^2}{p^2}\|h_t - h_\star\|^2 .
	\end{eqnarray}

Using the stepsize bound $\gamma \leq \frac{1}{A+\MM\tilde{A}}$, this leads to 
	\begin{align}
	\notag		\Exp{\Psi_{t+1}} 	&\leq 	(1-\gamma\mu)\|x_t - x_\star\|^2  + \gamma^2\MM \sigma_{t}\left(\frac{B+\MM \tilde{B}}{\MM}\right) + \gamma^2(C+\MM\tilde{C})+(1-p^2)\frac{\gamma^2}{p^2}\|h_t - h_\star\|^2.
\end{align}

Let us denote $\beta \eqdef \frac{B+\MM \tilde{B}}{\MM}$. In order to obtain a contraction, we need to have $\beta< 1$, which is satisfied when  $ \MM > \frac{B}{1-\tilde{B}}$, and we get
	\begin{eqnarray}
		\Exp{\Psi_{t+1}} 	&\leq & 	(1-\gamma\mu)\|x_t - x_\star\|^2  + \gamma^2\MM \sigma_{t}\beta 			+ \gamma^2(C+\MM\tilde{C})+(1-p^2)\frac{\gamma^2}{p^2}\|h_t - h_\star\|^2 \notag\\
	&\leq & \max\left( 1-p^2,\beta,1-\gamma\mu \right)\Psi_t +  \gamma^2(C+\MM\tilde{C}) \label{eq:final} .
\end{eqnarray}

Finally, using the tower property of expectation and unrolling recursion~\eqref{eq:final}, we get 
	\begin{equation*}
	\Exp{ \Psi_{T} } \leq \max \left\{(1-\gamma \mu)^{T},\beta^{T},(1-p^2)^T\right\} \Psi_{0}+\frac{\left(C+\MM \tilde{C}\right) \gamma^{2}}{\min \left\{\gamma \mu,p^2, 1 - \beta \right\}}.
\end{equation*}

\end{proof}

\clearpage
\section{Examples of Methods Without Variance Reduction}
\subsection{Proof of Theorem~\ref{thm:proxskip} (GD estimator)}
\label{sec:GD_est}


\begin{proof}
	Let us show that \algname{GD} estimator $(\hat{g}_t = \nabla f(x_t))$ satisfies Assumption~\ref{sigma_t}
	\begin{align}
	\notag	\Exp{\|\hat{g}_t - \nabla f(x_\star)\|^2} = \|\nabla f(x_t)- \nabla f(x_\star)\|^2 \stackrel{\eqref{eq:smooth-norms}}{\leq} 2 L D_f(x_t,x_\star).
	\end{align}
This means that Assumption~\ref{sigma_t} is satisfied with the following constant:
\begin{align*}
	A = L, \quad B = 0, \quad C = 0, \quad \tilde{A} = 0, \quad \tilde{B} = 0,  \quad \tilde{C} = 0, \quad \sigma_t \equiv 0.
\end{align*}
Applying Theorem~\ref{thm:main} leads to final recursion:
 \begin{equation}
 	\label{eq:req_gd}
 	\Exp{\Psi_{T}} \leq \max \left\{(1-\gamma \mu)^{T},(1-p^2)^T\right\} \Psi_{0},
 \end{equation}
By inspecting \eqref{eq:req_gd} it is easy to see that
 \begin{equation}
T \geq \max \left\{\frac{1}{\gamma \mu}, \frac{1}{p^{2}}\right\} \log \frac{1}{\varepsilon} \qquad \Longrightarrow \qquad\Exp{\Psi_{T}} \leq \varepsilon \Psi_{0}.
 \end{equation}
Then the communication complexity is equal to
\begin{equation}
	p T \geq \max \left\{\frac{p}{\gamma \mu}, \frac{1}{p}\right\} \log \frac{1}{\varepsilon} .
\end{equation}
Setting $\gamma = \frac{1}{L}$ and solving $\frac{p L}{\mu}=\frac{1}{p}$ gives the optimal probability
\begin{equation}
	p=\sqrt{\frac{\mu}{L}}=\frac{1}{\sqrt{\kappa}}
\end{equation}
Finally, the iteration complexity and communication complexity have the following form:
\begin{align}
	T &\geq \max \left\{\frac{1}{\gamma \mu}, \frac{1}{p^{2}}\right\} \log \frac{1}{\varepsilon} = \kappa \log \frac{1}{\varepsilon},\\
	p T &\geq \max \left\{\frac{p}{\gamma \mu}, \frac{1}{p}\right\} \log \frac{1}{\varepsilon} = \sqrt{\kappa} \log \frac{1}{\varepsilon}.
\end{align}

\end{proof}
This recovers the result obtained by~\citet{ProxSkip}.

\subsection{Proof of Theorem~\ref{thm:sproxskip} (SGD estimator)}

 \begin{proof}
 		Let us show that the \algname{SGD} estimator $\hat{g}_t = g(x_t,\xi_t)$ satisfying Assumption~\ref{Expected_smoothness} also satisfies Assumption~\ref{sigma_t}. Using Young's inequality we get 
 		\begin{eqnarray}
 \notag			\Exp{\|\hat{g}_t - \nabla f(x_\star)\|^2} &=& \Exp{\|g(x_t,\xi_t)- \nabla f(x_\star)\|^2} \\
 	 \notag			& =&  \Exp{\|g(x_t,\xi_t)- g(x_\star,\xi_t) + g(x_\star,\xi_t) -\nabla f(x_\star)\|^2}\\
 		 \notag		&\stackrel{\eqref{youngs}}{\leq}& 2\Exp{\|g(x_t,\xi_t)- g(x_\star,\xi_t)\|^2} + 2\Exp{g(x_\star,\xi_t) -\nabla f(x_\star)\|^2}\\
 			& \stackrel{\eqref{Expected_smoothness}}{\leq}& 4 A^{\prime\prime} D_f(x_t,x_\star) + 2 {\rm Var}(g(x_{\star},\xi)) .
 		\end{eqnarray}
 	This means that Assumption~\ref{sigma_t} is satisfied with the following constants:
 	 \begin{align*}
 		A = 2A^{\prime\prime}, \quad B = 0, \quad C = 2{\rm Var}(g(x_{\star},\xi)), \quad \tilde{A} = 0, \quad \tilde{B} = 0,  \quad \tilde{C} = 0, \quad \sigma_t \equiv 0.
 	\end{align*}
 Applying Theorem~\ref{thm:main} leads to the final bound:
  \begin{equation}
  	\label{sgd_final}
 	\Exp{\Psi_{T}} \leq \max \left\{(1-\gamma \mu)^{T},(1-p^2)^T\right\} \Psi_{0} + \gamma^2 \frac{2{\rm Var}(g(x_{\star},\xi))}{\min\left\lbrace \gamma\mu,p^2 \right\rbrace}.
 \end{equation}
In order to minimize the number of prox evaluations, whatever the choice of $\gamma$ will be, we choose the smallest probability $p$ which does not lead to any degradation of the rate $\min\{\gamma\mu, p^2\}.$ That is, we choose $p = \sqrt{\gamma \mu}$. The first term on the right-hand side of~\eqref{sgd_final} can be bounded as follows:
\begin{equation*}
T \geq \frac{1}{\gamma \mu} \log \left(\frac{2 \Psi_{0}}{\varepsilon}\right) \quad \Longrightarrow \quad(1-\gamma\mu)^{T} \Psi_{0} \leq \frac{\varepsilon}{2}.
\end{equation*}
The second term on the right-hand side of~\eqref{sgd_final} can be bounded as follows:
\begin{equation*}
	\gamma \leq \frac{\varepsilon \mu}{2 C} \quad \Longrightarrow \quad \frac{\gamma C}{\mu} \leq \frac{\varepsilon}{2}.
\end{equation*}
We choose the largest stepsize consistent with bounds $	\gamma \leq \frac{\varepsilon \mu}{2 C}$ and $\gamma\leq \frac{1}{A}$:
\begin{align*}
	\gamma=\min \left\{\frac{1}{A}, \frac{\varepsilon \mu}{2 C}\right\}.
\end{align*}
Using this stepsizem we get the following iteration and (expected) communication complexities: 
\begin{align*}
	T \geq \max \left\{\frac{A}{\mu}, \frac{2 C}{\varepsilon \mu^{2}}\right\} \log \left(\frac{2 \Psi_{0}}{\varepsilon}\right),\qquad  	pT \geq \max \left\{\sqrt{\frac{A}{\mu}}, \sqrt{\frac{2 C}{\varepsilon \mu^{2}}}\right\} \log \left(\frac{2 \Psi_{0}}{\varepsilon}\right).
\end{align*}
This recovers the result obtained by~\citet{ProxSkip}.
 \end{proof}

 \clearpage
\section{Analysis of ProxSkip-HUB}

In this section we provide analysis of the new algorithm \algname{ProxSkip-HUB}, which works for the new FL formulation described in Section~\ref{sec:tree}. The pseudocode is presented in Algorithm~\ref{alg:ProxSkip-HUB}.

\subsection{Lemma for minibatch sampling}
Fix a minibatch size $\tau \in \{1,2,\ldots,n\}$ and let $\cS_t$ be a random subset of
$\{1,2,\ldots,n\}$ of size $\tau$, chosen uniformly at random. Define the gradient estimator via 
\begin{equation}
	\label{nice_est}
	g(x) \eqdef \frac{1}{\tau} \sum_{j \in \cS_t} \nabla \widetilde{\phi}_{j}(x)
\end{equation}
\begin{lemma}
	\label{tau-nice}
	The gradient estimator $g(x)$ defined in~\eqref{nice_est} is unbiased. If we further assume that $n \geq 2$, $\widetilde{\phi}_j$ is convex and $L_j$-smooth for all $j$, and $f$ is $L$-smooth, then
	\begin{align*}
		\Exp{\|g(x_t)-g(x_\star)\|^{2}} \leq 2 L(\tau)D_{f}(x_t, x_\star),
	\end{align*}
	where 
	\begin{align*}
		L(\tau)\eqdef \frac{n-\tau}{\tau(n-1)} \max _{j} L_{j}+\frac{n(\tau-1)}{\tau(n-1)} L.
	\end{align*}
\end{lemma}
\begin{proof}
	Let $\chi_j$ be the random variable defined by
	\begin{align*}
		\chi_{j}= \begin{cases}1 & j \in S \\ 0 & j \notin S\end{cases}.
	\end{align*}
	It is easy to show that 
	\begin{align}
		\label{prob}
		\Exp{\chi_{j}}=\operatorname{Prob}(j \in S)=\frac{\tau}{n}.
	\end{align}
	Unbiasedness of $g(x)$ now follows via direct computation:
	\begin{eqnarray*}
		\notag	\Exp{g(x)} & \stackrel{\eqref{nice_est}}{=} & \Exp{\frac{1}{\tau} \sum_{j \in S} \nabla \widetilde{\phi}_{j}(x)} = \Exp{\frac{1}{\tau} \sum_{j=1}^{n} \chi_{i} \nabla \widetilde{\phi}_{j}(x)} =\frac{1}{\tau} \sum_{j=1}^{n} \Exp{\chi_{i}}\nabla \widetilde{\phi}_{j}(x) \\
		\notag	&=& \frac{1}{\tau} \sum_{j=1}^{n} \operatorname{Prob}(j \in S) \nabla \widetilde{\phi}_{j}(x)  \stackrel{\eqref{prob}}{=} \frac{1}{\tau} \sum_{j=1}^{n} \frac{\tau}{n} \nabla \widetilde{\phi}_{j}(x) = \nabla f(x).
	\end{eqnarray*}
	Let us define 
	\begin{equation}
		\label{a_j}
		a_{j} \eqdef \nabla \widetilde{\phi}_{j}(x)-\nabla \widetilde{\phi}_{j}(x_\star).
	\end{equation}
	Let $\chi_{j,k}$ be the random variable defined by
	\begin{align*}
		\chi_{j,k}= \begin{cases}1 & j \in \cS_t \text { and } k \in \cS_t \\ 0 & \text { otherwise }\end{cases}.
	\end{align*}
	Note that
	\begin{equation}
		\label{indet}
		\chi_{j,k}=\chi_{j} \chi_{k}.
	\end{equation}
	Further, it is easy to show that
	\begin{equation}
		\label{proba-2}
		\Exp{\chi_{j,k}}=\operatorname{Prob}(j \in \cS_t, k\in \cS_t)=\frac{\tau(\tau-1)}{n(n-1)}.
	\end{equation}
	Let us consider 
	\begin{eqnarray}
		\notag		\Exp{\|g(x_t)-g(x_\star)\|^{2}}  &=& 	\Exp{	\left\|\frac{1}{\tau} \sum_{j \in \cS_t} \nabla \widetilde{\phi}_{j}(x)-\frac{1}{\tau} \sum_{j \in \cS_t} \nabla \widetilde{\phi}_{j}(x_\star)\right\|^{2}}\\
		\notag	&=&	\Exp{	\left\|\frac{1}{\tau} \sum_{j \in \cS_t} \left(\nabla \widetilde{\phi}_{j}(x)- \nabla \widetilde{\phi}_{j}(x_\star)\right)\right\|^{2}}\\
		\notag &\stackrel{\eqref{a_j}}{=}& \Exp{	\left\|\frac{1}{\tau} \sum_{j \in \cS_t} a_j\right\|^{2}}\\
		\notag & =& \frac{1}{\tau^{2}} \Exp{\left\|\sum_{j=1}^{n} \chi_{j} a_{j}\right\|^{2}}\\
		\notag&=&	\frac{1}{\tau^{2}} \Exp{\sum_{j=1}^{n}\left\|\chi_{j} a_{j}\right\|^{2}+\sum_{k \neq j}\left\langle\chi_{j} a_{j}, \chi_{k} a_{k}\right\rangle}\\
		&\stackrel{\eqref{indet}}{=}&\frac{1}{\tau^{2}} \Exp{\sum_{j=1}^{n}\left\|\chi_{j} a_{j}\right\|^{2}+\sum_{k \neq j} \chi_{j,k}\left\langle a_{j}, a_{k}\right\rangle}.
	\end{eqnarray}
	Using the formulas \eqref{prob} and \eqref{proba-2} we can continue:
	\begin{align}
		\notag	\Exp{\|g(x_t)-g(x_\star)\|^{2}} &= \frac{1}{\tau^{2}}\left(\frac{\tau}{n} \sum_{j=1}^{n}\left\|a_{j}\right\|^{2}+\frac{\tau(\tau-1)}{n(n-1)} \sum_{j \neq k}\left\langle a_{j}, a_{k}\right\rangle\right)\\
		\notag	&=\frac{1}{\tau n} \sum_{j=1}^{n}\left\|a_{j}\right\|^{2}+\frac{\tau-1}{\tau n(n-1)} \sum_{j \neq k}\left\langle a_{j}, a_{k}\right\rangle\\
		\notag	&=\frac{1}{\tau n} \sum_{j=1}^{n}\left\|a_{j}\right\|^{2}+\frac{\tau-1}{\tau n(n-1)}\left(\left\|\sum_{j=1}^{n} a_{j}\right\|^{2}-\sum_{j=1}^{n}\left\|a_{j}\right\|^{2}\right)\\
		&=\frac{n-\tau}{\tau(n-1)} \frac{1}{n} \sum_{j=1}^{n}\left\|a_{j}\right\|^{2}+\frac{n(\tau-1)}{\tau(n-1)}\left\|\frac{1}{n} \sum_{j=1}^{n} a_{j}\right\|^{2}.
	\end{align}
	Since $\widetilde{\phi}_j$ is convex and $L_j$-smooth, we know that
	\begin{equation}
		\left\|a_{j}\right\|^{2} =\left\|\nabla f_{j}(x_t)-\nabla f_{j}(x_\star)\right\|^{2} \leq 2 L_{j} D_{\widetilde{\phi}_{j}}(x_t, x_\star).
	\end{equation}
	Since $f$ is convex and $L$-smooth, we know that
	\begin{align}
		\left\|\frac{1}{n} \sum_{i=1}^{n} a_{i}\right\|^{2} =\|\nabla f(x_t)-\nabla f(x_\star)\|^{2} \leq 2 L D_{f}(x_t, x_\star).
	\end{align}
	Let us apply the bound $L_j \leq \max_j L_j$ and use the identity $D_f(x_t, x_\star) = \frac{1}{n}\sum_{j=1}^{n} D_{\widetilde{\phi}_j}(x_t, x_\star):$
	\begin{align}
		\notag	\Exp{	\left\|\frac{1}{\tau} \sum_{j \in \cS_t} \nabla \widetilde{\phi}_{j}(x)-\frac{1}{\tau} \sum_{j \in \cS_t} \nabla \widetilde{\phi}_{j}(x_\star)\right\|^{2}}&\leq \frac{n-\tau}{\tau(n-1)} \frac{1}{n} \sum_{j=1}^{n} 2 L_{j} D_{\widetilde{\phi}_{j}}(x_t, x_\star)\\
		\notag	&\quad +\frac{n(\tau-1)}{\tau(n-1)} 2 L D_{f}(x_t, x_\star)\\
		\notag	&\leq 2 \frac{n-\tau}{\tau(n-1)} \max _{j} L_{j} \frac{1}{n} \sum_{j=1}^{n} D_{f_{j}}(x_t, x_\star)\\
		\notag	&\quad +2 \frac{n(\tau-1)}{\tau(n-1)} L D_{f}(x, y)\\
		\notag	&=2 \frac{n-\tau}{\tau(n-1)} \max _{j} L_{j} D_{f}(x, y)\\
		\notag&\quad +2 \frac{n(\tau-1)}{\tau(n-1)} L D_{f}(x_t, x_\star)\\
		\notag	&=2\left(\frac{n-\tau}{\tau(n-1)} \max _{j} L_{j}+\frac{n(\tau-1)}{\tau(n-1)} L\right) D_{f}(x_t, x_\star).
	\end{align}	
\end{proof}

\subsection{Proof of Theorem~\ref{thm:QLSVRG}}
As in previous analysis we need to show that Assumption~\ref{sigma_t} is satisfied for the \algname{ProxSkip-HUB} method. 
%
\begin{proof}
	Let us consider the first inequality in Assumption~\ref{sigma_t} and show that it holds for the new gradient estimator $\hat{g}_t = \frac{1}{\tau}\sum_{j\in \cS_t} Q(\nabla \widetilde{\phi}_j(x_t) - \nabla \widetilde{\phi}_j(y_t)) + \nabla f(y) $:
	\begin{align}
		&\Exp{\|\hat{g}_t - \nabla f(x_\star) \|^2} = \Exp{\left\| \frac{1}{\tau}\sum_{j\in \cS_t} Q(\nabla \widetilde{\phi}_j(x_t) - \nabla \widetilde{\phi}_j(y_t)) + \nabla f(y_t) - \nabla f(x_\star) \right\|^2}.
	\end{align}
Let $\Delta_{t} =  \frac{1}{\tau}\sum_{j\in \cS_t} (\nabla \widetilde{\phi}_j(x_t) - \nabla \widetilde{\phi}_j(y_t))$ and $ \hat{\Delta}_{t} = \frac{1}{\tau}\sum_{j\in \cS_t} Q(\nabla \widetilde{\phi}_j(x_t) - \nabla \widetilde{\phi}_j(y_t))$. Using smart zero $0=\Delta_{t} -\Delta_{t} $ we have
\begin{eqnarray}
	\notag		\Exp{\|\hat{g}_t - \nabla f(x_\star) \|^2} & = & \Exp{\left\| \hat{\Delta}_{t}  -  \Delta_{t} + \Delta_{t} + \nabla f(y) - \nabla f(x_\star) \right\|^2}\\
\notag	&\stackrel{\eqref{youngs}}{\leq}& 2\Exp{\|\hat{\Delta}_t - \Delta_{t}\|^2} + 2\Exp{\|\Delta_{t}+\nabla f(y_t) - \nabla f(x_\star)\|^2}\\
\notag	&	\leq	&	2\Exp{\left\| \frac{1}{\tau}\sum_{j\in \cS_t} Q(\nabla \widetilde{\phi}_j(x_t) - \nabla \widetilde{\phi}_j(y_t)) -  \frac{1}{\tau}\sum_{j\in \cS_t} (\nabla \widetilde{\phi}_j(x_t) - \nabla \widetilde{\phi}_j(y_t))\right\|^2}\\
\label{eq:12121}
	&& \quad +2\Exp{\left\| \frac{1}{\tau}\sum_{j\in \cS_t} (\nabla \widetilde{\phi}_j(x_t) - \nabla \widetilde{\phi}_j(y_t))+ \nabla f(y) - \nabla f(x_\star) \right\|^2}.
\end{eqnarray}
Let us consider the first term~\eqref{eq:12121}, let us define $\theta^i_t = Q(\nabla \widetilde{\phi}_j(x_t) - \nabla \widetilde{\phi}_j(y_t)) - (\nabla \widetilde{\phi}_j(x_t) - \nabla \widetilde{\phi}_j(y_t))$:
\begin{align}
\notag\Exp{\|\hat{\Delta}_t - \Delta_{t}\|^2}  &= 	\Exp{\left\| \frac{1}{\tau}\sum_{j\in \cS_t} Q(\nabla \widetilde{\phi}_j(x_t) - \nabla \widetilde{\phi}_j(y_t)) - (\nabla \widetilde{\phi}_j(x_t) - \nabla \widetilde{\phi}_j(y_t)) \right\|^2}\\
\notag& = 	\Exp{\left\| \frac{1}{\tau}\sum_{j\in \cS_t}\theta_t^i \right\|^2}\\
\notag& = 	\Exp{ \frac{1}{\tau^2}\left( \sum_{j\in \cS_t}\left\|\theta_t^i \right\|^2 + \sum_{i\neq j} \left\langle \theta_t^i,\theta_t^j  \right\rangle \right)}\\
\notag& = 	\frac{1}{\tau^2}\left( \sum_{j\in \cS_t}\Exp{ \left\|\theta_t^i \right\|^2} + \sum_{i\neq j}\Exp{  \left\langle \theta_t^i,\theta_t^j  \right\rangle} \right).
\end{align}
Using independence and unbiasedness of compressors we have 
\begin{eqnarray}
\notag	\Exp{\|\hat{\Delta}_t - \Delta_{t}\|^2}  &=& \frac{1}{\tau^2}\left( \sum_{j\in \cS_t}\Exp{ \left\|\theta_t^i \right\|^2} + \sum_{i\neq j}\Exp{  \left\langle \theta_t^i,\theta_t^j  \right\rangle} \right)\\
\notag		&=&	 \frac{1}{\tau^2}\left( \sum_{j\in \cS_t}\Exp{ \left\|\theta_t^i \right\|^2} + \sum_{i\neq j}  \left\langle \Exp{\theta_t^i},\Exp{\theta_t^j}  \right\rangle \right)\\
\notag		&=&	 \frac{1}{\tau^2} \sum_{j\in \cS_t}\Exp{ \left\|\theta_t^i \right\|^2} \\
\notag		&=& \frac{1}{\tau^2} \sum_{j\in \cS_t} \Exp{\|Q(\nabla \widetilde{\phi}_j(x_t) - \nabla \widetilde{\phi}_j(y_t)) - (\nabla \widetilde{\phi}_j(x_t) - \nabla \widetilde{\phi}_j(y_t))\|^2}\\
	&\stackrel{\eqref{compress}}{\leq}& \frac{\omega}{\tau} \Exp{\frac{1}{\tau}\sum_{j\in \cS_t} \|\nabla \widetilde{\phi}_j(x_t) - \nabla \widetilde{\phi}_j(y_t)\|^2}.
\end{eqnarray}
Using Young's inequality and expectation of client sampling we get 
\begin{align}
		\label{part1}
\notag	\Exp{\|\hat{\Delta}_t - \Delta_{t}\|^2}  		&\stackrel{\eqref{youngs}}{\leq} \frac{2\omega}{\tau} \Exp{\frac{1}{\tau}\sum_{j\in \cS_t} \|\nabla \widetilde{\phi}_j(x_t) - \nabla \widetilde{\phi}_j(x_\star)\|^2}\\
\notag	&\quad +\frac{2\omega}{\tau} \Exp{\frac{1}{\tau}\sum_{j\in \cS_t} \|\nabla \widetilde{\phi}_j(y_t) - \nabla \widetilde{\phi}_j(x_\star)\|^2}\\
	\notag	&\stackrel{\eqref{youngs}}{\leq} \frac{2\omega}{\tau} \frac{1}{n}\sum_{j=1}^n  \|\nabla \widetilde{\phi}_j(x_t) - \nabla \widetilde{\phi}_j(x_\star)\|^2\\
\notag	&\quad +\frac{2\omega}{\tau} \frac{1}{n}\sum_{j=1}^n \|\nabla \widetilde{\phi}_j(y_t) - \nabla \widetilde{\phi}_j(x_\star)\|^2\\
	&\stackrel{\eqref{eq:smooth-grad}}{\leq} \frac{4\omega}{\tau}L_{\max}D_f(x_t,x_\star)+\frac{2\omega}{\tau} \frac{1}{n}\sum_{j=1}^n \|\nabla \widetilde{\phi}_j(y_t) - \nabla \widetilde{\phi}_j(x_\star)\|^2 .
\end{align}
Let us consider the second term in~\eqref{eq:12121}:
\begin{align}
	\label{eq:aqaqaq}
\notag	&\Exp{\left\| \frac{1}{\tau}\sum_{j\in \cS_t} (\nabla \widetilde{\phi}_j(x_t) - \nabla \widetilde{\phi}_j(y_t))+ \nabla f(y) - \nabla f(x_\star) \right\|^2}\\
\notag	& = 	\Exp{\left\| \frac{1}{\tau}\sum_{j\in \cS_t} (\nabla \widetilde{\phi}_j(x_t) - \nabla \widetilde{\phi}_j(y_t)+\nabla \widetilde{\phi}_j(x_\star)-\nabla \widetilde{\phi}_j(x_\star))+ \nabla f(y) - \nabla f(x_\star) \right\|^2} \\
\notag& \stackrel{\eqref{youngs}}{\leq} 2\Exp{\left\|\frac{1}{\tau}\sum_{j\in \cS_t} \nabla \widetilde{\phi}_j(x_t) - \frac{1}{\tau}\sum_{j\in \cS_t} \nabla \widetilde{\phi}_j(x_\star)  \right\|^2}\\
\notag & \quad + 2\Exp{\left\|\frac{1}{\tau}\sum_{j\in \cS_t} \nabla \widetilde{\phi}_j(x_\star) - \frac{1}{\tau}\sum_{j\in \cS_t} \nabla \widetilde{\phi}_j(y_t) - \left(\nabla f(x_\star)-\frac{1}{\tau}\sum_{j\in \cS_t} \nabla \widetilde{\phi}_j(y_t)  \right) \right\|^2} \\
& \leq 2\Exp{\left\|\frac{1}{\tau}\sum_{j\in \cS_t} \nabla \widetilde{\phi}_j(x_t) - \frac{1}{\tau}\sum_{j\in \cS_t} \nabla \widetilde{\phi}_j(x_\star)  \right\|^2} + 2\Exp{\left\|\frac{1}{\tau}\sum_{j\in \cS_t} \nabla \widetilde{\phi}_j(y_t) - \frac{1}{\tau}\sum_{j\in \cS_t} \nabla \widetilde{\phi}_j(x_\star)  \right\|^2}.
\end{align}
Using Lemma~\ref{tau-nice} and Jensen's inequality~\eqref{jensen} we have
\begin{align}
	\label{part2}
	\notag	&\Exp{\left\| \frac{1}{\tau}\sum_{j\in \cS_t} (\nabla \widetilde{\phi}_j(x_t) - \nabla \widetilde{\phi}_j(y_t))+ \nabla f(y) - \nabla f(x_\star) \right\|^2}\\
	 \notag& \leq 2\Exp{\left\|\frac{1}{\tau}\sum_{j\in \cS_t} \nabla \widetilde{\phi}_j(x_t) - \frac{1}{\tau}\sum_{j\in \cS_t} \nabla \widetilde{\phi}_j(x_\star)  \right\|^2} + 2\Exp{\left\|\frac{1}{\tau}\sum_{j\in \cS_t} \nabla \widetilde{\phi}_j(y_t) - \frac{1}{\tau}\sum_{j\in \cS_t} \nabla \widetilde{\phi}_j(x_\star)  \right\|^2}\\
	&\leq 4L(\tau)D_f(x_t,x_\star) + \frac{2}{n}\sum_{j=1}^n \left\|\nabla \widetilde{\phi}_j(y_t) -  \nabla \widetilde{\phi}_j(x_\star)  \right\|^2 .
\end{align}
Combining two parts \eqref{part1}, \eqref{part2} and plugging into~\eqref{eq:12121} we get
\begin{align}
		\label{eq:first_in}
\notag	\Exp{\|\hat{g}_t - \nabla f(x_\star) \|^2} &	\leq		2\Exp{\left\| \frac{1}{\tau}\sum_{j\in \cS_t} Q(\nabla \widetilde{\phi}_j(x_t) - \nabla \widetilde{\phi}_j(y_t)) -  \frac{1}{\tau}\sum_{j\in \cS_t} (\nabla \widetilde{\phi}_j(x_t) - \nabla \widetilde{\phi}_j(y_t))\right\|^2}\\
\notag&\quad +2\Exp{\left\| \frac{1}{\tau}\sum_{j\in \cS_t} (\nabla \widetilde{\phi}_j(x_t) - \nabla \widetilde{\phi}_j(y_t))+ \nabla f(y) - \nabla f(x_\star) \right\|^2}\\
\notag&\leq 8L(\tau)D_f(x_t,x_\star) + \frac{4}{n}\sum_{j=1}^n \left\|\nabla \widetilde{\phi}_j(y_t) -  \nabla \widetilde{\phi}_j(x_\star)  \right\|^2\\
\notag	& \quad + \frac{8\omega}{\tau}L_{\max}D_f(x_t,x_\star)+\frac{4\omega}{\tau} \frac{1}{n}\sum_{j=1}^n \|\nabla \widetilde{\phi}_j(y_t) - \nabla \widetilde{\phi}_j(x_\star)\|^2\\
&\leq 2\cdot 4\left( L(\tau) + \frac{\omega}{\tau}L_{\max} \right)D_f(x_t,x_\star) + 4\left(1+\frac{\omega}{\tau}\right)\sigma_{t},
\end{align}
where $\sigma_{t} =  \frac{1}{n}\sum_{j=1}^n \|\nabla \widetilde{\phi}_j(y_t) - \nabla \widetilde{\phi}_j(x_\star)\|^2$.
Let us consider update of control variable $y_t$:
\begin{align}
	y_{t+1}=\left\{\begin{array}{lll}
		x_{t} & \text { with probability } & q \\
		y_{t} & \text { with probability } & 1-q
	\end{array}\right. .
\end{align}
Let us show that second inequality in Assumption~\ref{sigma_t}:
\begin{align}
	\label{eq:second_in}
\notag	\Exp{ \sigma_{t+1} } &= \Exp{\frac{1}{n} \sum_{j=1}^{n}  \|\nabla \widetilde{\phi}_j(y_{t+1}) - \nabla \widetilde{\phi}_j(x_\star)\|^2 }\\
\notag	&=\left(1-q\right)\frac{1}{n} \sum_{j=1}^{n}  \|\nabla \widetilde{\phi}_j(y_{t}) - \nabla \widetilde{\phi}_j(x_\star)\|^2 + q \frac{1}{n} \sum_{j=1}^{n}  \|\nabla \widetilde{\phi}_j(x_{t}) - \nabla \widetilde{\phi}_j(x_\star)\|^2\\
	& = \left(1-q\right)\sigma_{t} + 2 qL_{\max} D_f(x_t,x_\star).
\end{align}
Using \eqref{eq:first_in} and \eqref{eq:second_in} bounds we can confirm that Assumption~\ref{sigma_t} is satisfied with the following constants:
\begin{align*}	
	A = 4\left(L(\tau)+\frac{\omega}{\tau}L_{\max}\right) , 
	\quad B = 4\left(1+\frac{\omega}{\tau}\right), \quad C = 0, \quad \tilde{A} = q L_{\max},\quad \tilde{B} = 1-q, \quad \tilde{C} = 0.
\end{align*}
Applying Theorem~\ref{thm:main} leads to the final result
	\begin{align}
		\squeeze 	\Exp{\Psi_{T}} \leq \max \left\{(1-\gamma \mu)^{T},(1-p^2)^T,\left(1-\nicefrac{q}{2}\right)^{T}\right\} \Psi_{0},
	\end{align}
	where the Lyapunov function is defined by $$	\Psi_{t} \eqdef \|x_{t} - x_{\star}\|^2 + \frac{\gamma^2}{p^2}\|h_t - h_{\star}\|^2 + \gamma^{2} \frac{8}{q}\left(1+\frac{\omega}{\tau}\right) \sigma_{t} .$$
	Let us set $\gamma = \frac{1}{A+\MM \tilde{A}}$, $p = \sqrt{\gamma\mu}$ and $q = 2\gamma\mu$. Using the same proof as for \algname{ProxSkip} in Section~\ref{sec:GD_est} and $L(\tau) \leq L_{\max}$ we get communication and iteration complexities:  
\begin{align*}
	T_{\text{comm.}}=\mathcal{O}\left(\sqrt{\frac{L_{\max}}{\mu}\left(1+\frac{\omega}{\tau}\right)}\log\frac{1}{\varepsilon}\right), \qquad 
		T_{\text{iter.}} = \mathcal{O}\left(\frac{L_{\max}}{\mu}\left(1+\frac{\omega}{\tau}\right)\log\frac{1}{\varepsilon}\right).
\end{align*}

\end{proof}
If we use full participation $\tau = n$ and $q = \frac{1}{\omega+1}$ and $r(x) \equiv 0$ then we get the same rate as for \algname{DIANA}~\citep{DIANA,DIANA2} and \algname{RAND-DIANA}~\citep{Shifted}.

\section{Analysis of ProxSkip-LSVRG}
The analysis of \algname{ProxSkip-LSVRG} is almost the same to the analysis of \algname{ProxSkip-HUB}, with one exception. We use a different sigma component: 
\begin{align}
	\sigma_{t} = \Exp{\left\| \frac{1}{\tau} \sum_{j\in \cS_t}(\nabla \widetilde{\phi}_j(y_t) - \nabla \widetilde{\phi}_j(x_\star)) \right\|^2}.
\end{align}
Let us consider $\Exp{\| \hat{g}_t - \nabla f(x_\star) \|^2}$:
\begin{align}
	\Exp{\| \hat{g}_t - \nabla f(x_\star) \|^2} = \Exp{\left\|\frac{1}{\tau}\sum_{j \in S_t} \left( \nabla \widetilde{\phi}_j(x_t) - \nabla \widetilde{\phi}_j(y_t) \right) + \nabla f(y)   - \nabla f(x_\star) \right\|^2} 
\end{align}
Using~\eqref{eq:aqaqaq} we have 
\begin{align}
\notag		\Exp{\| \hat{g}_t - \nabla f(x_\star) \|^2} &\leq 2\Exp{\left\|\frac{1}{\tau}\sum_{j\in \cS_t} \nabla \widetilde{\phi}_j(x_t) - \frac{1}{\tau}\sum_{j\in \cS_t} \nabla \widetilde{\phi}_j(x_\star)  \right\|^2}\\
\notag			& + 2\Exp{\left\|\frac{1}{\tau}\sum_{j\in \cS_t} \nabla \widetilde{\phi}_j(y_t) - \frac{1}{\tau}\sum_{j\in \cS_t} \nabla \widetilde{\phi}_j(x_\star)  \right\|^2}\\
		&\leq 4L(\tau) D_f(x_t,x_\star) +2\sigma_{t} .
\end{align}
Let us show that second inequality in Assumption~\ref{sigma_t}, using Lemma~\ref{tau-nice} we get
\begin{align}
	\label{eq:second_in}
	\notag	\Exp{ \sigma_{t+1} } &=  \Exp{\left\| \frac{1}{\tau} \sum_{j\in \cS_t}(\nabla \widetilde{\phi}_j(y_{t+1}) - \nabla\widetilde{\phi}_j(x_\star)) \right\|^2}\\
	\notag	&=\left(1-q\right) \Exp{\left\| \frac{1}{\tau} \sum_{j\in \cS_t}(\nabla\widetilde{\phi}_j(y_t) - \nabla\widetilde{\phi}_j(x_\star)) \right\|^2} + q  \Exp{\left\| \frac{1}{\tau} \sum_{j\in \cS_t}(\nabla \widetilde{\phi}_j(x_t) - \nabla \widetilde{\phi}_j(x_\star)) \right\|^2}\\
	& = \left(1-q\right)\sigma_{t} + 2 qL(\tau) D_f(x_t,x_\star).
\end{align}
We showed that Assumption~\ref{sigma_t} is satisfied with following constants:
\begin{align}
		A = 2 L(\tau) , 
	\quad B = 2, \quad C = 0, \quad \tilde{A} = q L(\tau),\quad \tilde{B} = 1-q, \quad \tilde{C} = 0.
\end{align}
Applying Theorem~\ref{thm:main} with $\gamma = \frac{1}{6L(\tau)}$  we get final bound:
\begin{align*}
	\squeeze 	\Exp{\Psi_{T}} \leq \max \left\{(1-\gamma \mu)^{T},(1-p^2)^T,\left(1-\frac{q}{2}\right)^{T}\right\} \Psi_{0},
\end{align*}
where the Lyapunov function is defined as $$	\Psi_{t} \eqdef \|x_{t} - x_{\star}\|^2 + \frac{\gamma^2}{p^2}\|h_t - h_{\star}\|^2 + \gamma^{2} \frac{4}{q}\sigma_{t}.$$

Using the same argument as for \algname{ProxSkip} and setting $\frac{q}{2} = \gamma\mu$, we get 
\begin{align}
T_{\text{comms}} =	\mathcal{O}\left(\sqrt{\frac{L(\tau)}{\mu}}\log\frac{1}{\varepsilon}\right), \qquad T_{\text{iters}} =\mathcal{O}\left(\frac{L(\tau)}{\mu}\log\frac{1}{\varepsilon}\right).
\end{align}
If $r(x) \equiv 0$, this recovers results of \citet{L-SVRG}.

\end{document}